\documentclass[opre]{informs}
\OneAndAHalfSpacedXI

\usepackage{natbib}
\bibpunct[, ]{(}{)}{,}{a}{}{,}%
\def\bibfont{\small}%
\usepackage{tikz}
\usepackage{booktabs} 
\usepackage{framed}
\usepackage{mathtools}
\usepackage{comment}
\usepackage[normalem]{ulem}
\usepackage{enumitem}
\usepackage{amsmath}
\usepackage{epstopdf}
\usepackage{graphicx}
\usepackage{threeparttable}
\usepackage{subcaption}
\usepackage{svg} 
\usepackage{multirow}
\usepackage{threeparttable}
\usepackage{url}
\usepackage{subfloat}
\usepackage{bbm}
\usepackage{hyperref}
\usepackage{diagbox}
\usepackage{tablefootnote}
\usepackage[justification=centering]{caption}
\usepackage{graphicx}
\usepackage{epstopdf}
\graphicspath{{Figures/}{figures/}}

\makeatletter
\renewcommand\AAnormalsizeXI{%
   \@setfontsize\normalsizeXI{11}{12.66}%
   \Dskips{0.86}{0.24}{0.38}{0.45}%
   \let\@listi\@listI
   \baselineskip18.35pt plus 0.8pt minus 0.8pt
\@bls=18.35pt\@em=11pt
\relax}
\renewcommand\AAsmallX{%
   \@setfontsize\smallX{10}{11.5}%
   \Dskips{0.70}{0.24}{0.38}{0.45}%
   \def\@listi{\leftmargin\leftmargini
     \topsep 0\p@ \parsep 1.5\p@ \itemsep \parsep}%
   \baselineskip16.65pt plus 0.8pt minus 0.8pt
\@bls=16.65pt\@em=10pt
\relax}
\renewcommand\AAfootnotesizeIX{%
   \@setfontsize\footnotesizeIX\@ixpt\@xipt
   \Dskips{0.70}{0.24}{0.38}{0.45}%
   \def\@listi{\leftmargin\leftmargini
      \topsep 0\p@ \parsep 0\p@ \itemsep \parsep}%
   \baselineskip14.90pt plus 0.8pt minus 0.8pt
\@bls=14.90pt\@em=9pt
\relax}
\OneAndAHalfSpacedXI
\def\section{\@startsection {section}{1}{\z@}{-10pt plus -5pt minus -3pt}{2pt}%
  {\fs.13.15.\sffamily\bfseries\RAGG}}%
\def\subsection{\@startsection{subsection}{2}{\z@}{-10pt plus -5pt minus -2pt}{2pt}%
  {\TEN\sffamily\bfseries\RAGG}}%
\def\subsubsection{\@startsection{subsubsection}{3}{1.0em}{3pt plus 3pt minus 2pt}{-7pt}%
  {\TEN\bf}}%
\makeatother
\setlist{topsep=1pt,itemsep=1pt,parsep=0pt,partopsep=0pt}

\hypersetup{
	colorlinks=true,
	linkcolor=black,
	citecolor=blue
} 

\usepackage[hyphenbreaks]{breakurl}
\definecolor{myred}{rgb}{0.75,0.0,0.0}

\usepackage{stackengine}

\usepackage{algorithm}
\usepackage[noend]{algpseudocode}

\usepackage{fancybox}
\usepackage[framemethod=tikz]{mdframed}
\usepackage{lipsum}
\usetikzlibrary{shadows}
\newmdenv[shadow=true,shadowcolor=black]{myshadowbox}

\TheoremsNumberedThrough 
\ECRepeatTheorems

\EquationsNumberedThrough 

\MANUSCRIPTNO{}
\makeatletter

\usepackage[utf8]{inputenc} 
\usepackage[T1]{fontenc}    
\usepackage{hyperref}       
\usepackage{url}            
\usepackage{booktabs}       
\usepackage{amsfonts}       
\usepackage{nicefrac}       
\usepackage{microtype}      
\usepackage{xcolor}
\usepackage{amsmath}         

\usepackage{graphicx}
\usepackage{multirow}
\usepackage{pifont}
\usepackage{makecell}
\usepackage{colortbl}

\usepackage{xcolor}
\definecolor{mygray}{RGB}{192,192,192}
\definecolor{mygray2}{RGB}{230,230,230}
\usepackage[most]{tcolorbox}
\usepackage{float}
\usepackage{xspace}
\tcbset{
  aibox/.style={
    width=\textwidth,
    top=10pt,
    colback=white,
    colframe=black,
    colbacktitle=black,
    enhanced,
    center,
    attach boxed title to top left={yshift=-0.1in,xshift=0.15in},
    boxed title style={boxrule=1pt,colframe=black,colback=mygray2, colbacktitle=black},
    coltitle = black,
    fonttitle=\bfseries, 
  }
}
\newtcolorbox{AIbox}[2][]{aibox,title=#2,#1}
\usepackage{listings}
\definecolor{myblue}{RGB}{100, 150, 200}
\definecolor{mygreen}{RGB}{80, 160, 80}

\definecolor{darkgreen}{rgb}{0.0, 0.5, 0.0}
\definecolor{darkgray}{gray}{0.4}
\definecolor{maroon}{rgb}{0.5, 0.0, 0.0}
\definecolor{navy}{rgb}{0.0, 0.0, 0.5}
\definecolor{teal}{rgb}{0.0, 0.5, 0.5}

\newcommand{\affmark}[1]{\textsuperscript{#1}}
\newcommand{\cofirst}{\textsuperscript{*}}
\newcommand{\corrmark}{\textsuperscript{$\dagger$}}

\usepackage{wrapfig,lipsum,booktabs}
\usepackage{tikz}
\usetikzlibrary{shapes, arrows.meta, positioning}

\begin{document}
	
	
	\RUNAUTHOR{Authors}
	
	\RUNTITLE{InvEvolve: Evolving White-Box Inventory Policies via Large Language Models with Performance Guarantees}
	\TITLE{InvEvolve: Evolving White-Box Inventory Policies via Large Language Models with Performance Guarantees} 

\ARTICLEAUTHORS{%
	\AUTHOR{Chenyu Huang\affmark{1}\cofirst, Jianghao Lin\affmark{2}\cofirst\corrmark, Zhengyang Tang\affmark{3}\cofirst, Bo Jiang\affmark{1}}
	\AUTHOR{Ruoqing Jiang\affmark{4}\corrmark, Benyou Wang\affmark{3}, Lai Wei\affmark{5}\corrmark}
    \medskip
	\AFF{\affmark{1}Shanghai University of Finance and Economics \qquad \affmark{2}Shanghai Jiao Tong University}
	\AFF{\affmark{3}The Chinese University of Hong Kong, Shenzhen \qquad \affmark{4}Tsinghua University \qquad \affmark{5}Boston College}
}



\ABSTRACT{We study how large language models can be used to generate inventory policies in online settings with non-stationary demand. Our work is motivated by recent advances in LLM-based evolutionary search, such as AlphaEvolve, which demonstrates strong performance on static and highly structured problems such as mathematical discovery, but is not directly suited to dynamic inventory settings with online updates. We propose \emph{InvEvolve}, an end-to-end inventory policy evolution and inference framework grounded in confidence-interval-based certification. 
Built on a large language model trained via reinforcement learning, \emph{InvEvolve} can process demand data together with additional numerical and textual features, and generates white-box inventory policies with statistical safety guarantees for future deployment. 
We further introduce a unified framework with theoretical guarantees that connects training, inference, and deployment. This allows us to derive a lower bound on the probability that \emph{InvEvolve} evolves a statistically safe and improved policy, and to characterize the multi-period performance gap relative to the oracle-safe benchmark. Tested on both synthetic data and real-world retail data, \emph{InvEvolve} outperforms classical inventory policies and deep-learning-based methods. In canonical inventory settings, it generates new policies that outperform existing benchmarks. 
}


\maketitle

\begingroup
\renewcommand\thefootnote{}
\footnotetext{\cofirst These authors contributed equally. \corrmark Corresponding authors.}
\endgroup

\section{Introduction}

Inventory management is a central problem in operations management. 
Designing inventory policies that are both efficient and reliable has long been a fundamental problem for both researchers and practitioners. 
Existing 
methods can generally be divided into two categories. The first category develops ``white-box'' policies, which are structured, interpretable, and easy to implement. 
In many canonical settings, they  yield strong performance and  theoretical guarantee. However, they often struggle to handle high-dimensional features and non-stationary demand in real-world scenarios. 
The second category develops data-driven ``black-box'' methods, such as deep learning and reinforcement learning, which can incorporate richer features and adapt to more complex environments. 
However, they face several limitations, including limited interpretability, weak cold-start performance, and low transparency.

Recent development of large language models (LLMs)  
provides a new opportunity to generate inventory policies that combine the advantages of these two streams. 
In particular, the AlphaEvolve  approach \citep{alphaevolve2025}
can 
generate, evaluate, and refine candidate algorithms directly in the space of executable rules, without requiring a prespecified parameterized policy family. 
Such an approach have shown great potential in static and highly structured  problems with fixed evaluation criteria, such as the kissing number problem. 
%
Inventory problems, however,  differ from these static problems. 
Inventory decisions are sequential, demand changes over time, and policies to be deployed in future periods need be decided using historical data. 
As a result, a policy that performs well in the past may not necessarily perform well in the future. Moreover, 
inventory problems is often data-driven and feature-rich. 
These challenges call for a framework that evolves data-driven inventory policies in online, non-stationary settings. 



We develop \textit{InvEvolve}, a confidence-interval-based framework for iterative policy generation, replay evaluation, and statistical inference.
The framework is designed to combine the interpretability and implementability of white-box inventory policies with the flexibility of black-box learning methods. 
It allows the LLM to generate candidate white-box policies and then statistically screen them using replay evaluation based on historical data. 

Importantly, these candidate policies are not selected from a fixed library of existing inventory policies. Instead, InvEvolve searches over executable and operationally interpretable policy variants. These variants may extend classical policy structures by modifying their functional forms, adjustment terms, or decision rules.  The final output is a white-box inventory policy that is statistically safe relative to the baseline policy and is more likely to improve future performance.
Once established, the framework  can be applied to different inventory scenarios for continual policy search with rolling 
exploration across periods, 
without retraining a specialized 
model for each new scenario.

Theoretically, we analyze the performance of InvEvolve under three auditable interface conditions: replay transfer, adaptive confidence validity, and effective search coverage. 
We establish a lower bound on the probability that the LLM-evolved policy is both safe relative to the baseline policy and improves future performance. We further characterize the performance gap between the LLM-evolved policy and the oracle-safe policy under multi-period deployment, and analyze how key factors affect 
this gap. 

In addition, we systematically compare \textit{InvEvolve} with five white-box inventory policies and two representative learning-based methods using both synthetic and real data. \textit{InvEvolve} achieves win rates of  83\% and 67\% against the white-box policies on the two datasets, respectively, and its win rates are also higher than those of the deep learning benchmark policies. These results demonstrate the flexibility and competitive performance of \textit{InvEvolve}.
Furthermore, in the classical lost-sales system with positive lead time, we use the capped base-stock policy 
\citep{xin}, an established and well-performing policy in this setting, as the starting point for policy evolution. \textit{InvEvolve} identifies improved policy variants and outperforms the capped base-stock policy under multiple demand distributions. This result suggests that the framework may also 
provide a useful approach for discovering new policy structures and generating insights for  inventory theory.

To the best of our knowledge, this is the first work to introduce an AlphaEvolve-style evolutionary search framework to inventory policy design. We also provide a theoretical framework that connects  training, inference, and deployment. Our main contributions are summarized as follows:

\begin{enumerate}
    \item We propose a confidence interval-based LLM policy evolution framework for inventory problems. The framework balances 
    historical replay performance and future deployment risks. It provides an end-to-end procedure for generating executable white-box inventory policies. enabling large language models to generate executable white-box inventory policies end-to-end.
    \item We construct an unified theoretic framework for analyzing LLM-based white-box policy evolution, which connects the training, inference, and deployment stages. We establish a lower bound on the probability that 
    \textit{InvEvolve} identifies a statistically safe and improved policy, and characterize the multi-period performance gap between the deployed policy and the oracle-safe policy.
    \item We train and open-source an LLM for inventory policy exploration. The model can incorporate both numerical and textual features of demand information. Experiments on synthetic and real-world data show that \textit{InvEvolve} outperforms several white-box inventory policies and representative learning-based methods. In classical inventory settings, our framework also evolves new competitive inventory policies.
\end{enumerate}

\subsection{Literature Review}

Our research is related to three streams of literature:

\paragraph{Classical and learning-based inventory methods.}
The inventory literature has developed structured and interpretable white-box policies for canonical inventory models. The optimal policy and its structural properties are characterized in various settings \citep{zipkin2000foundations}. For settings where the exact optimal policy is out of reach, heuristics built on interpretable structures are proposed and shown to have strong performance guarantees  (e.g, \citealp{zipkin2008oldnew,xin}).  
Such policies and heuristics are simple to implement. However, they are less directly applicable in real-world dynamic environments, where the original policy assumptions may no longer hold, leading to less stable performance. This has motivated a growing stream of research on learning-based inventory management, where decisions are learned directly from data 
Examples include cost-oriented newsvendor learning \citep{ban2019bigdata}, prescriptive analytics with contextual features \citep{bertsimas2020predictive}, decision-aware learning \citep{elmachtoub2022spo}, and reinforcement-learning-based inventory control \citep{zhang2020closinggap,gijsbrechts2022can,oroojlooyjadid2022beergame,xie2026deepstock}. 
However, such methods are often less interpretable than classical white-box policies, and they may also suffer from a cold-start problem when the available data are limited.
Our work complements these two streams   
by using black-box LLMs to generate white-box inventory policies, which leverages the adaptivity 
of black-box models to online demand environments while preserving the interpretability of structured inventory policies. 

\paragraph{LLMs for operations research.}
A recent stream of research examines how LLMs can support operations research by translating natural language into optimization models, code, or heuristic procedures. Representative directions include automated optimization modeling \citep{huang2025orlm,liang2026large,zhou2026steporlm,tang2026calm} and agent-based LLM applications to inventory problems \citep{quan2024invagent,optiguide}. End-to-end LLM solvers further explore direct solution generation from textual descriptions \citep{jiang2025llmcosolver}. Most existing studies on LLM applications in operations research provide limited theoretical analysis of model performance and guarantees from training to inference. 
Our work complements this literature by introducing a unified framework with theoretical guarantees in inventory management that connects these stages and provides performance guarantees for statistically certified policy deployment.

\paragraph{Evolutionary discovery with language models.}
A third stream of related research studies closed-loop discovery using language models, in which candidate rules or programs are iteratively generated, evaluated, and refined. Pioneering work such as FunSearch \citep{romeraparedes2024funsearch} identifies improved algorithms for given combinatorial optimization problems. Its successor, AlphaEvolve \citep{alphaevolve2025}, further demonstrates strong performance for mathematical discovery, for example by improving the lower bound for the $11$-dimensional kissing number problem. More recent studies extend this paradigm to broader decision-making settings. For example, language agents are used for virtual machine scheduling in cloud computing \citep{wu2025learning}, and LLM-enhanced AutoML is used to incorporate organizational memory in marketing analytics \citep{lei2025orgmemory}. Relatedly, \citet{yang2025heuragenix} develop a two-stage LLM-based hyper-heuristic framework for combinatorial optimization that first evolves heuristics and then dynamically selects among them during problem solving. However, most existing studies focus on problems with relatively static and well-defined mathematical structures, rather than on dynamic and data-driven inventory policy design. This leaves a gap for methods that combine generative search with white-box inventory policies and explicit deployment safeguards.

\section{Model Setup and Training Framework}
\label{sec:basic_model_training}

\subsection{Model Setup}
\label{subsec:basic_model}

We consider a single-sourcing lost-sales inventory system with lead time $L$. We model the inventory system on two time scales. The outer, macroscopic time index $t = 1, \ldots, T$ denotes a policy-update epoch. At each such epoch, the decision maker observes the recent performance of the current inventory policy over a historical window and then selects a new policy for the subsequent period, during which the demand distribution may shift over time.

\begin{figure}[!ht]
\centering
\resizebox{\textwidth}{!}{
\includegraphics[width=\textwidth]{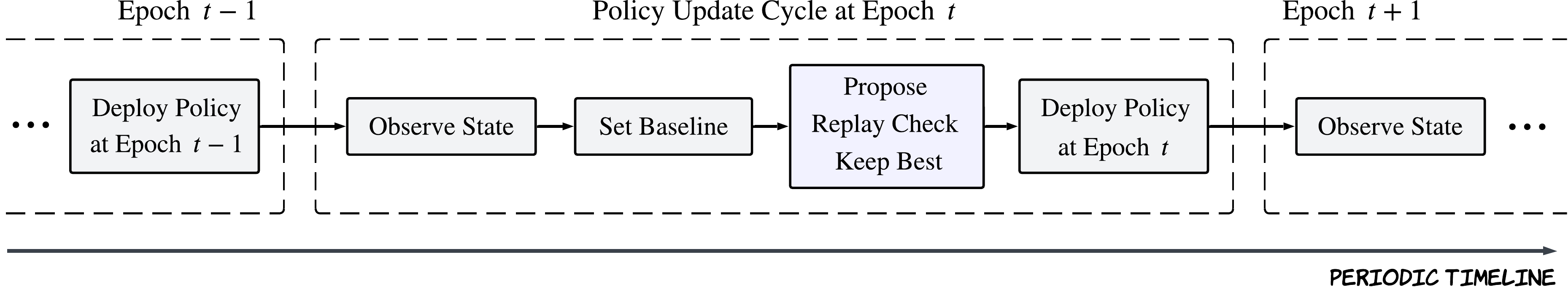}
}
\caption{Timeline of the InvEvolve update cycle. The system reviews recent information, picks a baseline, searches for a better policy, and deploys the selected policy.}
\label{fig:lgps_timeline}
\end{figure}

The inner, microscopic time is indexed by $n=1,2,\dots$, which corresponds to the standard period-level index in the classical lost-sales model. At each period, demand $D_n$ arrives and unmet demand is lost. Let $x_n$ denote the on-hand inventory at the beginning of period $n$, and let $\{p_{n,i}\}_{i=1}^L$ denote the pipeline orders that will arrive in future periods. The inventory position is defined as $\mathrm{IP}_n = x_n + \sum_{i=1}^L p_{n,i}$. After observing the current state, the decision maker places an order $q_n$ according to the inventory policy selected in this period, which arrives after $L$ periods. The objective is to minimize the long-run average cost, including holding cost $h$ and lost-sales penalty $p$.

We propose an LLM-based framework, \emph{InvEvolve}. Operationally, InvEvolve uses an inner-loop \emph{LLM-guided Policy Search} (LGPS) procedure to generate and select the policy to be used in the next period. Specifically, the LLM generates executable code representing candidate inventory policies. Since different code snippets may represent the same policy logic, we work with a \emph{canonical policy space} rather than raw code strings/outputs. Formally, let $\widetilde{\Pi}$ denote the set of syntactically admissible/valid code outputs, and let
\[
\mathrm{can}:\widetilde{\Pi}\to\Pi
\]
denote a canonicalization map that merges semantically or structurally equivalent code into the same canonical policy. All theoretical objects are defined on the canonicalized policy space $\Pi$. In the training-stage analysis, we treat the canonical policy space $\Pi$ as finite or countable after canonicalization.

Each canonical policy $\pi \in \Pi$ maps the current state and observed features to an order decision. To exclude invalid or non-deployable outputs, we introduce a structural validity indicator $g:\Pi \to \{0,1\}$, where $g(\pi)=1$ means that $\pi$ is executable, interpretable, inventory-consistent, and satisfies the required white-box constraints. These constraints may include, for example, valid state access, bounded arithmetic, a monotone threshold structure, or other auditability requirements. The framework searches for improved policies only among policies with $g(\pi)=1$.

Throughout Sections~\ref{sec:basic_model_training} and~\ref{sec:inference_deployment}, we focus on the evolution and update of inventory policies, and therefore primarily work with the outer index $t$. The inner index $n$ is mostly used in Section~\ref{sec:experiment}, where we analyze the structure of individual policies. At each outer epoch $t$, the system selects a policy $d_t\in\Pi$ to be deployed over the subsequent operational window. We assume that a fixed family of interpretable baseline policies is available, denoted by
\[
\Pi^{\mathrm{base}}
=
\{\pi^{\mathrm{base}}_1,\ldots,\pi^{\mathrm{base}}_M\}
\subseteq \Pi.
\]

At the beginning of outer epoch $t$, the decision maker observes a summary state $\mathcal{H}_t$, which includes recent demand trajectories, the current inventory status, in-transit inventory, lead-time information, cost parameters, and historical policy performance. In addition, $\mathcal{H}_t$ provides a fully observed (mature) replay window, 
which allows the decision maker to rerun policies that were not deployed. Let $d_{t-1}$ denote the policy deployed in the previous outer epoch. In each outer epoch $t$, the decision maker may also manually specify a reference baseline policy, denoted by $\pi_t^{\mathrm{ref}} \in \Pi^{\mathrm{base}}$, which represents the policy judged appropriate for the next operational window and serves as the unique safety anchor. 

Given $\pi_t^{\mathrm{ref}}$, the inference stage starts from the fallback pool $\mathcal{A}_{t,0}\triangleq\Pi^{\mathrm{base}}\cup\{d_{t-1}\},$ replay-evaluates these candidates against $\pi_t^{\mathrm{ref}}$, and initializes the champion $\pi_{t,0}^{\mathrm{ch}}$ as the strongest safety-certified fallback, with $\pi_t^{\mathrm{ref}}$ as the default if no other candidate is certified.

For a policy $\pi$ and outer epoch $t$, let $C_t^{\mathrm{dep}}(\pi)$ denote the expected cumulative deployment cost incurred over the operational window associated with epoch $t$. For any comparator policy $\tilde\pi$, define the true deployment gain of $\pi$ relative to $\tilde\pi$ as
\begin{equation}
V_t^{\mathrm{dep}}(\pi \mid \tilde\pi)
\triangleq
C_t^{\mathrm{dep}}(\tilde\pi) - C_t^{\mathrm{dep}}(\pi).
\label{eq:def_v_dep}
\end{equation}
Thus, $V_t^{\mathrm{dep}}(\pi \mid \tilde\pi) > 0$ indicates that $\pi$ has a lower deployment cost than $\tilde\pi$ over the operational window associated with epoch $t$.

During decision making at outer epoch $t$, the system has access to a mature replay window $W_t$, which enables sample-path evaluation of counterfactual policies. Let $V_t^{\mathrm{rep}}(\pi \mid \tilde\pi)$ denote the replay gain of $\pi$ relative to $\tilde\pi$ over this window. Note that replay performance and deployment performance do not necessarily coincide. 
Replay performance is evaluated on historical sample paths that are fully observed at the time of decision making, whereas deployment performance depends on future demand.

We use the same baseline family $\Pi^{\mathrm{base}}$ throughout training, inference, and deployment. Training compares generated policies against the full baseline family. In inference and deployment, the fixed within-period search budget is denoted by $J$, i.e., in each outer epoch $t$, the inner-loop search runs for at most $J$ iterations/rounds, where each iteration generates and evaluates one candidate policy.  
$\pi_{t,j}^{\mathrm{ch}}$ denotes the round-$j$ champion within epoch $t$, and $d_t$ denotes the policy deployed in epoch $t$.

\subsection{Training Procedure and Theoretical Property}
\label{subsec:training_binary_reward}

We train the model using reinforcement learning such that the LLM generates policies that satisfy $g(\pi)=1$ with high probability in each period, while also progressively improving performance. 
In reinforcement learning, 
an agent learns to select better actions 
through repeated interaction with the environment. In our setting, the policy $\pi$ generated by the LLM is viewed as the agent's action. For a  training instance $x \sim \mathcal{D}_{\mathrm{tr}}$, the cost $C(\pi;x)$ incurred by policy $\pi$ is treated as the outcome of that action. During training, the generated policy is compared against the baseline family $\Pi^{\mathrm{base}}=\{\pi^{\mathrm{base}}_1,\ldots,\pi^{\mathrm{base}}_M\}$. The reward is defined as a binary indicator:
\begin{equation*}
r(x,\pi)
\;=\;
\mathbf{1}\!\left\{
g(\pi)=1,\;
C(\pi;x) < \min_{\tilde\pi\in\Pi^{\mathrm{base}}} C(\tilde\pi;x)
\right\}.
\end{equation*}

This definition is consistent with the operational criterion used in the experiments. A generated policy receives a reward only if it is structurally valid and strictly outperforms the best baseline policy on the given instance. The reward therefore serves as a success indicator rather than a dense surrogate of the cost. This design 
both preserves a strong and stable training signal and admits a direct probabilistic interpretation.

Specifically, for any canonical policy $\pi$, define its success probability by $\Delta(\pi)\;\triangleq\;\mathbb{E}_{x\sim\mathcal{D}_{\mathrm{tr}}}[r(x,\pi)].$ The quantity $\Delta(\pi)$ is precisely the probability that $\pi$ is both structurally valid and strictly better than the baseline family on a random training instance. Hence, the role of training is to shape the model's policy-generation distribution so that policies with higher success probability $\Delta(\pi)$ are generated more often. To formalize this idea, 
let $\tau_{\mathrm{good}}\in(0,1)$ denote the  success-probability threshold to classify a policy as good during training. We define the \emph{training-stage good region} as
\begin{equation*}
\mathcal{G}^{\mathrm{tr}}
\;\triangleq\;
\left\{
\pi \in \Pi :
g(\pi)=1,\;
\Delta(\pi)\ge \tau_{\mathrm{good}}
\right\},
\label{eq:def_good_region}
\end{equation*}
and the effective upper envelope over the complement of the good region (i.e., bad region) as
\begin{equation*}
\Delta_{\mathrm{bad}}
\;\triangleq\;
\sup_{\pi \in \Pi \setminus \mathcal{G}^{\mathrm{tr}}} \Delta(\pi),
\qquad
\gamma
\;\triangleq\;
\tau_{\mathrm{good}}-\Delta_{\mathrm{bad}} > 0.
\label{eq:def_margin}
\end{equation*}
The margin $\gamma$ measures how much the good region exceeds the rest of the canonical policy space.

Let $p_k$ denote the proxy distribution over canonical policies induced by the model after the $k$-th reinforcement step, where raw code outputs are mapped through the canonicalization operator. Further define $\widehat{\Delta}_k(\pi)$ as the estimated success probability of policy $\pi$ at training step $k$, and $\eta>0$ be the update step size. 

In this paper, we train the model using the reinforcement learning algorithm Group Relative Policy Optimization (GRPO) \citep{shao2024deepseekmath}. A token-level analysis of the GRPO update is neither tractable nor aligned with the policy-generation quality relevant to our inventory analysis.\footnote{A token is the basic unit in a language model's output, such as a word piece, symbol, or code fragment. Since GRPO updates token probabilities over long code-generation trajectories, a token-level analysis would require tracking a high-dimensional model, token dependence, and the many-to-one mapping from raw code to economically equivalent inventory policies. This is not tractable and is not central to our objective.}. Therefore, we focus on the aggregate change in the proxy distribution over canonical policies induced by the trained model. In particular, training assigns more probability mass to policies with higher estimated success probabilities \citep{freund1997decision}:
\begin{equation}
p_{k+1}(\pi)
=
\frac{
p_k(\pi)\exp\{\eta \widehat{\Delta}_k(\pi)\}
}{
\sum_{\pi' \in \Pi}
p_k(\pi')\exp\{\eta \widehat{\Delta}_k(\pi')\}
}.
\label{eq:proxy_update}
\end{equation}

Equation~\eqref{eq:proxy_update} should be interpreted as an analytical proxy rather than an exact description of GRPO's token-level optimization\footnote{Equivalently, \eqref{eq:proxy_update} is the closed-form solution of the KL-regularized population policy-improvement problem
\[
p_{k+1}
\in
\arg\max_{q\in\Delta(\Pi)}
\left\{
\sum_{\pi\in\Pi} q(\pi)\widehat{\Delta}_k(\pi)
-
\frac{1}{\eta}\mathrm{KL}(q\|p_k)
\right\}.
\]}. It is consistent with the qualitative effect of policy-gradient reinforcement learning, namely that high-success policies are progressively upweighted. For analytical convenience, we assume a positive initial mass on the good region, i.e., $p_0(\mathcal{G}^{\mathrm{tr}})>0$. Accordingly, Theorem~\ref{thm:train_concentration} is a concentration result for the proxy distribution over canonical policies induced by equation \eqref{eq:proxy_update}. 

The following theorem shows that, over 
$K$ training steps, if the cumulative discrepancy between the estimated success signal and the true success probability is bounded by a constant $\varepsilon_K$, then the induced proxy distribution over canonical policies progressively concentrates on the structurally valid, high-success region. The concentration occurs at an exponential rate characterized by $K\gamma - 2\varepsilon_K$.

\begin{theorem}
\label{thm:train_concentration}
Assume there exist a training event $\mathcal{E}_{\mathrm{tr}}$ and a deterministic sequence $(\varepsilon_K)_{K\ge1}$ such that, on $\mathcal{E}_{\mathrm{tr}}$, for every $K\ge1$, $\sup_{\pi:\,p_0(\pi)>0}\left|\sum_{k=1}^{K}\bigl(\widehat{\Delta}_k(\pi)-\Delta(\pi)\bigr)\right|\le\varepsilon_K.$ 
Then, on $\mathcal{E}_{\mathrm{tr}}$, 
\begin{equation}
\frac{p_K(\Pi \setminus \mathcal{G}^{\mathrm{tr}})}{p_K(\mathcal{G}^{\mathrm{tr}})}
\le
\rho_K
\triangleq
\frac{p_0(\Pi \setminus \mathcal{G}^{\mathrm{tr}})}{p_0(\mathcal{G}^{\mathrm{tr}})}
\exp\!\bigl\{-\eta(K\gamma-2\varepsilon_K)\bigr\}.
\label{eq:train_ratio_bound}
\end{equation}
Consequently,
\begin{equation}
p_K(\mathcal{G}^{\mathrm{tr}})\ge \frac{1}{1+\rho_K},
\qquad
p_K(\Pi\setminus\mathcal{G}^{\mathrm{tr}})\le \frac{\rho_K}{1+\rho_K}.
\label{eq:train_bad_mass_bound}
\end{equation}
\end{theorem}

Theorem~\ref{thm:train_concentration} provides 
a lower bound on $p_K(\mathcal{G}^{\mathrm{tr}})$, the probability, under the trained policy-generation distribution, of generating a policy in the training-stage good-policy region. 
In the next section, building on this result, 
we develop an inference framework based on replay evaluation and confidence screening and establish its deployment guarantees. 

\section{Inference and Deployment Analysis}
\label{sec:inference_deployment}

In this section, we move from offline training to a periodic decision-making setting. Specifically, we use historical information collected by the decision maker to replay candidate policies. At each round, when the model generates new policies,  replay performance statistics and their associated confidence intervals are used to guide policy selection.

In Section~\ref{subsec:replay_eval_confidence}, we first introduce the key replay statistics and certification score functions used in the analysis. In Section~\ref{subsec:llm_guided_search_setup}, we present the inference framework together with the underlying analytical setting. In Section~\ref{subsec:theoretical_guarantees}, we establish the theoretical guarantees of the proposed framework.

\subsection{Replay Statistics and Certification Scores}
\label{subsec:replay_eval_confidence}

In current inventory systems, a decision maker typically has access to historical demand and operational data collected over multiple periods, such as daily sales, add-to-cart counts, and stockout records. These historical observations provide a replay window. Although only one policy is deployed and observed at each time point, the availability of replayable data allows us to conduct counterfactual evaluation for policies that were not deployed, thereby estimating their potential performance along the observed sample paths.


\noindent\textbf{Replay Evaluation.} Fix a period $t$. At the beginning of the period, the system has access to a mature replay window $W_t = \{\omega_{t,1},\ldots,\omega_{t,m_t}\},$ 
where each $\omega_{t,\ell}$ denotes a replayable sample path.  For theoretical clarity, 
we use the same replay sample size 
$m_t$ for all evaluated pairs in period $t$. 
For a replayable sample path $\omega$, let $C_t(\pi;\omega)$ denote the cumulative cost incurred by policy $\pi$ when it is replayed on path $\omega$ over the period-$t$ evaluation window. For a candidate policy $\pi$ and a reference policy (comparator) $\tilde\pi$, define the sample-path gain (cost reduction) as 
$$
Z_{t,\ell}(\pi \mid \tilde\pi)
\;\triangleq\;
C_t(\tilde\pi;\omega_{t,\ell}) - C_t(\pi;\omega_{t,\ell}),
\qquad
\ell=1,\ldots,m_t,
$$
where $m_t$ is the number of replayable paths used for each evaluated pair in period $t$. The mean of replay gain (replay mean) is defined as 
\begin{equation}
\widehat{\mu}_t(\pi \mid \tilde\pi)
\triangleq
\frac{1}{m_t}
\sum_{\ell=1}^{m_t}
Z_{t,\ell}(\pi \mid \tilde\pi). 
\label{eq:def_mu_hat}
\end{equation}
The empirical variance of replay gains (replay variance) is defined as 
\begin{equation}
\widehat{v}_t(\pi\mid\tilde\pi)
\triangleq
\frac{1}{m_t}\sum_{\ell=1}^{m_t}
\Bigl(Z_{t,\ell}(\pi\mid\tilde\pi)-\widehat{\mu}_t(\pi\mid\tilde\pi)\Bigr)^2.
\label{eq:def_v_hat}
\end{equation}

\noindent\textbf{Evaluation Statistics.} 
To construct confidence bounds, let $\mathcal P_t^{\mathrm{eval}}$ denote the set of policy--comparator pairs evaluated in period $t$. This set includes both pre-loop baseline evaluations and inner-loop candidate evaluations. We provide  an upper bound on the total number of comparisons performed by InvEvolve in period $t$.
Specifically, before the inner loop, at most $|\mathcal{A}_{t,0}| - 1$ policies are evaluated against the reference policy $\pi_t^{\mathrm{ref}}$. 
During the inner loop, each of the $J$ 
iterations contributes at most two additional comparisons, namely between $(\pi_{t,j}, \pi_t^{\mathrm{ref}})$ and between $(\pi_{t,j}, \pi_{t,j-1}^{\mathrm{ch}})$. 
Therefore, the total number of evaluated pairs in period $t$ is bounded by 
$N_t \triangleq (|\mathcal{A}_{t,0}| - 1) + 2J,$ and the realized set 
satisfies $|\mathcal{E}_t| \le N_t$. Let $\delta_t\in(0,1)$ be the period-level confidence budget used for simultaneous replay certification in period $t$. For the canonical bounded-gain implementation, when $|Z_{t,\ell}(\pi\mid\tilde\pi)|\le B_t$, we use the Hoeffding-style radius \citep{hoeffding1963}
\begin{equation}
\operatorname{rad}_t(\pi\mid\tilde\pi)
\triangleq
B_t\sqrt{\frac{2\log(2N_t/\delta_t)}{m_t}}.
\label{eq:def_radius}
\end{equation}
More generally, $\operatorname{rad}_t(\pi\mid\tilde\pi)$ can be any valid replay radius satisfying the confidence-event condition in Assumption~\ref{ass:adaptive_ci}.

Together with equation \eqref{eq:def_mu_hat}, 
this yields the classical upper and lower confidence bounds
\begin{equation}
\mathrm{UCB}_t(\pi \mid \tilde\pi)
\triangleq
\widehat{\mu}_t(\pi \mid \tilde\pi)+\operatorname{rad}_t(\pi \mid \tilde\pi),
\quad
\mathrm{LCB}_t(\pi \mid \tilde\pi)
\triangleq
\widehat{\mu}_t(\pi \mid \tilde\pi)-\operatorname{rad}_t(\pi \mid \tilde\pi).
\label{eq:ucb_lcb}
\end{equation}

The lower confidence bound is used to certify  conservative improvement, while the upper bound is used for optimistic selection among candidates that have already passed the safety screen. 
By convention, the reference policy has zero gain relative to itself:
$
V_t^{\mathrm{rep}}(\pi_t^{\mathrm{ref}}\mid \pi_t^{\mathrm{ref}})=
\widehat{\mu}_t(\pi_t^{\mathrm{ref}}\mid \pi_t^{\mathrm{ref}})=
\operatorname{rad}_t(\pi_t^{\mathrm{ref}}\mid \pi_t^{\mathrm{ref}})=0.
$ 
Hence, $\mathrm{LCB}_t(\pi_t^{\mathrm{ref}}\mid \pi_t^{\mathrm{ref}})=\mathrm{UCB}_t(\pi_t^{\mathrm{ref}}\mid \pi_t^{\mathrm{ref}})=0.$

Recall that replay performance and deployment performance may differ since they are evaluated under different path distributions. Let $P_t^{\mathrm{rep}}$ and $P_t^{\mathrm{dep}}$ denote, respectively, replay and deployment path distributions in period $t$. For each pair $(\pi,\tilde\pi)$ and a sample path $\omega$, define
\[
z_{t,\pi,\tilde\pi}(\omega):=C_t(\tilde\pi;\omega)-C_t(\pi;\omega),
\]
and assume $z_{t,\pi,\tilde\pi}\in\mathfrak F_t$ for a function class $\mathfrak F_t$. Define the replay-to-deployment mismatch as 
\begin{equation}
\xi_t
\triangleq
\sup_{f\in\mathfrak F_t}
\left|
\mathbb{E}_{P_t^{\mathrm{dep}}}[f]-\mathbb{E}_{P_t^{\mathrm{rep}}}[f]
\right|.
\label{eq:gap_assumption}
\end{equation}

This quantity is an integral probability metric (IPM) over $\mathfrak F_t$, which measures the worst-case discrepancy between replay and deployment expectations across the gain functions. 
Hence, for all relevant $(\pi,\tilde\pi)$,  $\left|V_t^{\mathrm{dep}}(\pi\mid \tilde\pi)-V_t^{\mathrm{rep}}(\pi\mid \tilde\pi)\right|\le\xi_t.$

\subsection{LLM-Guided Policy Search and Analytical Setup}
\label{subsec:llm_guided_search_setup}

The central inference mechanism in InvEvolve is an inner-loop procedure that integrates policy  generation and certification. We refer to this inner loop as the \emph{LLM-guided Policy Search} (LGPS) procedure. The key principle is that the LLM does not generate all candidate policies at once. 
Instead, each generated candidate is conditioned on the outcomes of previous replay evaluations and previous accept-or-reject decisions.

At the beginning of period $t$, recall the initial active candidate set $\mathcal{A}_{t,0}\triangleq\Pi^{\mathrm{base}} \cup \{d_{t-1}\}.$ Before the search loop starts, all policies in $\mathcal{A}_{t,0}$ are replay-evaluated against the period-specific reference baseline $\pi_t^{\mathrm{ref}}\in \Pi^{\mathrm{base}}$. Define the safety-feasible initial subset $\mathcal{A}_{t,0}^{\mathrm{feas}}
\triangleq
\left\{
\pi\in\mathcal{A}_{t,0}:
\mathrm{LCB}_t(\pi\mid\pi_t^{\mathrm{ref}})\ge\xi_t
\right\}$. Then initialize the current champion as
\[
\pi_{t,0}^{\mathrm{ch}}
\in
\arg\max_{\pi\in \mathcal{A}_{t,0}^{\mathrm{feas}}\cup\{\pi_t^{\mathrm{ref}}\}}
\mathrm{UCB}_t(\pi\mid\pi_t^{\mathrm{ref}}).
\]

Thus, the initial champion is the strongest safety-certified fallback policy among the baseline family and the incumbent, with $\pi_t^{\mathrm{ref}}$ as the default fallback if no other initial candidate is safety-feasible. Here $\pi_{t,j}^{\mathrm{ch}}$ denotes the current \emph{champion} after the $j$-th inner-loop round, and $\mathcal{A}_{t,j}$ denotes the active candidate pool.

Let $\mathcal H_{t,j-1}$ denote the round-$(j-1)$ prompt context available to the LLM within period $t$. It contains the period-level summary $\mathcal H_t$ together with the replay summaries, previous decisions, and feedback accumulated up to round $j-1$.

For each round $j=1,\ldots,J$, the LLM produces a candidate policy according to
\begin{equation}
\pi_{t,j}
\sim
Q_{\theta}\!\left(
\cdot
\mid
\mathcal{H}_{t,j-1},
\pi_{t,j-1}^{\mathrm{ch}},
\mathcal{A}_{t,j-1},
\mathrm{stats}_{t,j-1}
\right),
\label{eq:def_policy_generation_dist}
\end{equation}
where $\mathrm{stats}_{t,j-1}$ contains all replay summaries, confidence scores, and previous decisions up to round $j-1$. Note that the policy-generation distribution changes as the algorithm learns, within the same period, which directions appear promising and which failure modes should be avoided. For notational simplicity, we write
\[
Q_{t,j}
:=
Q_{\theta}\!\left(
\cdot
\mid
\mathcal{H}_{t,j-1},
\pi_{t,j-1}^{\mathrm{ch}},
\mathcal{A}_{t,j-1},
\mathrm{stats}_{t,j-1}
\right)
\]
for the round-$j$ policy-generation distribution conditional on the information available before generating $\pi_{t,j}$.

For the round-$j$ generated candidate $\pi_{t,j}$, we compute the certification scores defined as follows. 
\begin{align*}
S_t(\pi)
\triangleq
\mathrm{LCB}_t(\pi \mid \pi_t^{\mathrm{ref}}),
\quad
I_{t,j}(\pi)
\triangleq
\mathrm{LCB}_t(\pi \mid \pi_{t,j-1}^{\mathrm{ch}}),
\quad
O_{t,j}(\pi)
\triangleq
\mathrm{UCB}_t(\pi \mid \pi_{t,j-1}^{\mathrm{ch}}).
\end{align*}

The score $S_t$ measures certified safety relative to the period-specific reference baseline, $I_{t,j}$ measures certified improvement relative to the current champion, and $O_{t,j}$ measures the optimistic upside against the current champion. The current champion may be initialized from the baseline family or the incumbent, and may later be updated by promoted generated candidates. Intuitively, the first two scores drive promotion decisions, while the third captures the potential upside relative to the current champion.

For each period $t$, let $\{\mathcal{F}_{t,j}\}_{j=0}^J$ denote the natural filtration of the period-$t$ InvEvolve inner-loop process. 
$\mathcal{F}_{t,0}$ contains the realized period context $\mathcal{H}_t$, the replay window $W_t$, the reference baseline $\pi_t^{\mathrm{ref}}$, the initial active pool $\mathcal{A}_{t,0}$, and all pre-loop replay statistics used to construct $\mathcal{A}_{t,0}^{\mathrm{feas}}$ and $\pi_{t,0}^{\mathrm{ch}}$. For $j\ge 1$, $\mathcal{F}_{t,j}$ augments $\mathcal{F}_{t,j-1}$ by the round-$j$ generated candidate $\pi_{t,j}$, the replay observations used to evaluate $\pi_{t,j}$, and the resulting gate, champion-update, and pool-update decisions.

Fix an improvement threshold $\varepsilon>0$. The candidate promotion rule is as follows.

\noindent
\textbf{Pool update rule.}
For every round $j=1,\ldots,J$, after evaluating $\pi_{t,j}$, update
\[
\mathcal{A}_{t,j} \leftarrow \mathcal{A}_{t,j-1}\cup\{\pi_{t,j}:g(\pi_{t,j})=1\}.
\]

\noindent
\textbf{Champion update rule.}
If
\[
S_t(\pi_{t,j}) \ge \xi_t,
\qquad
I_{t,j}(\pi_{t,j}) \ge \varepsilon+\xi_t,
\]
then the candidate is accepted and promoted, $\pi_{t,j}^{\mathrm{ch}} \leftarrow \pi_{t,j}.$ Otherwise, $\pi_{t,j}^{\mathrm{ch}} \leftarrow \pi_{t,j-1}^{\mathrm{ch}}.$ 
Because the initial champion is selected from the strongest safety-certified baseline, passing the improvement gate means improving relative to the strongest certified fallback available at period start.

The final active pool at period $t$ is $\mathcal{A}_t:=\mathcal{A}_{t,J}$. Accordingly, the final safety-feasible set is defined as $\mathcal{A}_t^{\mathrm{feas}} := \{\pi \in \mathcal{A}_t : \mathrm{LCB}_t(\pi \mid \pi_t^{\mathrm{ref}}) \ge \xi_t\}$. At the end of the period, the system deploys
\begin{equation}
d_t
:=
\begin{cases}
\arg\max_{\pi\in \mathcal{A}_t^{\mathrm{feas}}}
\mathrm{UCB}_t(\pi \mid \pi_t^{\mathrm{ref}}),
& \mathcal{A}_t^{\mathrm{feas}}\neq\varnothing,\\[4pt]
\pi_t^{\mathrm{ref}},
& \mathcal{A}_t^{\mathrm{feas}}=\varnothing.
\end{cases}
\label{eq:deployment_rule}
\end{equation}

If no generated candidate is ultimately selected, deployment uses the best safety-certified policy already available in the active pool, which may be a baseline policy or the incumbent. In particular, $\pi_t^{\mathrm{ref}}$ remains the explicit fallback when no other candidate is certified.
\begin{figure}[!htbp]
\centering
\includegraphics[width=\textwidth]{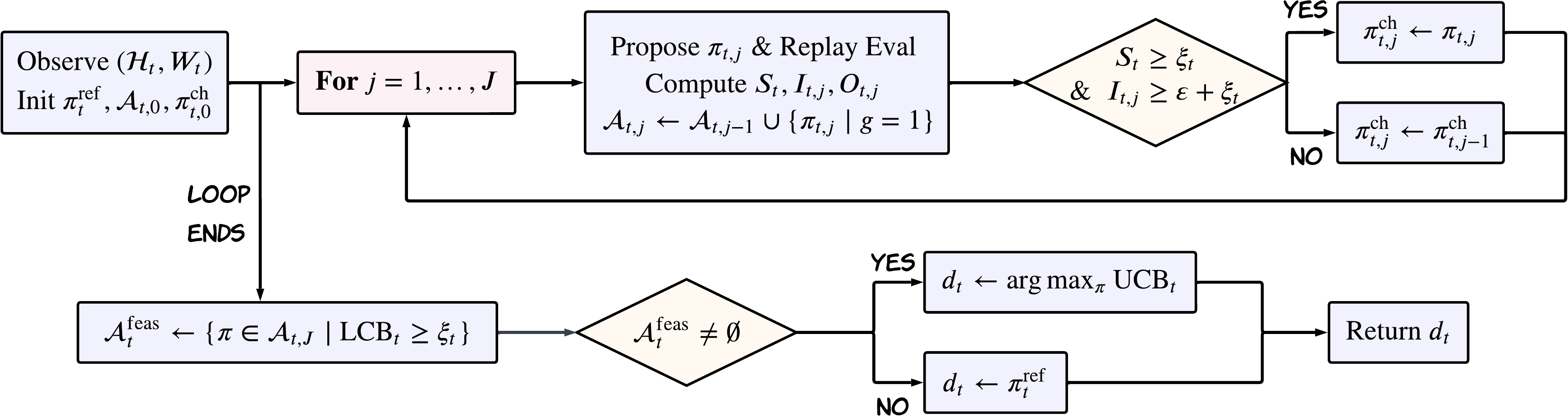}
\caption{Inference workflow in period $t$: after forming the fallback pool and a safety-certified initial champion, InvEvolve iteratively generates candidates, retains all structurally valid evaluated policies in the active pool, promotes only gate-passing policies, and finally deploys a safety-screened policy.}
\label{fig:inference_flowchart}
\end{figure}

\subsection{Analysis Setting and Theoretical Guarantees}
\label{subsec:theoretical_guarantees}

\noindent\textbf{Regularity Conditions for Analysis.} We introduce three assumptions that establish regularity conditions for the LLM-guided search: (i) replay-to-deployment transfer, (ii) validity of adaptive replay confidence, and (iii) effective search coverage. Together, these assumptions make replay-certified LLM deployment a well-posed statistical decision problem, and they can be justified in certain settings.

The first assumption does not require replay and deployment to coincide. It only requires that,
for every pair of policies, the cost difference observed in historical replay is bounded by the IPM-based quantity defined in \eqref{eq:gap_assumption}.

\begin{assumption}[Bounded replay--deployment mismatch]
\label{ass:gap}
For each period $t$, the bound induced by IPM in \eqref{eq:gap_assumption} holds for every policy--comparator pair that enters certification or deployment.
\end{assumption}

Assumption~\ref{ass:gap} can be interpreted as a distributionally robust condition.  $\xi_t$ is the radius of an ambiguity set around the replay distribution, and the replay-based certificate is required to remain valid under all deployment distributions within that set. Appendix~\ref{app:xi-estimation} provides operational estimators based on historical quantile calibration, shift-conditioned quantile regression, and IPM plug-in constructions.

Our analysis also involves statistical confidence intervals constructed for a sequence of adaptively generated candidate policies. To ensure that the certificates produced during the replay process are themselves statistically reliable, we define the replay confidence event as $\mathcal{C}_t:=
\bigcap_{(\pi,\tilde\pi)\in\mathcal{E}_t}\left\{\left|
\widehat{\mu}_t(\pi\mid\tilde\pi)-V_t^{\mathrm{rep}}(\pi\mid\tilde\pi)\right|\le\operatorname{rad}_t(\pi\mid\tilde\pi)
\right\},$ and impose the following assumption.

\begin{assumption}[Adaptive replay confidence validity]
\label{ass:adaptive_ci}
For each period $t$, $\mathbb{P}(\mathcal{C}_t)\ge 1-\delta_t.$
\end{assumption}

Assumption~\ref{ass:adaptive_ci} requires that all replay estimates queried by the algorithm are simultaneously accurate with high probability. This type of high-probability coverage argument is commonly used in UCB-style bandit analysis and adaptive sequential decision making \citep{auer2002finite,abbasi2011online,dwork2015preserving,howard2021time}. A common proof strategy in these works is to first construct a high-probability event under which the empirical means of all relevant arms lie within their confidence intervals, and then show that, on this event, optimistic selection is not systematically distorted by estimation error. Appendix~\ref{app:adaptive-ci-verification} shows that Assumption~\ref{ass:adaptive_ci} follows from the conditional sampling and bounded-gain conditions used in the Hoeffding construction. The Hoeffding radius provides one convenient formal certificate.

We next define the target policy classes that the LLM-guided search must be able to reach. The first target policy class supports single-period promotion, while the second supports rolling comparison with the near-oracle safe policy class. To formalize the promotable policy class, define the replay safety margin and improvement margin as follows:
\[
\begin{aligned}
\Delta_t^{\mathrm{safe}}(\pi)
&\triangleq
V_t^{\mathrm{rep}}(\pi\mid\pi_t^{\mathrm{ref}})
-\xi_t
-2\,\operatorname{rad}_t(\pi\mid\pi_t^{\mathrm{ref}}),
\\
\Delta_{t,j}^{\mathrm{imp}}(\pi)
&\triangleq
V_t^{\mathrm{rep}}(\pi\mid\pi_{t,j-1}^{\mathrm{ch}})
-\varepsilon-\xi_t
-2\,\operatorname{rad}_t(\pi\mid\pi_{t,j-1}^{\mathrm{ch}}).
\end{aligned}
\]

A candidate policy is promotable in replay when both margins are nonnegative. Formally, define
\[
\mathcal{G}^{\mathrm{prom}}_{t,j}(\varepsilon)
\triangleq
\left\{\pi\in\Pi:
g(\pi)=1,
\Delta_t^{\mathrm{safe}}(\pi)\ge 0,
\Delta_{t,j}^{\mathrm{imp}}(\pi)\ge 0
\right\}.
\]
The promotable policy class $\mathcal{G}_{t,j}^{\mathrm{prom}}(\varepsilon)$ connects replay-side certification with gate correctness.

For rolling deployment, define the certifiably safe policy class as
\[
\Pi_t^{\mathrm{safe}}
=
\{\pi_t^{\mathrm{ref}}\}\cup
\left\{\pi\in\Pi:g(\pi)=1,\ \Delta_t^{\mathrm{safe}}(\pi)\ge0\right\}. 
\]
This policy class contains the reference policy and all structurally valid policies that are safe relative to it. The oracle-safe benchmark, i.e., the best deployment gain achievable within this certifiably safe policy class, is
\[
V_t^{\mathrm{safe},*}
=
\sup_{\pi\in\Pi_t^{\mathrm{safe}}}V_t^{\mathrm{dep}}(\pi\mid\pi_t^{\mathrm{ref}}).
\]

For a tolerance $\nu_t\ge0$, define the near-oracle safe policy class as
\[
\Pi_t^{\mathrm{safe},\nu_t}=\left\{\pi\in\Pi_t^{\mathrm{safe}}:V_t^{\mathrm{safe},*}-V_t^{\mathrm{dep}}(\pi\mid\pi_t^{\mathrm{ref}})\le\nu_t\right\}.
\]


Rather than imposing separate structural assumptions on the internal LLM policy-generation mechanism, we summarize the required search capability through a single effective coverage condition.

\begin{assumption}[Effective search coverage]
\label{ass:effective_coverage}
Condition on any training realization in $\mathcal{E}_{\mathrm{tr}}$. For each period $t$ and round $j$, conditional on $\mathcal{F}_{t,j-1}$, there exist constants $q_t,\bar q_t\in[0,1]$ such that, almost surely  $Q_{t,j}\bigl(\mathcal{G}^{\mathrm{prom}}_{t,j}(\varepsilon)\bigr)
\ge q_t,
Q_{t,j}\bigl(\Pi_t^{\mathrm{safe},\nu_t}\bigr)
\ge \bar q_t .$
\end{assumption}

Assumption~\ref{ass:effective_coverage} is a coverage condition of policy search. It does not require the LLM to deterministically identify a good policy in every round of replay evaluation. It only requires that the policy-generation distribution assign nonzero mass to the target policy classes that the certification rule attempts to discover. This condition is analogous to positive-mass conditions in many-armed bandits and randomized search heuristics, where the learner samples policies from a reservoir and the difficulty of identifying a good policy is governed by the reservoir mass of good or optimal policies \citep{deheide2021bandits,solis1981minimization,auger2011theory}. 

Appendix~\ref{app:effective-coverage-sufficient} provides one set of sufficient conditions, tied to the LLM training and inference procedure, under which Assumption~\ref{ass:effective_coverage} holds. The resulting probability bounds are informative when $q_t>0$ or $\bar q_t>0$. If the relevant target policy class is empty, the corresponding coverage parameter $q_t$ or $\bar q_t$ may be zero. In that case, the theory reduces to a safety/fallback statement rather than an improvement guarantee, because an improvement or near-oracle target cannot be discovered when the corresponding target policy class is empty.

\begin{remark}[A canonical stationary case]
In a classical inventory setting with a stationary demand distribution, as studied later in Section~\ref{sec:exp-cbs}, these assumptions can hold naturally. First, Assumption~\ref{ass:gap} is satisfied at the population level. When replay and deployment are conducted under the same distribution, the replay--deployment discrepancy is naturally zero. Second, Assumption~\ref{ass:adaptive_ci} can be enforced by evaluating policies on independent simulation replications, or more generally by using any replay radius that yields a valid replay confidence event. Third, Assumption~\ref{ass:effective_coverage} is supported ex post by the policies discovered in Section~\ref{sec:exp-cbs}. 
Numerical experiments of these policies provide evidence that the policy-generation mechanism assigns non-negligible mass to the target policy classes appearing in Assumption~\ref{ass:effective_coverage}, namely the promotable policy class $G^{\mathrm{prom}}_{t,j}(\varepsilon)$ and the near-oracle safe policy class
$\Pi^{\mathrm{safe},\nu_t}_t$.
\end{remark}

Table~\ref{tab:theorem_assumption_breakdown} summarized where the assumptions are used in the following results. 

\begin{table}[!htbp]
\centering
\caption{Breakdown of the main results by assumption usage}
\label{tab:theorem_assumption_breakdown}
\small
\renewcommand{\arraystretch}{1.12}
\setlength{\tabcolsep}{4pt}
\begin{tabular}{p{0.22\textwidth} c c c p{0.48\textwidth}}
\toprule
Result & A\ref{ass:gap} & A\ref{ass:adaptive_ci} & A\ref{ass:effective_coverage} & Main message \\
\midrule
Theorem~\ref{thm:train_concentration} & -- & -- & -- & Training concentrates probability mass on $\mathcal{G}^{\mathrm{tr}}$. \\

Replay confidence event $\mathcal{C}_t$ & -- & \ding{51} & -- & Replay estimates for adaptively queried pairs are simultaneously accurate. \\
Lemma~\ref{lem:gate_correctness} & \ding{51} & \ding{51} & -- & Replay-certified safety and improvement transfer to deployment. \\

Theorem~\ref{thm:single_month_promotion} & \ding{51} & \ding{51} & \ding{51} & Effective coverage plus gate correctness yields finite-round certified promotion. \\

Theorem~\ref{thm:rolling_regret} & \ding{51} & \ding{51} & \ding{51} & Effective coverage of the near-oracle safe policy class yields rolling safety and an oracle-safe deployment-gap bound. \\
\bottomrule
\end{tabular}
\end{table}

Before presenting the main results, we first establish a useful intermediate property.

\begin{lemma}[Gate correctness under replay confidence and IPM transfer]
\label{lem:gate_correctness}
Fix period $t$ and round $j$, and suppose the period-level replay confidence event $\mathcal{C}_t$ holds. If a candidate $\pi$ belongs to $\mathcal{G}^{\mathrm{prom}}_{t,j}(\varepsilon)$,
then the candidate promotion rule promotes $\pi$, and the promoted policy satisfies
\[
V_t^{\mathrm{dep}}(\pi \mid \pi_t^{\mathrm{ref}})\ge 0,
\qquad
V_t^{\mathrm{dep}}(\pi \mid \pi_{t,j-1}^{\mathrm{ch}})\ge \varepsilon.
\]
\end{lemma}


Lemma~\ref{lem:gate_correctness} shows that any candidate in the promotable policy class $G^{\mathrm{prom}}_{t,j}(\varepsilon)$ passes the promotion gates whenever the replay confidence event $\mathcal C_t$ holds. Specifically, the margin conditions $\Delta^{\mathrm{safe}}_t(\pi)\ge 0$ and $\Delta^{\mathrm{imp}}_{t,j}(\pi)\ge 0$ ensure that $S_t(\pi)\ge \xi_t$ and $I_{t,j}(\pi)\ge \varepsilon+\xi_t$. The IPM transfer bound then converts these replay-side certificates into deployment guarantees, ensuring that the promoted policy is safe relative to $\pi^{\mathrm{ref}}_t$ and achieves at least an $\varepsilon$ improvement over the previous champion.

\noindent\textbf{Main Results.} 
We introduce Theorem \ref{thm:single_month_promotion} and \ref{thm:rolling_regret} to address two main questions, respectively. First, within a single-period loop, with what probability can the LLM identify a policy that outperforms the current champion while remaining safe relative to the baseline policy? Second, under future deployment, how far is the policy selected by the inference framework from the truly best-performing policy within the certifiably safe policy class?


\begin{theorem}[Single-period certified promotion]
\label{thm:single_month_promotion}
Fix any period $t$ and improvement threshold $\varepsilon>0$, and condition on any 
training realization in $\mathcal{E}_{\mathrm{tr}}$. Under Assumptions~\ref{ass:gap}, \ref{ass:adaptive_ci}, and \ref{ass:effective_coverage},
define 
\[
\mathcal{E}^{\mathrm{prom}}_t(\varepsilon)
:=
\left\{
\begin{aligned}
&\exists j\le J:\ \pi_{t,j}\ \text{is promoted,}
\\
&V_t^{\mathrm{dep}}(\pi_{t,j}\mid\pi_t^{\mathrm{ref}})\ge 0,
\quad V_t^{\mathrm{dep}}(\pi_{t,j}\mid \pi_{t,j-1}^{\mathrm{ch}})\ge\varepsilon
\end{aligned}
\right\},
\]
then 
$\mathbb{P}\!\left(\mathcal{E}^{\mathrm{prom}}_t(\varepsilon)\right)\ge1-\delta_t-(1-q_t)^J.$
\end{theorem}

Theorem~\ref{thm:single_month_promotion} states that if replay certificates are statistically valid and transferable, and if each search round has probability at least $q_t$ of generating a candidate in the promotable policy class, then a certifiable promotion occurs within $J$ search rounds with probability at least $1-\delta_t-(1-q_t)^J$. The quantity $q_t$ summarizes the effective search difficulty of the period: larger $q_t$ or larger $J$ increases the chance of discovering a gate-passing improvement. Appendix~\ref{app:effective-coverage-sufficient} relates $q_t$ back to training concentration, contextual generation shift, and overlap with the training-good region. Moreover, since the initial champion is the strongest safety-certified baseline fallback policy, period-level promotion is benchmarked against a stronger initial comparator. Therefore, the improvement guaranteed by this theorem is nontrivial.

Theorem~\ref{thm:single_month_promotion} provides a single-period guarantee. 
It establishes certified improvement within a fixed period, but it does not control long-horizon performance relative to an oracle over the full certifiably safe policy class. 
We thus introduce Theorem \ref{thm:rolling_regret}
to connect the search procedure to the oracle-safe benchmark and bounds the gap to $V_t^{\mathrm{safe},*}$. 

Recall the period-level replay confidence event $\mathcal{C}_t$. We define
\[
\mathcal{D}_t := \{\exists j \le J : \pi_{t,j} \in \Pi_t^{\mathrm{safe},\nu_t}\}, 
\quad
G_T := \bigcap_{t=1}^T (\mathcal{C}_t \cap \mathcal{D}_t).
\]
Here, $\mathcal{D}_t$ denotes a discovery event that at least one candidate in period $t$ lies in the near-oracle safe policy class. This discovery event is sufficient but not necessary for finding such a policy: the initial active pool may already contain a near-oracle safe baseline or incumbent before the search loop begins. We focus on generated-policy discoveries because they are directly controlled by the effective coverage parameter $\bar q_t$. Therefore,  $G_T$  ensures that, over the entire horizon, both statistical validity and successful discovery of policies in the near-oracle safe policy class hold simultaneously. Under this event, we establish the following property:

\begin{theorem}[Rolling safe deployment and oracle-safe dynamic deployment gap]
\label{thm:rolling_regret}
Fix a horizon $T$, and condition on any training realization in $\mathcal{E}_{\mathrm{tr}}$, under Assumptions~\ref{ass:gap}, \ref{ass:adaptive_ci}, and \ref{ass:effective_coverage}, 
\[
\mathbb{P}(G_T)
\ge
1-\sum_{t=1}^T\left[\delta_t+(1-\bar q_t)^J\right].
\]
Moreover, on $G_T$, for any $\mathcal{F}_{t,J}$-measurable selector $\tilde{\pi}_t \in \Pi_t^{\mathrm{safe},\nu_t} \cap \mathcal{A}_t,$ let $d_t$ be the policy selected by the deployment rule in \eqref{eq:deployment_rule} at period $t$, we have
\[
V_t^{\mathrm{dep}}(d_t\mid\pi_t^{\mathrm{ref}})\ge 0,
\quad\text{for all } t=1,\ldots,T,
\quad
\text{and}
\quad
\sum_{t=1}^T
\Bigl(
V_t^{\mathrm{safe},*}-V_t^{\mathrm{dep}}(d_t\mid\pi_t^{\mathrm{ref}})
\Bigr)
\le
\sum_{t=1}^T\Gamma_t(\tilde\pi_t,d_t).
\]
where $\Gamma_t(\tilde\pi_t,d_t):=\nu_t+2\,\operatorname{rad}_t(\tilde\pi_t\mid\pi_t^{\mathrm{ref}})+2\,\operatorname{rad}_t(d_t\mid\pi_t^{\mathrm{ref}})+2\xi_t.$
\end{theorem}

Theorem~\ref{thm:rolling_regret} shows that, under the stated assumptions, with high probability, the policy deployed in each period does not perform worse than the reference baseline. 
It also bounds the cumulative gap to the best policy in the certifiably safe policy class over the horizon by 
$\sum_{t=1}^T \Gamma_t\left(\tilde{\pi}_t, d_t\right)$. This bound consists of four terms, which correspond to (1) the tolerance $\nu_t$ of the near-oracle safe policy class, (2) the statistical error in the replay estimate of $\tilde{\pi}_t$, (3) the statistical uncertainty of the final deployed policy $d_t$, and (4) the distribution shift error from replay to deployment.

The tolerance $\nu_t$ measures the approximation level of the near-oracle safe policy class. The probability of discovering such a policy is controlled by $\bar q_t$ and $J$: stronger generation quality or a larger search budget increases the probability that $\mathcal{D}_t$ occurs. Conditional on this event, the realized deployment gap is bounded by the statistical radii, the transfer margin $\xi_t$, and the approximation tolerance $\nu_t$. Larger replay sample sizes reduce the statistical uncertainty terms. The final term arises from Assumption~\ref{ass:gap} and can be interpreted as a distributionally robust ambiguity parameter.

We summarize the main theoretical thread linking 
the three stages of the framework: training, inference, and deployment. Training improves generation quality by concentrating the model's policy-generation distribution on $\mathcal{G}^{\mathrm{tr}}$ (Theorem \ref{thm:train_concentration}). Inference converts this improved generation behavior into within-period certified improvements (Theorem \ref{thm:single_month_promotion}). Deployment guarantees rolling safety and dynamic performance relative to the oracle-safe policy class (Theorem \ref{thm:rolling_regret}).

\section{Operational Framework}
\label{sec:experiments}

In this section, we describe the experimental setup. 
Section~\ref{subsec:synthetic_data} presents the construction of the training data, and Section~\ref{sec:exp-setup} presents the training procedure and implementation settings.

\subsection{Synthetic Seed Data Generation}
\label{subsec:synthetic_data}

The design of the synthetic data directly affects the training performance of the LLM. 
We therefore construct data that closely resembles real-world scenarios. 
Our synthetic data construction consists of two steps: we first construct seed data and then apply a slicing strategy to expand the dataset.

\subsubsection{Seed Data Construction} In real-world inventory settings, demand data often include both demand information and features such as holidays, weather conditions, and promotional activities. These features are not limited to numerical values and may also include textual information. For example, unexpected events are often recorded in text format. 
we construct a sparse textual feature, denoted as \textit{note}, to simulate textual records under unexpected events.

Our data synthesis follows a structured design. Each dataset is constructed from timestamps and demand values, where demand is generated as a function of three components: numerical features, a textual feature denoted by \textit{note}, and unobserved latent variables. The numerical features and the textual field \texttt{note} are observed contextual variables available to the decision maker. Consistent with the preceding theoretical framework, these variables are part of the period-level information set \(\mathcal H_t\) and are therefore included in the algorithmic filtration \(\mathcal F_{t,0}\). In contrast, the latent variables used in data generation are not observed by the algorithm. They affect the realized demand process and the resulting costs, but they are not measurable with respect to the filtration available to the LLM.

We construct a family of \textbf{47 synthetic seed datasets}. These datasets cover a range of scenarios, including consumer products, industrial components, medical supplies, and regional power load. They are designed to simulate contextual demand patterns, distribution shifts, and rare disruptions observed in real-world settings. Note that in the synthetic data, a \textit{note} is recorded only on the first day of an event window, although the event may continue to affect demand throughout that window.

Specifically, each dataset is indexed by daily timestamps from January 1, 2024 to December 31, 2025. For dataset $j$ and calendar day $r$, we generate an observed context vector $x_{j,r}$, an optional textual note $n_{j,r}$, an unobserved latent state $z_{j,r}$, and a realized demand $y_{j,r}$.  The context vector is dataset-specific and may include calendar features, weather variables, and promotion intensity. We do not impose a unified feature schema across datasets. Instead, the set of covariates is tailored to the operational characteristics of each SKU or demand entity.

Demand is generated from a time-varying distribution, $y_{j,r} \sim \mathcal{D}_{j,r}\bigl(\theta_{j,r}\bigr)$, where both the distribution family $\mathcal{D}_{j,r}$ and its parameter vector $\theta_{j,r}$ may vary across datasets and over time. The conditional parameter vector is defined as $\theta_{j,r} = g_j\!\left(x_{j,r}, z_{j,r}, e_{j,r}, r\right)$, where $e_{j,r}$ is a latent event-state vector related to $n_{j,r}$. 

The variable $e_{j,r}$ captures the persistent effects of disruptive events, such as typhoons, strikes, or rumors. On ordinary days, $e_{j,r}$ is typically zero. However, once an event occurs and is recorded in $n_{j,r}$, $e_{j,r}$ may remain nonzero for several subsequent days, even if no additional note is observed. Therefore, $n_{j,r}$ serves as an imperfect observable proxy for the latent event-state process $e_{j,r}$, rather than a one-to-one representation.

We focus on count-valued demand and 
model the conditional mean $\mu_{j,r} := \mathbb{E}[\,y_{j,r}\mid x_{j,r}, z_{j,r}, e_{j,r}\,]$. The logarithm of the conditional mean is specified as 
\[
\log \mu_{j,r} = \beta_{j,0}(r)+\beta_j(r)^\top x_{j,r}+\gamma_j^\top z_{j,r}+\delta_j^\top e_{j,r}.
\]

Each term has a clear interpretation. The scalar $\beta_{j,0}(r)$ is a time-varying intercept that determines the baseline demand level for dataset $j$. The vector $\beta_j(r)$ contains the coefficients on observed structured features and captures how demand responds to covariates such as temperature, promotions, utilization, or admissions. The vector $\gamma_j$ represents the effect of the latent demand regime $z_{j,r}$, while the vector $\delta_j$ represents the effect of the latent event-state process $e_{j,r}$. Allowing $\beta_{j,0}(r)$ and $\beta_j(r)$ to vary over time enables the model to capture both baseline drift and coefficient drift.

This formulation captures two distinct forms of nonstationarity. First, the distribution of the observed covariates $x_{j,r}$ may evolve over time, leading to \emph{context drift}. Second, even conditional on $x_{j,r}$, the mapping from covariates and latent states to demand may change over time through $\beta_{j,0}(r)$ and $\beta_j(r)$, leading to \emph{concept drift}. In addition, the probabilities of different latent demand regimes and the effects of disruptive events 
are state-dependent and time-varying, which further implies that the conditional distribution of $y_{j,r}$ is nonstationary. 

Finally, to make the synthetic demand patterns more plausible, 
we use Gemini-3.1-Pro 
to generate contextual features and disruptive events, and to calibrate the signs and approximate magnitudes of their effects on demand. 
This improves the realism and credibility of the synthetic data.

Overall, the 47 seed datasets provide a flexible testbed for evaluating  inventory control under contextual demand, persistent effects of  disruptive events, and multiple forms of nonstationarity. Appendix~\ref{app:synthetic_data_generation} provides additional implementation details, including the mechanisms for event effects, within-archetype heterogeneity, and the demand distribution families.

\subsubsection{Slicing Strategy} We expand the dataset by slicing the 47 seed datasets. Each seed dataset spans 731 daily observations from January 1, 2024 to December 31, 2025. We apply a randomized slicing procedure. For each seed dataset $j$, we extract $N_{\mathrm{slice}}=10$ temporal slices. Each slice consists of a \textbf{100-day historical window} followed by a \textbf{30-day evaluation window}, resulting in a total length of 130 consecutive days.

To avoid excessive redundancy across slices, we randomly sample $N_{\mathrm{slice}}=10$ positions subject to a minimum separation constraint of 15 days between any two selected endpoints. This constraint reduces the number of near-duplicate slices (i.e., slices with substantial overlap in their underlying time windows), while still allowing partial overlap among the 130-day windows, which helps preserve temporal diversity.

This procedure yields $47 \times N_{\mathrm{slice}} = \mathbf{470}$ problem workspaces in total, spanning 15 industry domains. Figure~\ref{fig:slicing-stats} summarizes the distributional characteristics of the resulting 470~workspaces. The dataset exhibits substantial heterogeneity in both demand scale and variability: mean daily demand ranges from below~1 (intermittent industrial spare parts) to nearly~1{,}000 (power-grid loads), with a median of~57; the coefficient of variation ranges from~0.07 to~2.72, with a median of~0.35. This diversity allows the trained agent to encounter a wide spectrum of inventory environments during learning.

\begin{figure}[!htbp]
\centering
\includegraphics[width=\textwidth]{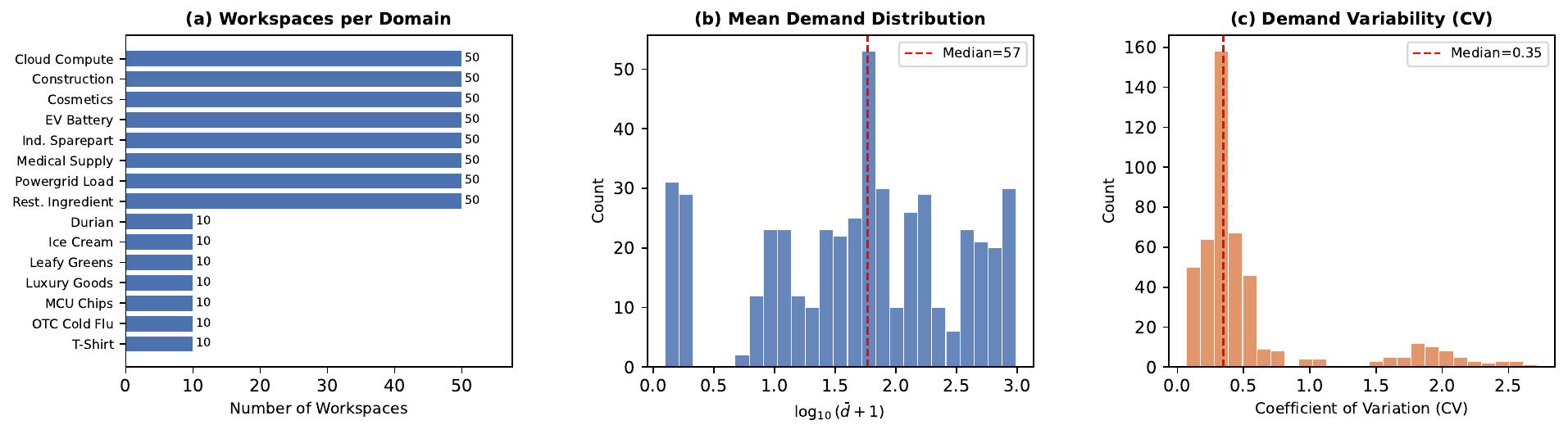}
\caption{Summary statistics of the 470 synthetic workspaces obtained by slicing 47 seed datasets. (a)~Number of workspaces per industry domain.  (b)~Distribution of mean daily demand across workspaces (log scale); the red dashed line marks the median.  (c)~Distribution of the demand coefficient of variation (CV); the red dashed line marks the median.}
\label{fig:slicing-stats}
\end{figure}

\subsection{Training Procedure and Experimental Setup}
\label{sec:exp-setup}

In this section, we describe how training instances are constructed, the agent architecture, the reinforcement learning configuration, and the evaluation procedure.

\paragraph{Workspace construction.}
Each problem instance is packaged as a self-contained \emph{workspace}, which provides all the information required for the agent to analyze a demand management problem and generate an inventory policy. Specifically, each workspace contains four components.

\begin{enumerate}[leftmargin=2em, itemsep=2pt, topsep=4pt]
    \item \textbf{Problem description} 
          (\texttt{problem\_description.md}): a natural-language 
          specification of the inventory setting, including the SKU 
          (stock-keeping unit) type, cost parameters (unit holding cost $h$, unit penalty cost $p$), 
          and replenishment lead time $L$.
    \item \textbf{Historical demand data} 
          (\texttt{data/historical\_sequence.json}): 100 days of 
          realized demand observations together with the corresponding 
          exogenous feature vectors. The agent uses this sequence to 
          learn demand patterns before designing a policy.
    \item \textbf{Evaluation script} (\texttt{evaluation.py}): a 
          simulator that evaluates any policy 
          generated by the agent by computing its average cost (holding cost plus penalty cost) over the 
          historical sequence. The hyperparameters of each submitted 
          policy are further tuned via Bayesian optimization (Optuna) 
          before reporting the cost, ensuring a well-tuned evaluation  
          for each candidate solution.
    \item \textbf{Baseline policies} (\texttt{baseline\_policies/}): 
          five standard parametric policies, including base stock, capped base 
          stock, constant order, newsvendor, and $(s,S)$ policies. Each policy is 
          pre-tuned via Bayesian optimization on the same historical 
          sequence. These baselines represent the performance level of 
          classical inventory methods and serve as benchmarks that the 
          agent aims to surpass.
\end{enumerate}

\paragraph{Agent architecture.}
We adopt a design aligned with mainstream coding agents such as Claude Code \citep{claudecode2025}, where an agent scaffold enables the LLM to autonomously explore the workspace. The agent interacts with the workspace exclusively through a general-purpose \texttt{bash} shell, which allows it to read files, execute Python scripts, and write new policy files. Each executed \texttt{bash} command is counted as one tool call. The design is intentionally minimal: the agent has no built-in numerical solvers, no persistent memory module, and no access to external knowledge sources during a rollout. All reasoning is grounded solely in the workspace contents.

Within a budget of $J = 60$ tool-call iterations, the agent follows a structured \emph{explore--develop--submit} workflow:
\begin{enumerate}[leftmargin=2em, itemsep=1pt, topsep=2pt]
    \item \textbf{Explore}: issue arbitrary \texttt{bash} commands to read the problem description and analyze the historical demand data or files in order to understand the demand environment.
    \item \textbf{Benchmark}: invoke the evaluation script on each of the five baseline policies and record their costs to establish the performance threshold.
    \item \textbf{Develop}: write Python policy files for candidate policies, evaluate them using the simulator, and iteratively refine the design based on cost feedback.
    \item \textbf{Submit}: write the best-performing policy to \texttt{final\_submit.py}, which signals completion. The rollout terminates either upon submission or when the iteration budget is exhausted.
\end{enumerate}

\paragraph{High-quality trajectory generation.}
Before reinforcement learning, we first use an advanced large model DeepSeek-R1-0528 under the above architecture to construct high-quality demonstration demand trajectories for training. Each trajectory records how a capable model analyzes and solves a given workspace. It consists of the full sequence of tool calls, including file reads, script executions, and policy edits, together with the corresponding model outputs that lead to a submitted policy.

A trajectory is labeled as \emph{improved} if the final submitted policy achieves a strictly lower average cost than the best of the five pre-tuned baseline policies. Among the total 470 workspaces, we  select 10 improved workspaces together with their corresponding trajectories as the training data for reinforcement learning.

\paragraph{Reinforcement learning training.}
We apply Group Relative Policy Optimization 
\citep[GRPO;][]{shao2024deepseekmath} to GLM-4.7-Flash (30B, MoE), a compact language model suitable for agentic tasks under computational constraints. GRPO is a policy-gradient method that estimates advantages by comparing outcomes within a group of rollouts generated from the same prompt, without requiring a separate value function. The training is implemented using the SlimeRL framework \citep{slime_github}.

The reward is binary, defined as $r \in \{0,1\}$, where $r=1$ if and only if the submitted policy improves upon the best baseline, which is consistent with the theoretical analysis in Section~\ref{sec:basic_model_training}. Detailed training configurations are provided in Appendix~\ref{app:rl_config}. We refer to the trained model implemented by the best checkpoint\footnote{A checkpoint is a saved version of a model's parameters at a particular point during training, which can later be used for evaluation, continued training, or deployment.} as \textit{InvEvolve}.

Recall that each workspace contains a 130-day demand sequence. The first 100 days form the \emph{training horizon}, during which the agent observes historical data and develops its inventory policy. The final submitted policy is then evaluated on the subsequent \emph{30-day test horizon}, which is not observed by the agent. This temporal split ensures that the reported costs reflect out-of-sample generalization rather than in-sample overfitting.

\section{Experiments}
\label{sec:experiment}

In this section, we evaluate the trained model on both synthetic and real-world datasets through comprehensive experiments. In addition to the inference framework in Section~\ref{sec:basic_model_training}, we use a practical calibration of the replay--deployment discrepancy budget $\xi_t$ (Appendix~\ref{app:xi-estimation}) and a practical small-sample replay-radius adjustment for short replay windows (Appendix~\ref{app:small-sample-radius}). All other inference and deployment procedures follow the framework illustrated in Figure~\ref{fig:inference_flowchart}.

For comparison, we use five standard baseline policies, including the base-stock policy, capped base-stock policy, constant-order policy, newsvendor policy, and $(s,S)$ policy. These policies play a dual role in our framework. First, they serve as classical benchmark policies for comparison. Second, they provide the starting point for policy evolution, from which InvEvolve iteratively improves the policies. Consequently, the policies generated by InvEvolve remain comparable to these baselines in terms of structure and interpretation. They are still explicit and executable white-box ordering rules, although their functional forms may differ.  We present several representative examples in Appendix~\ref{app:structural-evolved-policies}.

In addition, we include two representative black-box models, denoted as A3C \citep{gijsbrechts2022can} and E2E \citep{qi2023practicale2einventory}. The former represents a general-purpose reinforcement learning approach, while the latter corresponds to a practical end-to-end deep learning method. 
Finally, we examine whether InvEvolve can extend the performance frontier of classical inventory policies.

\subsection{Experiment~1: Synthetic Inventory Benchmark}
\label{sec:exp-synthetic}

\paragraph{Dataset.}
We randomly select 30 instances from 10 industry domains from the remaining 460 synthetic workspaces after excluding the 10 training workspaces, and use them to form 30 test workspaces. These test instances are distinct from the 10 instances used for training.
Crucially, no SKU appears in both training and test, ensuring a strict out-of-distribution (OOD) evaluation at the product level.

In addition, the test set includes industry-specific exogenous features (e.g., maintenance indices, weather variables, and promotional indicators) and stochastic shock events (e.g., supplier disruptions and demand spikes). 
These features and shocks vary across industry domains, creating heterogeneous feature structures.
Moreover, 5 of the 10 test industry domains are entirely absent from the training data, providing a stringent test of cross-domain generalization. Across all 30 test workspaces in the test set, the lead time is set at $L = 5$, while the holding cost and lost-sales cost vary across SKUs, with a fixed ratio of $p/h = 10$.

\paragraph{Results.}
Table~\ref{tab:synthetic-results} reports the performance of the base GLM-4.7-Flash and \textit{InvEvolve} across 30 workspaces at the domain level.
We evaluate each model using two metrics: the \emph{success rate}, 
defined as 
the proportion of instances in which each model outperforms the best baseline policy (reported under the ``Base GLM'' and ``InvEvolve'' columns), 
and the average relative cost reduction of the trained model compared to the best baseline within each domain (defined as \((C_{\mathrm{baseline}}-C_{\mathrm{method}})/C_{\mathrm{baseline}}\times100\%\) and reported under the ``Avg.\ Cost $\downarrow$'' columns)

The evaluation protocol follows the same procedure as in previous sections. The LLM operates over the past 100-day window, where it iteratively 
generates, evaluates, filters, and selects a final policy. The selected policy is then compared with the five baseline policies, whose hyperparameters are tuned over the same 100-day window under an identical Bayesian optimization budget. All policies are subsequently evaluated on the future 30-day window, where the final costs are computed.

\begin{table}[!htbp]
\centering
\caption{Success rate and average cost reduction on the synthetic inventory benchmark (30 OOD test workspaces). Domains marked with $\dagger$ are entirely absent from the training data. Avg.\ Cost $\downarrow$ is the mean relative cost reduction over successful cases compared to the best classical baseline.}
\label{tab:synthetic-results}
\small
\begin{tabular}{lccccc}
\toprule
\textbf{Industry Domain} & \textbf{\# Cases} & \textbf{Base GLM} & \textbf{InvEvolve} & \textbf{Avg.\ Cost $\downarrow$} & \textbf{Domain OOD} \\
\midrule
Cloud Compute & 6 & 2/6\;\;(33\%) & 6/6\;\;(100\%) & 7.6\% & $\dagger$ \\
Power Grid Load & 5 & 2/5\;\;(40\%) & 3/5\;\;(60\%) & 20.6\% & $\dagger$ \\
Construction Material & 4 & 4/4\;(100\%) & 4/4\;\;(100\%) & 28.7\% & $\dagger$ \\
EV Battery Supply & 4 & 2/4\;\;(50\%) & 4/4\;\;(100\%) & 6.8\% & \\
Medical Supply & 3 & 1/3\;\;(33\%) & 3/3\;\;(100\%) & 15.4\% & \\
Restaurant Ingredient & 3 & 1/3\;\;(33\%) & 2/3\;\;(67\%) & 20.0\% & \\
Cosmetics & 2 & 1/2\;\;(50\%) & 1/2\;\;(50\%) & 31.2\% & \\
Leafy Greens & 1 & 0/1\;\;(0\%) & 1/1\;\;(100\%) & 15.2\% & $\dagger$ \\
T-Shirt & 1 & 0/1\;\;(0\%) & 1/1\;\;(100\%) & 14.5\% & $\dagger$ \\
Industrial Spare Parts & 1 & 0/1\;\;(0\%) & 0/1\;\;(0\%) & --- & \\
\midrule
\textbf{Overall} & \textbf{30} & \textbf{13/30\;(43\%)} & \textbf{25/30\;(83\%)} & \textbf{15.9\%} & \\
\quad Domain-OOD only & 17 & 8/17\;\;(47\%) & 15/17\;(88\%) & 16.8\% & \\
\bottomrule
\end{tabular}
\end{table}

The base model achieves an overall success rate of 43\%, while InvEvolve attains 83\%, corresponding to a \textbf{93\% relative improvement}. Notably, on the 5 OOD industry domains, 
InvEvolve achieves 15/17 (88\%), which indicates that the learned generation quality generalizes effectively to previously unseen industry domains.

\paragraph{Further Analysis: Training Dynamics.}
We further examine how the performance of the LLM evolves with cumulative training compute on the 30 OOD test workspaces. Figure~\ref{fig:checkpoint} plots the success rate against training compute. Starting from 43\% for the base model, the success rate remains in the 43\%--47\% range at lower compute budgets, rises to 62\% at \(2.3\times10^{17}\) FLOPs, peaks at 83\% at \(2.9\times 10^{17}\) FLOPs, and then declines to 67\% at \(4.1\times 10^{17}\) FLOPs, suggesting overfitting beyond the best checkpoint. The rise toward the best checkpoint is qualitatively consistent with the training-to-inference interface suggested by Theorem~\ref{thm:train_concentration}: The proxy distribution over canonical policies concentrates on the training-stage good region $\mathcal{G}^{\operatorname{tr}}$ at an exponential rate in $K$, while observable performance gains emerge only after the probability mass ratio $\rho_K$ exceeds a critical threshold.

\begin{figure}[!htbp]
    \centering
    \includegraphics[width=0.55\textwidth]{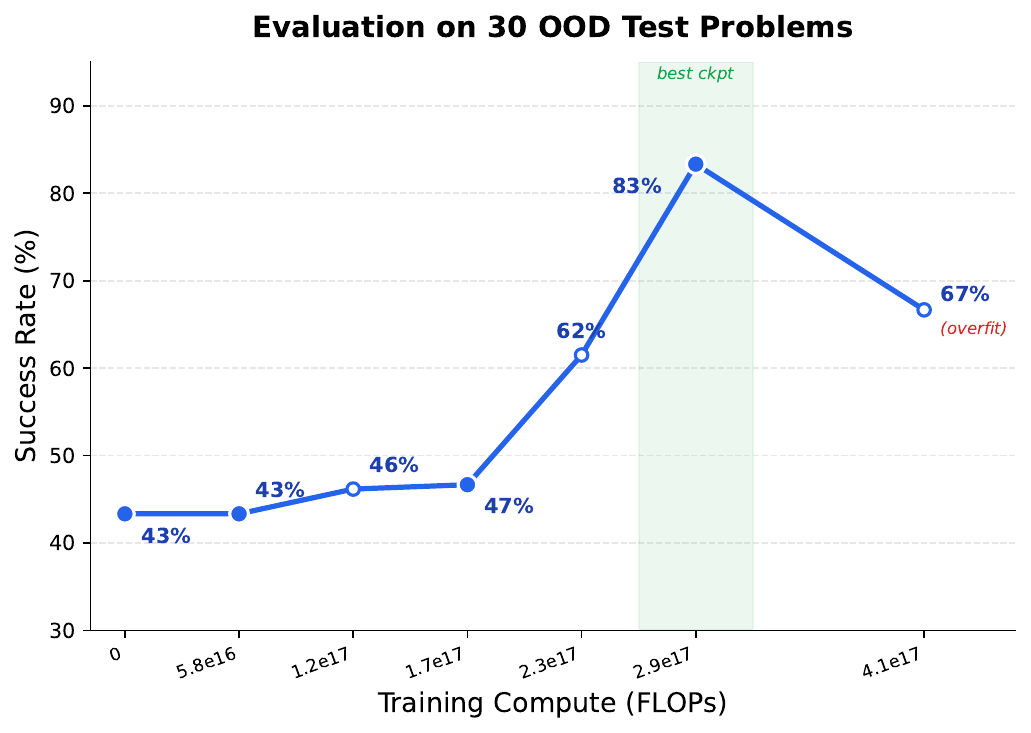}
    \caption{Success rate vs. GRPO training compute on 30 OOD test workspaces.}
    \label{fig:checkpoint}
\end{figure}

\subsection{Experiment~2: Real-World Retail Data (Dunnhumby Complete Journey)}
\label{sec:exp-cj}

In this section, we evaluate the performance of our framework on real-world data. To ensure that deep learning policies can perform effectively with sufficiently rich training data, we extend the lookback window in the real-data setting and primarily use the previous 365 days of data to determine the inventory ordering policy for the subsequent 30 days. Across all test instances, we set the lead time to 5 and the shortage-to-holding cost ratio to $p/h = 10$.

\paragraph{Dataset.}
We construct a second testbed from the Dunnhumby Complete Journey dataset \citep{dunnhumby2020completejourney}, which contains 2.6 million point-of-sale transactions covering 92{,}000 products and 582 retail stores over the 2016--2018 period. The dataset records product-level transactions, coupon usage, and calendar indicators such as weekends and holidays, which makes it well aligned with everyday retail inventory settings. We aggregate transactions into daily demand at the product level and then randomly sample 30 time windows of length $365+30$ days from product sales histories to construct 30 workspaces, which we refer to as the \textit{CJ 365+30 benchmark} throughout the rest of the paper. Each workspace includes eight exogenous features: \texttt{name}, \texttt{categories}, \texttt{is\_weekend}, \texttt{is\_holiday}, \texttt{discount\_rate}, \texttt{is\_on\_display}, \texttt{is\_in\_mailer}, and \texttt{day\_of\_week}. Notably, \texttt{name} and \texttt{categories} are textual features. LLM-based methods can use these fields as auxiliary signals, for example to infer whether an item is perishable or exhibits seasonal demand patterns, whereas standard deep learning methods generally cannot exploit such information as directly.

This test set differs from the synthetic benchmark in two main respects. First, it has a different feature structure, relying on retail promotion and calendar variables rather than the synthetic covariates used in the benchmark. Second, it reflects a different data-generating environment: its demand scale and variability arise from real consumer behavior, and its product categories and data source are unrelated to those in the training data. 

Because the CJ dataset records point-of-sale transactions rather than true customer demand, observed sales may be censored under a lost-sales setting. We therefore construct a real-data benchmark in two steps. First, we generate rolling product-level temporal slices to reduce the share of zero-sales observations and to remove extremely sparse or irregular retail records. Second, on days with potential stockout signals, we treat observed sales as right-censored demand and recover latent demand using a censored count-data model \citep{trapero2024demand,huh2011adaptive}. Appendix~\ref{app:cj-lost-sales} presents the full construction procedure and reports summary statistics for the 30 selected workspaces.

\paragraph{Results.}

Table~\ref{tab:cj-selected30-by-dept} reports the success rates on the CJ 365+30 benchmark across 30 workspaces, where success means outperforming the strongest classical baseline for the same workspace. Overall, \textit{InvEvolve} achieves the highest success rate, succeeding in 20 out of 30 cases (67\%). This success rate is substantially higher than that of A3C, which succeeds in 11 out of 30 cases (37\%), and also exceeds that of E2E, which succeeds in 16 out of 30 cases (53\%). Conditional on the successful cases, \textit{InvEvolve} reduces average cost by 9.2\% relative to the best classical baseline.

\begin{table}[!htbp]
  \centering
  \caption{Success rates on the CJ 365+30 benchmark across 30 workspaces. Each entry reports the number and percentage of cases in which the learner outperforms the strongest classical baseline for the same workspace. For InvEvolve, we also report the mean relative cost reduction over successful cases compared to the best classical baseline.}
  \label{tab:cj-selected30-by-dept}
  \small
  \begin{tabular}{lccccc}
    \toprule
    \textbf{Department} & \textbf{\# Cases} & \textbf{InvEvolve} & \textbf{Avg.\ Cost $\downarrow$} & \textbf{A3C} & \textbf{E2E} \\
    \midrule
    Deli               & 1  & 1/1\;(100\%)    & 6.5\%  & 0/1\;\;(0\%)    & 0/1\;\;(0\%) \\
    Drug \& GM         & 2  & 0/2\;\;(0\%)    & --     & 2/2\;(100\%)    & 0/2\;\;(0\%) \\
    Grocery            & 12 & 9/12\;(75\%)    & 9.4\%  & 4/12\;(33\%)    & 9/12\;(75\%) \\
    Meat               & 1  & 0/1\;\;(0\%)    & --     & 0/1\;\;(0\%)    & 1/1\;(100\%) \\
    Meat (Packaged)    & 1  & 1/1\;(100\%)    & 11.7\% & 0/1\;\;(0\%)    & 1/1\;(100\%) \\
    Produce            & 13 & 9/13\;(69\%)    & 8.9\%  & 5/13\;(38\%)    & 5/13\;(38\%) \\
    \midrule
    \textbf{Overall} & \textbf{30} & \textbf{20/30\;(67\%)} & \textbf{9.2\%} & \textbf{11/30\;(37\%)} & \textbf{16/30\;(53\%)} \\
    \bottomrule
  \end{tabular}
\end{table}

Beyond the aggregate success rates, Table~\ref{tab:cj-selected30-by-dept} suggests that the three methods have different strengths across inventory categories. \textit{InvEvolve} performs particularly well in \textit{Grocery} and \textit{Produce}, which are categories where demand often exhibits repeated calendar, promotion, display, and seasonal patterns. A possible explanation is that these signals can be effectively translated into compact white-box ordering rules, such as rule-based adjustments for promotional exposure, weekday effects, or short-term demand shifts. This is especially useful for perishable or semi-perishable products, where the policy needs to react to demand variation while remaining simple and robust.

By contrast, A3C appears stronger in \textit{Drug \& GM}, a category that is typically less perishable and may involve slower-moving or more intermittent demand. One possible reason is that reinforcement learning can benefit from sequential state-dependent adjustments when the inventory dynamics are less dominated by short shelf-life constraints. E2E is more competitive in categories such as \textit{Meat} and \textit{Meat (Packaged)}, where nonlinear demand patterns or SKU-level heterogeneity may be harder to summarize with a small number of explicit rules. However, its black-box flexibility may also make performance less stable across heterogeneous workspaces.

Overall, these results suggest that \textit{InvEvolve} is most effective when the key demand drivers are structured and interpretable enough to be encoded into executable inventory rules. This makes it especially suitable for retail categories with recurring calendar and promotion effects, while black-box learners may have advantages in categories where the relevant demand response is more nonlinear or less easily captured by simple policy forms.

\subsection{Experiment 3: Classical Inventory Benchmark and Policy Discovery}
\label{sec:exp-cbs}

In the previous experiments, we show that the policies synthesized by InvEvolve frequently outperform the best of the five standard baselines. However, the policies generated by InvEvolve may exploit additional feature variables that are not available to classical inventory policies. 

We next examine whether InvEvolve can discover more effective and structurally novel policies in a classical setting: a single-sourcing lost-sales inventory system with lead time. This setting provides a clean basis for comparison, since all methods operate under the same information structure. More importantly, it allows us to assess whether InvEvolve can contribute to classical inventory policy design by identifying new policy structures. 

\subsubsection{Preliminaries and settings.}

We consider a single-sourcing lost-sales inventory system with lead time $L$ under a stationary demand distribution. We now focus on the inventory dynamics at the microscopic time scale. Time is indexed by $n=1,2,\dots$, as introduced at the beginning of Section~\ref{subsec:basic_model}.

We follow the comparison framework and methodology in \citep{xin,zipkin2008oldnew}, while incorporating a broader set of demand distributions. Specifically, we construct a controlled evaluation testbed with six demand distributions that span a range of variability patterns, as summarized in Table~\ref{tab:cbs-dists}.

\begin{table}[!htbp]
\centering
\caption{Demand distributions used in the CBS comparison experiment. All distributions have mean $\approx 5$ and produce non-negative integer demand over 2{,}000 periods.}
\label{tab:cbs-dists}
\small
\begin{tabular}{llc}
\toprule
\textbf{Distribution} & \textbf{Description} & \textbf{CV} \\
\midrule
Geometric$(p{=}1/6)$ & Memoryless, heavy-tailed discrete & ${\sim}1.0$ \\
Poisson$(\lambda{=}5)$ & Classical count data, equi-dispersed & ${\sim}0.45$ \\
Binomial$(10, 0.5)$ & Bounded, low variance & ${\sim}0.32$ \\
Gamma$(k{=}2,\, \mu{=}5)$ & Moderate tail, continuous-rounded & ${\sim}0.71$ \\
HalfNormal$(\mu{\approx}5)$ & Folded normal, moderate variance & ${\sim}0.76$ \\
Uniform$[0,10]$ & Bounded, equi-dispersed & ${\sim}0.58$ \\
\bottomrule
\end{tabular}
\end{table}

Following the experimental design in \cite{xin}, we consider six demand distributions. For each distribution, we evaluate all combinations of four lead times, $L \in \{1,2,3,4\}$, and four penalty-to-holding-cost ratios, $p/h \in \{4,9,19,39\}$, with $h=1$ fixed. This setup yields a total of $6 \times 4 \times 4 = 96$ scenarios. In each scenario, every policy receives the same budget of 50 Bayesian optimization trials for hyperparameter tuning, using a fixed random seed.

We emphasize the capped base-stock (CBS) policy, which shows consistently strong empirical performance. Table~\ref{tab:baseline-dominance} reports the number of scenarios in which each policy achieves the lowest cost. 

\begin{table}[!htbp]
\centering
\caption{Number of scenarios (out of 16) in which each baseline achieves the lowest cost, across six demand distributions (\(L \in \{1,2,3,4\}\), \(p/h \in \{4,9,19,39\}\), \(h=1\), \(N_{\mathrm{sim}}=2{,}000\), and 50 Optuna TPE trials per policy).}
\label{tab:baseline-dominance}
\small
\begin{tabular}{lccccc}
\toprule
\textbf{Distribution} & \textbf{Base stock} & \textbf{Capped base stock} & \textbf{Constant order} & \textbf{Newsvendor} & \(\boldsymbol{(s,S)}\) \\
\midrule
Geometric\((p=1/6)\)          & 0 & \textbf{15} & 1 & 0 & 0 \\
Poisson\((\lambda=5)\)        & 2 & \textbf{14} & 0 & 0 & 0 \\
Binomial\((10, 0.5)\)         & 4 & \textbf{12} & 0 & 0 & 0 \\
Gamma\((k=2,\mu=5)\)          & 0 & \textbf{15} & 0 & 1 & 0 \\
HalfNormal\((\mu \approx 5)\) & 1 & \textbf{15} & 0 & 0 & 0 \\
Uniform\([0,10]\)             & 1 & \textbf{15} & 0 & 0 & 0 \\
\midrule
\textbf{Total (96)}           & 8 & \textbf{86} & 1 & 1 & 0 \\
\bottomrule
\end{tabular}
\end{table}

The results show that CBS is the most effective baseline, achieving the lowest cost in 86 out of 96 scenarios (89.6\%). This finding is consistent with the theoretical insights in \citep{xin,zipkin2008oldnew}. The remaining wins are primarily attributed to the base-stock policy (8/96), which coincides with CBS when the cap is non-binding. In contrast, constant order, newsvendor, and $(s,S)$ policies are rarely optimal.

These observations motivate the use of CBS as the primary benchmark in the subsequent analysis. We will further investigate whether the proposed InvEvolve framework can discover policies that are both more effective and more robust than CBS.

\subsubsection{Policies Exploration}

Our exploration procedure is designed as follows. We begin with the five baseline policies as the initial policy set. The LLM then operates under the same reasoning framework as before, with a tool-call budget of $J=60$ in each round. Different from Experiments~1 and~2, we further introduce an iterative refinement mechanism. After one round of reasoning is completed, the best policy selected by InvEvolve is added to the current baseline set, and the LLM starts the next round of reasoning using this expanded candidate pool. We repeat this process for 10 rounds. In addition, we conduct two independent runs of this 10-round iterative procedure, which finally produce two LLM-discovered policies.

The motivation for this design is that CBS performs extremely well under stationary demand distributions, which makes it difficult for the LLM to discover a better policy in a single round of reasoning. A direct increase in the tool-call budget, such as setting $J=200$, would lead to excessively long contexts and substantially higher hardware costs. In contrast, the iterative refinement scheme offers a practical alternative. It allows the model to accumulate useful policy structures across rounds while keeping the context length manageable. This idea can also be applied to other experiments and practical use to further improve performance.

Recall that $\mathrm{IP}_n$ denotes the inventory position at time $n$. Define the inventory gap as $\Delta_n := \max(0,\, S - \mathrm{IP}_n)$, which measures the shortfall from the target level. The capped base-stock (CBS) policy  orders $q_n = \min\!\bigl(\Delta_n,\; r\bigr)$, where $S$ is the base-stock level and $r$ is the order cap. Our framework generates the following two policies:

\begin{enumerate}
    \item \textbf{Tilted-CBS (Elastic Capped Base-Stock).} The Tilted-CBS policy replaces the fixed cap with a state-dependent cap:
$$q_n = \min\!\bigl(\Delta_n,\; r_{\mathrm{base}} + \alpha\,\Delta_n\bigr),$$
where $r_{\mathrm{base}} \ge 0$ is a base cap and $\alpha \in [0,1]$ controls the elasticity. When $\alpha=0$, the policy reduces to CBS. When $\alpha>0$, the cap increases with the gap.

    \item \textbf{Tilted-PIC (Proportional Inventory Controller with Elastic Cap).} The Tilted-PIC policy introduces a proportional gain $K_p \in (0,1.5]$ and $\alpha \in [0,1]$:
$$
q_n = \max\!\Bigl(0,\;\min\!\bigl(\lfloor K_p \cdot \Delta_n \rceil,\; r_{\mathrm{base}} + \alpha\,\Delta_n\bigr)\Bigr).
$$

Here $\lfloor \cdot \rceil$ denotes rounding to the nearest integer. When $K_p = 1$, Tilted-PIC reduces to Tilted-CBS. When $K_p = 1$ and $\alpha = 0$, Tilted-PIC further reduces to CBS.
\end{enumerate}

We find that the two policies generated by the agent are both structural extensions of CBS. A plausible explanation is that CBS performs exceptionally well under geometric demand, making it difficult to identify fundamentally different policies that can outperform it. The results in \citet{xin} support this observation, showing that under geometric demand, CBS is already very close to optimal, with a minimum optimality gap of approximately $0.3\%$ and a maximum of about $1.4\%$.

We further evaluate these two policies against CBS under alternative demand distributions. Specifically, a policy is classified as a \emph{win} (W) if its optimized cost is more than $2\%$ lower than that of CBS, a \emph{loss} (L) if it is more than $2\%$ higher, and a \emph{tie} (T) otherwise.

\subsubsection{Results.} Table~\ref{tab:cbs-aggregate} reports the aggregate comparison results. Tilted-PIC achieves 41 wins with only 8 losses, corresponding to a beat-or-tie rate of 91.7\%. In comparison, Tilted-CBS achieves 18 wins and 4 losses, with a higher beat-or-tie rate of 95.8\%. The two policies exhibit complementary performance profiles. Tilted-CBS is more conservative, with fewer losses, while Tilted-PIC identifies substantially more improvements at the cost of a slightly higher number of losses.

\begin{table}[!htbp]
\centering
\caption{Aggregate comparison with CBS across 96 scenarios (6 distributions $\times$ 16 $(L, p/h)$ pairs).  W/T/L counts use a $\pm 2\%$ threshold relative to CBS cost.}
\label{tab:cbs-aggregate}
\small
\begin{tabular}{lcccccc}
\toprule
\textbf{Policy} & \textbf{W} & \textbf{T} & \textbf{L} & \textbf{Beat-or-Tie} & \textbf{Mean \%} & \textbf{Worst \%} \\
\midrule
Tilted-CBS & 18 & 74 & 4 & 95.8\% & $-$0.60\% & $+$5.4\% \\
\textbf{Tilted-PIC} & \textbf{41} & \textbf{47} & \textbf{8} & \textbf{91.7\%} & $\boldsymbol{-}$\textbf{1.52\%} & $+$4.2\% \\
\bottomrule
\end{tabular}
\end{table}

Table~\ref{tab:cbs-perdist} reports the performance of Tilted-PIC across different demand distributions. Three distributions, namely Geometric, Binomial, and Gamma, exhibit \emph{zero losses}, meaning that Tilted-PIC is never worse than CBS across all 16 $(L, p/h)$ scenarios.
The remaining losses are concentrated in Poisson (3), Uniform (4), and HalfNormal (1). These cases mainly occur in the short lead-time and high penalty regime, where the aggressive full-gap ordering behavior of CBS is close to optimal.
\begin{table}[!htbp]
\centering
\caption{Per-distribution performance of Tilted-PIC vs CBS.  Bold rows indicate zero-loss distributions.}
\label{tab:cbs-perdist}
\small
\begin{tabular}{lcccccc}
\toprule
\textbf{Distribution} & \textbf{W} & \textbf{T} & \textbf{L} & \textbf{Mean \%} & \textbf{Best \%} & \textbf{Worst \%} \\
\midrule
\textbf{Geometric}$(p{=}1/6)$ & \textbf{8} & \textbf{8} & \textbf{0} & $\boldsymbol{-}$\textbf{2.3\%} & $-$7.8\% & $+$0.2\% \\
Poisson$(\lambda{=}5)$ & 6 & 7 & 3 & $-$0.8\% & $-$7.2\% & $+$4.2\% \\
\textbf{Binomial}$(n{=}10)$ & \textbf{8} & \textbf{8} & \textbf{0} & $\boldsymbol{-}$\textbf{1.7\%} & $-$8.3\% & $+$2.0\% \\
\textbf{Gamma}$(k{=}2)$ & \textbf{7} & \textbf{9} & \textbf{0} & $\boldsymbol{-}$\textbf{2.1\%} & $-$5.9\% & $+$0.2\% \\
HalfNormal$(\mu{\approx}5)$ & 6 & 9 & 1 & $-$1.4\% & $-$6.2\% & $+$3.2\% \\
Uniform$[0,10]$ & 6 & 6 & 4 & $-$0.9\% & $-$5.6\% & $+$3.7\% \\
\bottomrule
\end{tabular}
\end{table}

Figure~\ref{fig:cbs-comparison} provides a multi-dimensional view of the comparison.  Panel~(a) shows the W/T/L composition per distribution for both Tilted-PIC and Tilted-CBS; Tilted-PIC consistently exhibits a larger green (win) region.  Panel~(b) presents a heatmap of the average cost change by $(L, p/h)$.  The improvement is strongest in the bottom-left region ($L \ge 3$, $p/h \le 9$), reaching $-5.3\%$ at $(L{=}4, p/h{=}4)$, and weakest in the top-right ($L{=}1$, $p/h \ge 19$), where CBS is near-optimal.  Panel~(c) compares the mean improvement of Tilted-PIC and Tilted-CBS across distributions, showing that Tilted-PIC's advantage is broad-based.  Panel~(d) reveals that the optimized $K_p$ decreases monotonically with lead time (from~0.88 at $L{=}1$ to~0.56 at $L{=}4$) and increases with the penalty ratio (from~0.57 at $p/h{=}4$ to~0.81 at $p/h{=}39$). This pattern is consistent with the intuition underlying Tilted-PIC. Longer lead times call for stronger damping, which is reflected in a lower proportional gain \(K_p\). By contrast, higher lost-sales penalties make more aggressive replenishment more desirable and push the controller toward the CBS limit \(K_p = 1\).

\begin{figure}[!htbp]
\centering
\includegraphics[width=\textwidth]{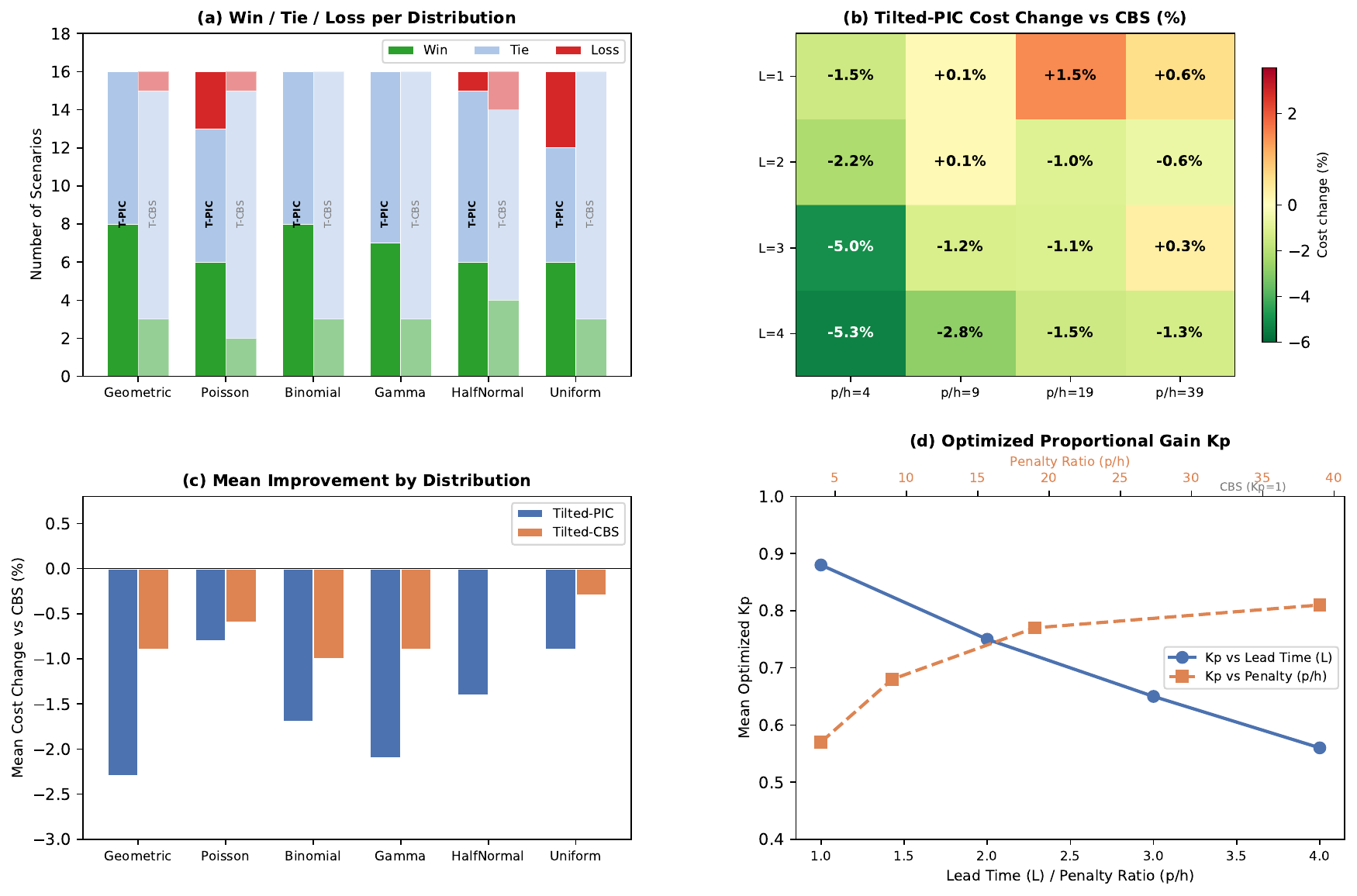}
\caption{Comparison of Tilted-PIC and Tilted-CBS against CBS across 96 scenarios.  (a)~Win/Tie/Loss composition by distribution.  (b)~Heatmap of Tilted-PIC's average cost change vs CBS by lead time and penalty ratio; green indicates improvement.  (c)~Mean cost reduction by distribution.  (d)~Optimized proportional gain $K_p$ as a function of lead time (solid, decreasing) and penalty ratio (dashed, increasing); the horizontal reference at $K_p{=}1$ corresponds to CBS behavior.}
\label{fig:cbs-comparison}
\end{figure}

\subsubsection{Structural Interpretation}
We provide a qualitative interpretation of the above experimental results rather than a formal analysis which is beyond the scope of this paper. 
For a rigorous discussion of CBS, we refer to \citet{xin}. 

A useful way to interpret these results
is that lost-sales systems with positive lead times are affected by two structural frictions: \emph{excessively aggressive replenishment after demand spikes} and \emph{slow recovery after severe inventory depletion}. CBS uses a fixed target level together with a fixed order cap. While this design is robust and easy to tune, it is structurally  rigid. It may order  too aggressively when the inventory gap is moderate, and too conservatively when the gap is large. 

The first friction arises because lost demand does not accumulate as backlog. After a temporary demand spike, CBS attempts to close the full inventory gap via $\min(\Delta_n,r)$. When demand returns to normal, this aggressive response may lead to excess inventory and increased holding cost, especially under long lead times. Tilted-PIC addresses this issue through the proportional factor $K_p<1$, which dampens the replenishment response and reduces overshooting.

The second friction occurs after large stockout events. When $\mathrm{IP}_n$ falls far below $S$, the fixed cap $r$ limits the recovery speed, creating a bottleneck. This effect is amplified by positive lead times, which delay replenishment arrivals. Tilted-CBS addresses this issue by introducing an elastic cap $r_{\mathrm{base}}+\alpha \Delta_n$, which increases the allowable order size when the gap is large. Tilted-PIC inherits this mechanism.

Overall, Tilted-CBS improves upon CBS by relaxing the fixed-cap constraint in large-gap states, while preserving the base-stock structure. Tilted-PIC further enhances performance by combining proportional damping with an elastic cap, allowing it to respond conservatively when the gap is small and aggressively when the gap is large. 

The gains are most pronounced under high demand variability. In such environments, both over-reaction and slow recovery occur more frequently. The elastic cap accelerates recovery after large shocks, while proportional damping reduces over-ordering after temporary spikes. This combination leads to consistent improvements over CBS in high-variance settings.

\section{Conclusion}
\label{conclusion}

This paper studies how to evolve executable white-box inventory policies with large language models in online, non-stationary settings. 
We develop a unified framework that connects training, inference, and deployment. 
The framework explicitly incorporates structural validity, replay-based certification, and deployment safety into the policy-evolution process. 
Our theoretical analysis links the guarantees across these three stages. 
Taken together, these results show that LLM-based policy evolution can be both statistically disciplined and operationally deployable in weakly structured inventory environments.
Training makes structurally valid and high-success policies more likely to be generated. Inference converts this improved generation behavior into single-period certified improvements through confidence-screened replay evaluation. During deployment, the framework remains conservatively safe while allowing us to bound the dynamic gap relative to the oracle-safe benchmark. Taken together, these results show that LLM-based policy evolution can be both 
statistically disciplined and operationally deployable for inventory problems.

The numerical results further demonstrate the practical value of the framework. 
InvEvolve substantially outperforms the base model on both the synthetic out-of-distribution benchmark and the real-world retail benchmark. It also achieves stronger task completion and broader cross-domain generalization. 
In the classical lost-sales setting, the framework evolves  meaningful structural variations including Tilted-CBS and Tilted-PIC. This finding suggests that generative policy search can serve not only as a competitive decision tool, but also as a useful approach for discovering new structures of inventory policies. Overall, the paper shows that generative AI may support inventory-policy design 
while preserving interpretability and deployment discipline. This points to a potentially useful direction for data-driven and auditable policy design in operations management.

\newpage

\bibliography{reference}
\bibliographystyle{plainnat}

\newpage
\begin{APPENDICES}
\counterwithin{proposition}{section}
\counterwithin{lemma}{section}
\counterwithin{example}{section}
\counterwithin{theorem}{section}
\numberwithin{equation}{section}
\renewcommand{\theequation}{\thesection. \arabic{equation}}
\setcounter{page}{1}

\setcounter{equation}{0}

\begin{center}
		{\LARGE \vspace{.1cm}
	  \Large \textbf{Online Supplement} 
		}
		\medskip
\end{center}

\section{Notation Summary}\label{sec:appendix_notation}

\paragraph{Global indexing conventions.}
Throughout the paper, we use the following indexing convention unless stated otherwise. In particular, in Tables~\ref{tab:notation_summary_inference} and \ref{tab:notation_summary_deployment}, $t$ denotes the outer policy-update epoch.

\begin{table}[H]
\centering
\caption{Global indexing conventions}
\label{tab:global_indexing_conventions}
\begin{tabular}{ll}
\toprule
Symbol & Meaning \\
\midrule
$t$ & outer policy-update / deployment epoch \\
$j$ & inner InvEvolve search round \\
$r$ & calendar-day index in synthetic or real datasets \\
$n$ & period index in the standalone classical inventory model of Section~\ref{sec:exp-cbs} \\
\bottomrule
\end{tabular}
\end{table}

To improve readability, we split notation into stage-specific tables.

\begin{table}[H]
\centering
\caption{Training-stage notation}
\label{tab:notation_summary_training}
\begin{tabular}{p{0.30\textwidth}p{0.64\textwidth}}
\toprule
\textbf{Notation} & \textbf{Meaning} \\
\midrule
$K,\ \eta$ & Number of training update steps and training update step size. \\
$\widetilde{\Pi},\ \Pi,\ \mathrm{can}(\cdot),\ g(\pi)$ & Raw code space, canonical policy space, canonicalization map, and structural-validity indicator. \\
$x\sim\mathcal{D}_{\mathrm{tr}},\ \Pi^{\mathrm{base}}$ & Training instance and shared baseline-policy family used in binary reward evaluation. \\
$r(x,\pi),\ \Delta(\pi)$ & Binary training success indicator and success probability under $\mathcal{D}_{\mathrm{tr}}$. \\
$\mathcal{G}^{\mathrm{tr}},\ \tau_{\mathrm{good}},\ \Delta_{\mathrm{bad}},\ \gamma$ & Training-stage good region, training-good threshold, bad-region envelope, and separation margin. \\
$p_k,\ \widehat{\Delta}_k(\pi),\ \mathcal{E}_{\mathrm{tr}},\ \varepsilon_K,\ \rho_K$ & Proxy distribution over canonical policies, estimated success, concentration event, cumulative estimation error budget, and bad-to-good mass ratio bound. \\
\bottomrule
\end{tabular}
\end{table}

\begin{table}[H]
\centering
\caption{Inference-stage notation}
\label{tab:notation_summary_inference}
\begin{tabular}{p{0.30\textwidth}p{0.64\textwidth}}
\toprule
\textbf{Notation} & \textbf{Meaning} \\
\midrule
$J,\ \mathcal{H}_t,\ d_{t-1},\ \pi_t^{\mathrm{ref}}$ & Per-period search rounds, period-level context, previously deployed policy, and period-specific reference baseline. \\
$W_t,\ C_t(\pi;\omega),\ Z_{t,\ell}(\pi\mid \tilde\pi)$ & Replay window, pathwise replay cost, and sample-path gain. \\
$\widehat{\mu}_t(\pi\mid \tilde\pi),\ \widehat{v}_t(\pi\mid \tilde\pi),\ \operatorname{rad}_t(\pi\mid \tilde\pi)$ & Empirical replay mean, empirical replay variance, and pairwise confidence radius. \\
$m_t,\ \delta_t,\ B_t,\ N_t,\ \mathcal{C}_t$ & Replay sample size, period-level confidence budget, Hoeffding bounded-gain constant used in the canonical verification, period-level evaluated-pair budget, and replay confidence event. \\
$\{\mathcal{F}_{t,j}\}_{j=0}^J,\ \mathcal{Q}_t(\pi,\tilde\pi)$ & Period-$t$ InvEvolve filtration and appendix-only query-time sigma-field for an evaluated policy pair. \\
$\mathrm{LCB}_t,\ \mathrm{UCB}_t$ & Lower and upper confidence bounds used for screening and optimistic selection. \\
$S_t(\pi),\ I_{t,j}(\pi),\ O_{t,j}(\pi)$ & Empirical safety, improvement, and optimistic screening scores used by the round-$j$ gate. \\
$\Delta_t^{\mathrm{safe}}(\pi),\ \Delta_{t,j}^{\mathrm{imp}}(\pi)$ & Replay-population safety and improvement margins used in promotability analysis. \\
$\mathcal{G}^{\mathrm{prom}}_{t,j}(\varepsilon),\ \mathcal{E}^{\mathrm{prom}}_t(\varepsilon)$ & Round-$j$ promotable policy class and period-level certified promotion event. \\
$Q_{t,j},\ \mathcal{H}_{t,j-1},\ \pi_{t,j},\ \pi_{t,j}^{\mathrm{ch}}$ & Round-$j$ policy-generation distribution, round-level prompt context, generated candidate, and running champion. \\
$\pi_{t,0}^{\mathrm{ch}},\ \mathcal{A}_{t,0},\ \mathcal{A}_{t,j},\ \mathcal{A}_t$ & Initial champion and initial/evolving/final active pools. \\
$\mathcal{A}_{t,0}^{\mathrm{feas}},\ \mathcal{A}_t^{\mathrm{feas}},\ \mathcal{E}_t$ & Initial and final safety-feasible active subsets, together with the period-level evaluated policy--comparator pair set. \\
$q_t,\ \bar q_t$ & Effective one-round coverage lower bounds for promotable and near-oracle safe policy classes. \\
$\tau_t,\ \kappa_t,\ \bar\kappa_t$ & Appendix-only primitive constants used to verify effective coverage through contextual generation shift and overlap. \\
\bottomrule
\end{tabular}
\end{table}

\begin{table}[H]
\centering
\caption{Deployment-stage notation}
\label{tab:notation_summary_deployment}
\begin{tabular}{p{0.30\textwidth}p{0.64\textwidth}}
\toprule
\textbf{Notation} & \textbf{Meaning} \\
\midrule
$d_t$ & Policy finally deployed at period $t$. \\
$V_t^{\mathrm{rep}}(\pi\mid \tilde\pi),\ V_t^{\mathrm{dep}}(\pi\mid \tilde\pi)$ & Replay gain and deployment gain of policy $\pi$ relative to comparator $\tilde\pi$. \\
$P_t^{\mathrm{rep}},\ P_t^{\mathrm{dep}},\ \mathfrak F_t,\ D_{\mathfrak F_t},\ \xi_t$ & Replay/deployment path distributions, IPM function class, induced replay-to-deployment discrepancy, and mismatch budget. \\
$\Pi_t^{\mathrm{safe}},\ V_t^{\mathrm{safe},*}$ & Certifiably safe policy class over $\Pi$ and the corresponding oracle safe value. \\
$\Pi_t^{\mathrm{safe},\nu_t},\ \nu_t$ & Near-oracle safe policy class and search-approximation tolerance. \\
$\Gamma_t(\tilde\pi_t,d_t),\ \mathcal{D}_t,\ G_T$ & Per-period oracle-safe deployment-gap bound, period-level near-oracle discovery event, and horizon-level rolling good event. \\
\bottomrule
\end{tabular}
\end{table}

\section{Proofs and Verification of Interface Conditions}
\label{sec:appendix_proofs}

This appendix provides proofs fully aligned with the revised interface assumptions in the main text.

\subsection{Proof of Theorem~\ref{thm:train_concentration}}

For notational convenience, we relabel the \(K\) multiplicative reweighting
updates by \(k=1,\ldots,K\), so that \(p_K\) denotes the proxy distribution over canonical policies
after these \(K\) updates. This is only an indexing convention and does not
change the update rule in \eqref{eq:proxy_update}.

For any policy subset \(S\subseteq \Pi\), define
\begin{equation}
W_K(S)
\triangleq
\sum_{\pi\in S}p_0(\pi)
\exp\!\left\{\eta\sum_{k=1}^K\widehat{\Delta}_k(\pi)\right\},
\label{eq:def_WK}
\end{equation}
which represents the total unnormalized weight of all policies in \(S\) after
\(K\) multiplicative reweighting steps. By a straightforward induction argument
based on \eqref{eq:proxy_update}, we have
\[
p_K(\pi)
=
\frac{
p_0(\pi)\exp\!\left\{\eta\sum_{k=1}^K\widehat{\Delta}_k(\pi)\right\}
}{
\sum_{\pi'\in\Pi}
p_0(\pi')\exp\!\left\{\eta\sum_{k=1}^K\widehat{\Delta}_k(\pi')\right\}
}.
\]

By comparing this expression with the definition in \eqref{eq:def_WK}, we obtain that, for any subset $S\subseteq\Pi$, 
\begin{equation} 
p_K(S)=\frac{W_K(S)}{W_K(\Pi)}. \label{eq:def_pK_via_WK} 
\end{equation}

By \eqref{eq:def_pK_via_WK}, the comparison of probability masses reduces to the comparison of the corresponding unnormalized weights. Therefore, in the sequel, it is sufficient to compare $W_K(\cdot)$ across different subsets, rather than tracking the distribution $p_K$ directly, since all such quantities share the same denominator $W_K(\Pi)$.

By assumption, there exists an event $\mathcal{E}_{\mathrm{tr}}$ such that, for any $\pi$ with $p_0(\pi)>0$,
\[
\sum_{k=1}^K\widehat{\Delta}_k(\pi)
=
K\Delta(\pi)+\sum_{k=1}^K\bigl(\widehat{\Delta}_k(\pi)-\Delta(\pi)\bigr).
\]
The second term is bounded in absolute value by $\varepsilon_K$. Hence, we immediately obtain
\[
K\Delta(\pi)-\varepsilon_K
\le
\sum_{k=1}^K\widehat{\Delta}_k(\pi)
\le
K\Delta(\pi)+\varepsilon_K.
\]

Next, we apply this bound to the good region $\mathcal{G}^{\mathrm{tr}}$ and its complement.

\begin{enumerate}
    \item For $\pi\in\mathcal{G}^{\mathrm{tr}}$, the definition of $\mathcal{G}^{\mathrm{tr}}$ implies that $\Delta(\pi)\ge \tau_{\mathrm{good}}$. Therefore, $\sum_{k=1}^K\widehat{\Delta}_k(\pi)\ge K\tau_{\mathrm{good}}-\varepsilon_K$. Substituting this lower bound into \eqref{eq:def_WK} yields
    \[
    W_K(\mathcal{G}^{\mathrm{tr}})
    \ge
    p_0(\mathcal{G}^{\mathrm{tr}})
    \exp\!\bigl\{\eta(K\tau_{\mathrm{good}}-\varepsilon_K)\bigr\}.
    \]

    \item For $\pi\in\Pi\setminus\mathcal{G}^{\mathrm{tr}}$, the definition of the complement region gives $\sum_{k=1}^K\widehat{\Delta}_k(\pi)\le K\Delta_{\mathrm{bad}}+\varepsilon_K.$ By the same argument,
    \[
    W_K(\Pi\setminus\mathcal{G}^{\mathrm{tr}})
    \le
    p_0(\Pi\setminus\mathcal{G}^{\mathrm{tr}})
    \exp\!\bigl\{\eta(K\Delta_{\mathrm{bad}}+\varepsilon_K)\bigr\}.
    \]
\end{enumerate}

Now, taking the ratio yields
\[
\frac{W_K(\Pi\setminus\mathcal{G}^{\mathrm{tr}})}{W_K(\mathcal{G}^{\mathrm{tr}})}
\le
\frac{p_0(\Pi\setminus\mathcal{G}^{\mathrm{tr}})}{p_0(\mathcal{G}^{\mathrm{tr}})}
\exp\!\bigl\{-\eta(K\gamma-2\varepsilon_K)\bigr\}
=\rho_K.
\]

By \eqref{eq:def_pK_via_WK},
\[
\frac{p_K(\Pi\setminus\mathcal{G}^{\mathrm{tr}})}{p_K(\mathcal{G}^{\mathrm{tr}})}
=
\frac{W_K(\Pi\setminus\mathcal{G}^{\mathrm{tr}})}{W_K(\mathcal{G}^{\mathrm{tr}})}
\le\rho_K,
\]
proving \eqref{eq:train_ratio_bound}. Since $p_K(\mathcal{G}^{\mathrm{tr}})+p_K(\Pi\setminus\mathcal{G}^{\mathrm{tr}})=1,$ we obtain
\[
p_K(\mathcal{G}^{\mathrm{tr}})\ge\frac{1}{1+\rho_K},
\qquad
p_K(\Pi\setminus\mathcal{G}^{\mathrm{tr}})\le\frac{\rho_K}{1+\rho_K},
\]
which is \eqref{eq:train_bad_mass_bound}.

\subsection{Verifying adaptive replay confidence validity}
\label{app:adaptive-ci-verification}

The main text states Assumption~\ref{ass:adaptive_ci} directly as an
interface-level confidence requirement. This appendix provides one primitive
sufficient condition under which the assumption is valid. The condition is
sufficient but not necessary: any other replay radius or resampling procedure
that delivers the same simultaneous replay confidence event can be substituted without
changing the main proof.

The key issue is adaptivity. InvEvolve generates candidate policies using
previous replay feedback, and therefore the candidate queried at a later round
may depend on earlier replay outcomes. The following verification is valid when
the replay observations used to evaluate each realized query are conditionally
fresh relative to the information available when that query is formed. In
practice, this can be implemented by independent simulation replications,
pre-specified holdout replay blocks, sample splitting of the mature replay
window, or another valid adaptive-inference construction. If the same replay
paths are reused both to guide policy generation and to evaluate later adaptive
queries, the conditional sampling condition below should be regarded as an
additional design requirement rather than an automatic consequence of having a
mature replay window.

\begin{assumption}[Conditionally fresh replay blocks]
\label{ass:fixed_replay_appendix}
For each period \(t\), let \(\{\mathcal{F}_{t,j}\}_{j=0}^J\) denote the
InvEvolve filtration. For every realized policy--comparator query
\((\pi,\tilde\pi)\) evaluated by the algorithm in period \(t\), let
\(\mathcal{Q}_t(\pi,\tilde\pi)\) denote the query-time sigma-field, i.e., the
information available immediately before the replay observations used for this
query are revealed.

For this realized query, the replay observations
\[
\{Z_{t,\ell}(\pi\mid\tilde\pi)\}_{\ell=1}^{m_t}
\]
are conditionally independent given \(\mathcal{Q}_t(\pi,\tilde\pi)\), satisfy
\[
\mathbb{E}\!\left[
Z_{t,\ell}(\pi\mid\tilde\pi)
\mid
\mathcal{Q}_t(\pi,\tilde\pi)
\right]
=
V_t^{\mathrm{rep}}(\pi\mid\tilde\pi),
\]
and are uniformly bounded as
\[
|Z_{t,\ell}(\pi\mid\tilde\pi)|\le B_t
\quad\text{almost surely}
\]
for some finite constant \(B_t<\infty\). The replay blocks used for different
realized queries need not be mutually independent; the proof only requires the
conditional validity of each realized confidence statement at its query time.
\end{assumption}

\begin{lemma}[Conditional Hoeffding verification of Assumption~\ref{ass:adaptive_ci}]
\label{lem:adaptive_ci_appendix}
Under Assumption~\ref{ass:fixed_replay_appendix}, the period-\(t\) confidence
event satisfies
\[
\mathbb{P}(\mathcal{C}_t)\ge 1-\delta_t.
\]
\end{lemma}

\noindent\textit{Proof.}
Fix a period \(t\). The InvEvolve procedure generates at most \(N_t\) realized
confidence statements in period \(t\): the pre-loop comparisons against
\(\pi_t^{\mathrm{ref}}\) and at most two additional queried pairs per inner-loop
round. Enumerate these realized confidence statements in their chronological
order as \(r=1,\ldots,R_t\), where \(R_t\le N_t\). Let
\((\pi_{t,r},\tilde\pi_{t,r})\) denote the \(r\)-th realized query and let
\(\mathcal{Q}_{t,r}\) denote its query-time sigma-field.

For a fixed realized query \(r\), condition on \(\mathcal{Q}_{t,r}\). By
Assumption~\ref{ass:fixed_replay_appendix}, the replay observations used for
this query are conditionally independent, bounded in \([-B_t,B_t]\), and have
conditional mean
\(V_t^{\mathrm{rep}}(\pi_{t,r}\mid\tilde\pi_{t,r})\). Therefore Hoeffding's
inequality gives
\[
\mathbb{P}\!\left(
\left|
\widehat{\mu}_t(\pi_{t,r}\mid\tilde\pi_{t,r})
-
V_t^{\mathrm{rep}}(\pi_{t,r}\mid\tilde\pi_{t,r})
\right|
>
B_t\sqrt{\frac{2\log(2N_t/\delta_t)}{m_t}}
\ \middle|\
\mathcal{Q}_{t,r}
\right)
\le
\frac{\delta_t}{N_t}.
\]
Taking expectations gives the same unconditional bound for the \(r\)-th realized
confidence statement. Since there are at most \(N_t\) such statements, a union
bound yields
\[
\mathbb{P}(\mathcal{C}_t^c)
\le
\sum_{r=1}^{R_t}\frac{\delta_t}{N_t}
\le
\delta_t.
\]
Hence \(\mathbb{P}(\mathcal{C}_t)\ge 1-\delta_t\).

\subsection{Verifying effective search coverage from primitive conditions}
\label{app:effective-coverage-sufficient}

The main text uses Assumption~\ref{ass:effective_coverage} as the interface between the trained policy-generation distribution and the deployment search problem. These primitive conditions are one way to verify effective coverage; they are not required by the main theorems. This appendix provides one set of primitive sufficient conditions under which this interface condition holds. The argument uses the training concentration guarantee in Theorem~\ref{thm:train_concentration}, a bounded contextual shift condition between the inference-time policy-generation distribution and the trained proxy distribution over canonical policies, and overlap between the training-good region and the deployment target policy classes.

\begin{assumption}[Primitive condition: bounded contextual generation shift]
\label{ass:uniform_ci_appendix}
For each period $t$, there exists a $\tau_t \in [0,1]$ such that, for each round $j$, conditional on $\mathcal{F}_{t,j-1}$, the round-$j$ policy-generation distribution $Q_{t,j} := Q_\theta(\cdot\mid \mathcal{H}_{t,j-1},\pi_{t,j-1}^{\mathrm{ch}},\mathcal{A}_{t,j-1},\mathrm{stats}_{t,j-1})$ satisfies, almost surely,
\[
\operatorname{TV}(Q_{t,j},p_K)
:=
\sup_{A\subseteq\Pi}|Q_{t,j}(A)-p_K(A)|
\le \tau_t.
\]
\end{assumption}

\begin{assumption}[Primitive condition: overlap with the promotable policy class]
\label{ass:overlap_promotable_appendix}
There exists a deterministic constant $\kappa_t\in[0,1]$ such that, for every round $j$, almost surely,
\[
p_K(\mathcal{G}^{\mathrm{prom}}_{t,j}(\varepsilon)\mid\mathcal{G}^{\mathrm{tr}})
\ge \kappa_t.
\]
\end{assumption}

\begin{assumption}[Primitive condition: overlap with the near-oracle safe policy class]
\label{ass:overlap_safe_appendix}
There exists a deterministic constant $\bar\kappa_t\in[0,1]$ such that, almost surely,
\[
p_K(\Pi_t^{\mathrm{safe},\nu_t}\mid\mathcal{G}^{\mathrm{tr}})
\ge\bar\kappa_t.
\]
\end{assumption}

\begin{lemma}[Sufficient conditions for effective search coverage]
\label{lem:sufficient_effective_coverage}
Condition on any training realization in $\mathcal{E}_{\mathrm{tr}}$. Suppose the primitive conditions in Assumptions~\ref{ass:uniform_ci_appendix}, \ref{ass:overlap_promotable_appendix}, and \ref{ass:overlap_safe_appendix} hold. Then Assumption~\ref{ass:effective_coverage} holds with
\[
q_t=
\left(\frac{\kappa_t}{1+\rho_K}-\tau_t\right)_+,
\qquad
\bar q_t=
\left(\frac{\bar\kappa_t}{1+\rho_K}-\tau_t\right)_+.
\]
\end{lemma}

\noindent\textit{Proof.} Condition on any training realization in $\mathcal{E}_{\mathrm{tr}}$. By Theorem~\ref{thm:train_concentration}, we have $p_K(\mathcal{G}^{\mathrm{tr}})\ge\frac{1}{1+\rho_K}.$ For the promotable policy class, Assumption~\ref{ass:overlap_promotable_appendix} gives
\[
p_K\bigl(\mathcal{G}^{\mathrm{prom}}_{t,j}(\varepsilon)\bigr)
\ge p_K\left(\mathcal{G}^{\mathrm{prom}}_{t,j}(\varepsilon) \cap \mathcal{G}^{\mathrm{tr}}\right)=
p_K(\mathcal{G}^{\mathrm{tr}})\,p_K\bigl(\mathcal{G}^{\mathrm{prom}}_{t,j}(\varepsilon)\mid\mathcal{G}^{\mathrm{tr}}\bigr)
\ge
\frac{\kappa_t}{1+\rho_K}.
\]

According to Assumption~\ref{ass:uniform_ci_appendix}, we have $\operatorname{TV}(Q_{t,j},p_K)=\sup_{A\subseteq\Pi}|Q_{t,j}(A)-p_K(A)|,$ which implies that for every subset $A\subseteq\Pi$,
\[
Q_{t,j}(A)\ge p_K(A)-\tau_t.
\]
Taking $A=\mathcal{G}^{\mathrm{prom}}_{t,j}(\varepsilon)$ yields
\[
Q_{t,j}\bigl(\mathcal{G}^{\mathrm{prom}}_{t,j}(\varepsilon)\bigr)
\ge
\frac{\kappa_t}{1+\rho_K}-\tau_t.
\]
Because probabilities are nonnegative, then we obtain
\[
Q_{t,j}\bigl(\mathcal{G}^{\mathrm{prom}}_{t,j}(\varepsilon)\bigr)
\ge
\max\!\left\{\frac{\kappa_t}{1+\rho_K}-\tau_t,0\right\}
= q_t.
\]

For the near-oracle safe policy class, Assumption~\ref{ass:overlap_safe_appendix} similarly gives
\[
p_K\bigl(\Pi_t^{\mathrm{safe},\nu_t}\bigr)
\ge
p_K(\mathcal{G}^{\mathrm{tr}})\,p_K\bigl(\Pi_t^{\mathrm{safe},\nu_t}\mid\mathcal{G}^{\mathrm{tr}}\bigr)
\ge
\frac{\bar\kappa_t}{1+\rho_K}.
\]
Applying Assumption~\ref{ass:uniform_ci_appendix} with $A=\Pi_t^{\mathrm{safe},\nu_t}$ yields $Q_{t,j}\bigl(\Pi_t^{\mathrm{safe},\nu_t}\bigr)
\ge
\frac{\bar\kappa_t}{1+\rho_K}-\tau_t.$ Again using nonnegativity of probabilities,
\[
Q_{t,j}\bigl(\Pi_t^{\mathrm{safe},\nu_t}\bigr)
\ge
\max\!\left\{\frac{\bar\kappa_t}{1+\rho_K}-\tau_t,0\right\}
= \bar q_t.
\]
This proves the lemma.

\subsection{Auxiliary lemmas for inference and deployment}

\begin{lemma}[IPM transfer]
\label{lem:ipm_transfer_appendix}
For every relevant pair $(\pi,\tilde\pi)$, $\left|V_t^{\mathrm{dep}}(\pi\mid \tilde\pi)-V_t^{\mathrm{rep}}(\pi\mid \tilde\pi)\right|\le\xi_t.$
\end{lemma}

\noindent\textit{Proof.} By definition,
\[
V_t^{\mathrm{dep}}(\pi\mid \tilde\pi)-V_t^{\mathrm{rep}}(\pi\mid \tilde\pi)
=
\mathbb{E}_{P_t^{\mathrm{dep}}}[z_{t,\pi,\tilde\pi}] - \mathbb{E}_{P_t^{\mathrm{rep}}}[z_{t,\pi,\tilde\pi}],
\]
and since $z_{t,\pi,\tilde\pi}\in\mathfrak F_t$, the absolute value is at most
\[
\sup_{f\in\mathfrak F_t}\left|\mathbb{E}_{P_t^{\mathrm{dep}}}f-\mathbb{E}_{P_t^{\mathrm{rep}}}f\right|=\xi_t.
\]

\subsection{Proof of Lemma~\ref{lem:gate_correctness}}

Recall that $\mathcal{G}_{t, j}^{\mathrm{prom}}(\varepsilon) \triangleq\left\{\pi \in \Pi: g(\pi)=1, \Delta_t^{\mathrm{safe}}(\pi) \geq 0, \Delta_{t, j}^{\mathrm{imp}}(\pi) \geq 0\right\}$. For any $\pi\in\mathcal{G}^{\mathrm{prom}}_{t,j}(\varepsilon)$, expanding $\Delta_t^{\mathrm{safe}}(\pi) \geq 0$ and $\Delta_{t, j}^{\mathrm{imp}}(\pi) \geq 0$ gives:
\begin{equation}
V_t^{\mathrm{rep}}(\pi\mid \pi_t^{\mathrm{ref}})
\ge
\xi_t + 2\,\operatorname{rad}_t(\pi\mid \pi_t^{\mathrm{ref}}),
\label{eq:detailed-safe-replay}
\end{equation}
\begin{equation}
V_t^{\mathrm{rep}}(\pi\mid \pi_{t,j-1}^{\mathrm{ch}})
\ge
\varepsilon + \xi_t + 2\,\operatorname{rad}_t(\pi\mid \pi_{t,j-1}^{\mathrm{ch}}).
\label{eq:detailed-imp-replay}
\end{equation}
We now prove the claim in two steps.

\paragraph{Step 1: On \(\mathcal{C}_t\), every promotable candidate passes the replay gate.}
Assume throughout this step that \(\mathcal{C}_t\) holds and that \(\pi\in\mathcal{G}_{t,j}^{\mathrm{prom}}(\varepsilon)\).
Because \(\mathcal{C}_t\) holds, the replay estimate against the reference baseline satisfies
\[
\left|
\widehat{\mu}_t(\pi\mid \pi_t^{\mathrm{ref}})
-
V_t^{\mathrm{rep}}(\pi\mid \pi_t^{\mathrm{ref}})
\right|
\le
\operatorname{rad}_t(\pi\mid \pi_t^{\mathrm{ref}}) \Rightarrow
\widehat{\mu}_t(\pi\mid \pi_t^{\mathrm{ref}})
\ge
V_t^{\mathrm{rep}}(\pi\mid \pi_t^{\mathrm{ref}})
-
\operatorname{rad}_t(\pi\mid \pi_t^{\mathrm{ref}}).
\]
Subtracting another radius from both sides yields
\begin{align*}
\mathrm{LCB}_t(\pi\mid \pi_t^{\mathrm{ref}})
=
\widehat{\mu}_t(\pi\mid \pi_t^{\mathrm{ref}})
-
\operatorname{rad}_t(\pi\mid \pi_t^{\mathrm{ref}}) \ge
V_t^{\mathrm{rep}}(\pi\mid \pi_t^{\mathrm{ref}})
-
2\,\operatorname{rad}_t(\pi\mid \pi_t^{\mathrm{ref}}).
\end{align*}
Now apply \eqref{eq:detailed-safe-replay}, which implies $V_t^{\mathrm{rep}}(\pi\mid \pi_t^{\mathrm{ref}})-2\,\operatorname{rad}_t(\pi\mid \pi_t^{\mathrm{ref}})\ge\xi_t,$ combining the last two displays gives $\mathrm{LCB}_t(\pi\mid \pi_t^{\mathrm{ref}})\ge \xi_t.$ Thus the \emph{safety gate} is passed.

Exactly the same reasoning applies to the comparison with the current champion. Since \(\mathcal{C}_t\) holds, we can also obtain
\[
\widehat{\mu}_t(\pi\mid \pi_{t,j-1}^{\mathrm{ch}})
\ge
V_t^{\mathrm{rep}}(\pi\mid \pi_{t,j-1}^{\mathrm{ch}})
-
\operatorname{rad}_t(\pi\mid \pi_{t,j-1}^{\mathrm{ch}}).
\]
Subtracting one more radius gives
\begin{align*}
\mathrm{LCB}_t(\pi\mid \pi_{t,j-1}^{\mathrm{ch}})
=
\widehat{\mu}_t(\pi\mid \pi_{t,j-1}^{\mathrm{ch}})
-
\operatorname{rad}_t(\pi\mid \pi_{t,j-1}^{\mathrm{ch}}) 
\geq V_t^{\mathrm{rep}}(\pi\mid \pi_{t,j-1}^{\mathrm{ch}})
-
2\,\operatorname{rad}_t(\pi\mid \pi_{t,j-1}^{\mathrm{ch}}).
\end{align*}
Using \eqref{eq:detailed-imp-replay}, we obtain
\[
V_t^{\mathrm{rep}}(\pi\mid \pi_{t,j-1}^{\mathrm{ch}})
-
2\,\operatorname{rad}_t(\pi\mid \pi_{t,j-1}^{\mathrm{ch}})
\ge
\varepsilon+\xi_t,
\]
and therefore $\mathrm{LCB}_t(\pi\mid \pi_{t,j-1}^{\mathrm{ch}})
\ge
\varepsilon+\xi_t.$ Thus the \emph{improvement gate} is also passed. Since both gate inequalities hold, the candidate promotion rule accepts and promotes \(\pi\).

\paragraph{Step 2: A promoted promotable candidate is safe and genuinely improving under deployment.}

We now transfer the replay guarantees to deployment. By Lemma~\ref{lem:ipm_transfer_appendix},
for every relevant pair \((\pi,\widetilde{\pi})\), we have $
\left|
V_t^{\mathrm{dep}}(\pi\mid \widetilde{\pi})
-
V_t^{\mathrm{rep}}(\pi\mid \widetilde{\pi})
\right|
\le
\xi_t.
$ Equivalently, 
\[
V_t^{\mathrm{dep}}(\pi\mid \widetilde{\pi})
\ge
V_t^{\mathrm{rep}}(\pi\mid \widetilde{\pi})-\xi_t.
\]

Apply this first with \(\widetilde{\pi}=\pi_t^{\mathrm{ref}}\). Then we have $
V_t^{\mathrm{dep}}(\pi\mid \pi_t^{\mathrm{ref}})
\ge
V_t^{\mathrm{rep}}(\pi\mid \pi_t^{\mathrm{ref}})-\xi_t.
$. Using \eqref{eq:detailed-safe-replay}, we obtain
\[
V_t^{\mathrm{dep}}(\pi\mid \pi_t^{\mathrm{ref}})
\ge
\bigl(\xi_t + 2\,\operatorname{rad}_t(\pi\mid \pi_t^{\mathrm{ref}})\bigr)-\xi_t
=
2\,\operatorname{rad}_t(\pi\mid \pi_t^{\mathrm{ref}})
\ge 0.
\]
Hence \(\pi\) is deployment-safe relative to the reference baseline: $V_t^{\mathrm{dep}}(\pi\mid \pi_t^{\mathrm{ref}})\ge 0.$

Next apply the same transfer inequality with
\(\widetilde{\pi}=\pi_{t,j-1}^{\mathrm{ch}}\). Then
\[
V_t^{\mathrm{dep}}(\pi\mid \pi_{t,j-1}^{\mathrm{ch}})
\ge
V_t^{\mathrm{rep}}(\pi\mid \pi_{t,j-1}^{\mathrm{ch}})-\xi_t.
\]
Invoking \eqref{eq:detailed-imp-replay}, we get
\[
V_t^{\mathrm{dep}}(\pi\mid \pi_{t,j-1}^{\mathrm{ch}})
\ge
\bigl(\varepsilon+\xi_t+2\,\operatorname{rad}_t(\pi\mid \pi_{t,j-1}^{\mathrm{ch}})\bigr)-\xi_t
=
\varepsilon+2\,\operatorname{rad}_t(\pi\mid \pi_{t,j-1}^{\mathrm{ch}})
\ge
\varepsilon.
\]
Therefore the promoted candidate is not only replay-certified but also truly better under deployment: $V_t^{\mathrm{dep}}(\pi\mid \pi_{t,j-1}^{\mathrm{ch}})\ge \varepsilon.$

Combining Steps 1 and 2 proves the claim: if \(\mathcal{C}_t\) holds and
\(\pi\in \mathcal{G}_{t,j}^{\mathrm{prom}}(\varepsilon)\), then the candidate promotion rule promotes \(\pi\), and the promoted
policy satisfies
\[
V_t^{\mathrm{dep}}(\pi\mid \pi_t^{\mathrm{ref}})\ge 0,
\qquad
V_t^{\mathrm{dep}}(\pi\mid \pi_{t,j-1}^{\mathrm{ch}})\ge \varepsilon.
\]

\subsection{Proof of Theorem~\ref{thm:single_month_promotion}}

Condition throughout on the event $\mathcal{E}_{\mathrm{tr}}$. The proof relies on Assumption~\ref{ass:effective_coverage} and Lemma~\ref{lem:gate_correctness}. Since the initial comparator $\pi_{t,0}^{\mathrm{ch}}$ may already be the strongest safety certified baseline or incumbent fallback policy, successful promotion is benchmarked against this stronger period-initial fallback. Define
\[
F_{t,0}^{\mathrm{prom}}:=\Omega,
\qquad
F_{t,j}^{\mathrm{prom}}
:=
\bigcap_{r=1}^j\{\pi_{t,r}\notin\mathcal{G}^{\mathrm{prom}}_{t,r}(\varepsilon)\},
\qquad j=1,\ldots,J.
\]
The event $F_{t,j}^{\mathrm{prom}}$ means that, over the first $j$ rounds, none of the generated candidates belongs to the corresponding promotable policy class. Equivalently,
\[
F_{t,J}^{\mathrm{prom}}
=
\Bigl\{\forall j\le J:\ \pi_{t,j}\notin\mathcal{G}^{\mathrm{prom}}_{t,j}(\varepsilon)\Bigr\}
=
\mathcal{S}_t^c, \text{ where } 
\mathcal{S}_t
:=
\{\exists j\le J:\ \pi_{t,j}\in\mathcal{G}^{\mathrm{prom}}_{t,j}(\varepsilon)\}.
\]

By Assumption~\ref{ass:effective_coverage}, for each round $j$,
\[
\mathbb{P}\bigl(\pi_{t,j}\in\mathcal{G}^{\mathrm{prom}}_{t,j}(\varepsilon)\mid\mathcal{F}_{t,j-1}\bigr)
=
Q_{t,j}\bigl(\mathcal{G}^{\mathrm{prom}}_{t,j}(\varepsilon)\bigr)
\ge q_t.
\]
Therefore,
\[
\mathbb{P}(F_{t,j}^{\mathrm{prom}})
=
\mathbb{E}\!\left[
\mathbf{1}_{F_{t,j-1}^{\mathrm{prom}}}
\mathbb{P}\bigl(\pi_{t,j}\notin\mathcal{G}^{\mathrm{prom}}_{t,j}(\varepsilon)\mid\mathcal{F}_{t,j-1}\bigr)
\right]
\le
(1-q_t)\mathbb{P}(F_{t,j-1}^{\mathrm{prom}}).
\]
Here, $F_{t,j-1}^{\mathrm{prom}}$ is $\mathcal{F}_{t,j-1}$-measurable. Moreover, since the conditional success probability is bounded below by $q_t$, the corresponding conditional failure probability is bounded above by $1-q_t$.

Applying this inequality recursively gives
\[
\mathbb{P}(\mathcal{S}_t^c)
=
\mathbb{P}(F_{t,J}^{\mathrm{prom}})
\le
(1-q_t)^J.
\]

Next, Assumption~\ref{ass:adaptive_ci} implies that
\[
\mathbb{P}(\mathcal{C}_t)\ge 1-\delta_t.
\]
On the event $\mathcal{C}_t\cap\mathcal{S}_t$, Lemma~\ref{lem:gate_correctness} guarantees the existence of a promoted candidate that satisfies deployment safety and achieves $\varepsilon$-improvement. Hence, $\mathcal{C}_t\cap\mathcal{S}_t
\subseteq
\mathcal{E}^{\mathrm{prom}}_t(\varepsilon).$ It follows that
\[
\mathbb{P}\bigl(\mathcal{E}^{\mathrm{prom}}_t(\varepsilon)\bigr)
\ge
1-\delta_t-(1-q_t)^J,
\]
which proves Theorem~\ref{thm:single_month_promotion}.

\subsection{Proof of Theorem~\ref{thm:rolling_regret}}

Condition on $\mathcal{E}_{\mathrm{tr}}$. We first control the discovery event $\mathcal{D}_t$. Define
\[
\mathcal{D}_t:=\{\exists j\le J:\pi_{t,j}\in\Pi_t^{\mathrm{safe},\nu_t}\},
\qquad
F_{t,0}^{\mathrm{safe}}:=\Omega,
\]
$\mathcal{D}_t$ can be interpreted as the event that, within the $J$ search rounds of period $t$, the search process hits at least one candidate in the near-oracle safe policy class $\Pi_{t}^{\mathrm{safe},\nu_t}$. This discovery event is sufficient but not necessary, because a near-oracle safe baseline or incumbent may already be present in $\mathcal A_{t,0}$. For $j=1,\ldots,J$,
\[
F_{t,j}^{\mathrm{safe}}
:=
\bigcap_{r=1}^j\{\pi_{t,r}\notin\Pi_t^{\mathrm{safe},\nu_t}\}.
\]
Then $F_{t,J}^{\mathrm{safe}}=\mathcal{D}_t^c$. 

Recall that $\mathcal{C}_t$ is the period-$t$ replay confidence event under which all replay estimates used by the algorithm are simultaneously accurate. Thus, $\mathcal{C}_t$ guarantees that replay-based certificates are reliable, and $\mathcal{D}_t$ guarantees that the active pool contains a candidate that is close to the best policy in the certifiably safe policy class. 

Accordingly, the proof proceeds in three steps:
\begin{enumerate}
    \item First, we bound the probability of $\mathcal{D}_t$, which corresponds to the event that the near-oracle safe policy class is hit at least once within period $t$.
    \item Second, on the event $\mathcal{C}_t$, we establish deployment safety.
    \item Third, on the event $\mathcal{C}_t \cap \mathcal{D}_t$, we derive the oracle-safe deployment gap bound.
\end{enumerate}

\medskip
\noindent
\textbf{Step 1: Control Probability of  $\mathcal{D}_t$.}

By Assumption~\ref{ass:effective_coverage}, for each round $j$, $\mathbb{P}(\pi_{t,j}\in\Pi_t^{\mathrm{safe},\nu_t}\mid\mathcal{F}_{t,j-1})
=
Q_{t,j}\bigl(\Pi_t^{\mathrm{safe},\nu_t}\bigr)
\ge \bar q_t.$ 
Hence
\[
\mathbb{P}(F_{t,j}^{\mathrm{safe}})
=
\mathbb{E}\!\left[
\mathbf{1}_{F_{t,j-1}^{\mathrm{safe}}}
\mathbb{P}(\pi_{t,j}\notin\Pi_t^{\mathrm{safe},\nu_t}\mid\mathcal{F}_{t,j-1})
\right]
\le
(1-\bar q_t)\mathbb{P}(F_{t,j-1}^{\mathrm{safe}}),
\]
so by induction,
\[
\mathbb{P}(\mathcal{D}_t^c)=\mathbb{P}(F_{t,J}^{\mathrm{safe}})\le(1-\bar q_t)^J.
\]
By Assumption~\ref{ass:adaptive_ci}, $\mathbb{P}(\mathcal{C}_t)\ge1-\delta_t$.

\medskip
\noindent
\textbf{Step 2: Safety on $\mathcal{C}_t$.}
If $\mathcal{A}_t^{\mathrm{feas}}=\varnothing$, then by definition $d_t=\pi_t^{\mathrm{ref}}$, and therefore $V_t^{\mathrm{dep}}(d_t\mid\pi_t^{\mathrm{ref}})=0.$ Suppose now that $\mathcal{A}_t^{\mathrm{feas}}\neq\varnothing$. Then $d_t\in\mathcal{A}_t^{\mathrm{feas}}$, based on definition of $\mathcal{A}_t^{\mathrm{feas}}$, we have:
\[
\mathrm{LCB}_t(d_t\mid\pi_t^{\mathrm{ref}})\ge\xi_t.
\]
Since $\mathcal{C}_t$ holds, by definition we have $\left|\widehat{\mu}_t\left(d_t \mid \pi_t^{\mathrm{ref}}\right)-V_t^{\mathrm{rep}}\left(d_t \mid \pi_t^{\mathrm{ref}}\right)\right| \leq \operatorname{rad}_t\left(d_t \mid \pi_t^{\mathrm{ref}}\right),$ which implies
$$
V_t^{\mathrm{rep}}\left(d_t \mid \pi_t^{\mathrm{ref}}\right) \geq \widehat{\mu}_t\left(d_t \mid \pi_t^{\mathrm{ref}}\right)-\operatorname{rad}_t\left(d_t \mid \pi_t^{\mathrm{ref}}\right)=\operatorname{LCB}_t\left(d_t \mid \pi_t^{\mathrm{ref}}\right) \ge\xi_t.
$$

By Lemma~\ref{lem:ipm_transfer_appendix}, we have $V_t^{\mathrm{dep}}(d_t\mid\pi_t^{\mathrm{ref}})\ge0.$ The deployment rule itself is safety screened. It does not select arbitrarily from the entire active pool. Instead, it only selects from the set of candidates satisfying $\mathrm{LCB}\ge \xi_t$. Here, $\xi_t$ is chosen to match the worst-case replay-to-deployment mismatch budget. Therefore, on the event $\mathcal{C}_t$, any candidate that appears safe in replay is also guaranteed to be no worse than the baseline after deployment.

\medskip
\noindent
\textbf{Step 3: Oracle-safe deployment-gap bound on $\mathcal{C}_t\cap\mathcal{D}_t$.}
Assume now that \(\mathcal{C}_t\cap\mathcal{D}_t\) holds. Then there exists
some round \(j^*\le J\) such that
\[
\pi_{t,j^*}\in\Pi_t^{\mathrm{safe},\nu_t}.
\]
Because every structurally valid generated candidate is inserted into the active
pool and
\[
\Pi_t^{\mathrm{safe},\nu_t}
\subseteq
\{\pi_t^{\mathrm{ref}}\}\cup\{\pi:g(\pi)=1\},
\]
the intersection
\[
\Pi_t^{\mathrm{safe},\nu_t}\cap\mathcal{A}_t
\]
is nonempty on \(\mathcal{D}_t\). Let \(\tilde\pi_t\) be any
\(\mathcal{F}_{t,J}\)-measurable selector from this intersection.
If \(V_t^{\mathrm{safe},*}\) is defined as a supremum rather than a maximum,
the argument is interpreted for any \(\nu_t>0\) such that
\(\Pi_t^{\mathrm{safe},\nu_t}\) is nonempty. When \(\nu_t=0\), this corresponds
to the case where the safe oracle value is attained.
If $\tilde\pi_t=\pi_t^{\mathrm{ref}}$, then
\[
V_t^{\mathrm{safe},*}-V_t^{\mathrm{dep}}(d_t\mid\pi_t^{\mathrm{ref}})
\le
V_t^{\mathrm{safe},*}-V_t^{\mathrm{dep}}(\pi_t^{\mathrm{ref}}\mid\pi_t^{\mathrm{ref}})
\le \nu_t
\le \Gamma_t(\tilde\pi_t,d_t),
\]
where the first inequality uses results in Step~2, which states that $d_t$ is not worse than $\pi_t^{\mathrm{ref}}$,  and the second uses $\tilde\pi_t\in\Pi_t^{\mathrm{safe},\nu_t}$. Therefore, it remains to consider the case $\tilde\pi_t\neq\pi_t^{\mathrm{ref}}$.

Because $\tilde\pi_t\in\Pi_t^{\mathrm{safe},\nu_t}\setminus\{\pi_t^{\mathrm{ref}}\}\subseteq\Pi_t^{\mathrm{safe}}$, we have $\Delta_t^{\mathrm{safe}}(\tilde\pi_t)\ge0$, that is,
\[
V_t^{\mathrm{rep}}(\tilde\pi_t\mid\pi_t^{\mathrm{ref}})
\ge
\xi_t+2\,\operatorname{rad}_t(\tilde\pi_t\mid\pi_t^{\mathrm{ref}}).
\]
On $\mathcal{C}_t$, this implies
\[
\mathrm{LCB}_t(\tilde\pi_t\mid\pi_t^{\mathrm{ref}})
\ge
V_t^{\mathrm{rep}}(\tilde\pi_t\mid\pi_t^{\mathrm{ref}})-2\,\operatorname{rad}_t(\tilde\pi_t\mid\pi_t^{\mathrm{ref}})
\ge\xi_t,
\]
so $\tilde\pi_t\in\mathcal{A}_t^{\mathrm{feas}}$. By the deployment rule, $d_t$ maximizes $\mathrm{UCB}_t(\cdot\mid\pi_t^{\mathrm{ref}})$ over $\mathcal{A}_t^{\mathrm{feas}}$, so
\[
\widehat{\mu}_t(d_t\mid\pi_t^{\mathrm{ref}})+\operatorname{rad}_t(d_t\mid\pi_t^{\mathrm{ref}})
\ge
\widehat{\mu}_t(\tilde\pi_t\mid\pi_t^{\mathrm{ref}})+\operatorname{rad}_t(\tilde\pi_t\mid\pi_t^{\mathrm{ref}}).
\]
Since \(\mathcal{C}_t\) holds, we have
\[
V_t^{\mathrm{rep}}(\tilde\pi_t\mid\pi_t^{\mathrm{ref}})
\le
\widehat{\mu}_t(\tilde\pi_t\mid\pi_t^{\mathrm{ref}})
+
\operatorname{rad}_t(\tilde\pi_t\mid\pi_t^{\mathrm{ref}})
\]
and
\[
V_t^{\mathrm{rep}}(d_t\mid\pi_t^{\mathrm{ref}})
\ge
\widehat{\mu}_t(d_t\mid\pi_t^{\mathrm{ref}})
-
\operatorname{rad}_t(d_t\mid\pi_t^{\mathrm{ref}}).
\]
Combining these two inequalities with the UCB optimality of \(d_t\) gives
\[
V_t^{\mathrm{rep}}(\tilde\pi_t\mid\pi_t^{\mathrm{ref}})
-
V_t^{\mathrm{rep}}(d_t\mid\pi_t^{\mathrm{ref}})
\le
2\,\operatorname{rad}_t(\tilde\pi_t\mid\pi_t^{\mathrm{ref}})
+
2\,\operatorname{rad}_t(d_t\mid\pi_t^{\mathrm{ref}}).
\]
By Lemma~\ref{lem:ipm_transfer_appendix},
\[
V_t^{\mathrm{dep}}(\tilde\pi_t\mid\pi_t^{\mathrm{ref}})-V_t^{\mathrm{dep}}(d_t\mid\pi_t^{\mathrm{ref}})
\le 2\,\operatorname{rad}_t(\tilde\pi_t\mid\pi_t^{\mathrm{ref}})+2\,\operatorname{rad}_t(d_t\mid\pi_t^{\mathrm{ref}})+2\xi_t.
\]
By definition of $\Pi_t^{\mathrm{safe},\nu_t}$, we have  $V_t^{\mathrm{safe},*}-V_t^{\mathrm{dep}}(\tilde\pi_t\mid\pi_t^{\mathrm{ref}})\le\nu_t.$ Hence
\begin{equation}
\label{eq:Gamma_t eq}
    V_t^{\mathrm{safe},*}-V_t^{\mathrm{dep}}(d_t\mid\pi_t^{\mathrm{ref}})
\le \nu_t+2\,\operatorname{rad}_t(\tilde\pi_t\mid\pi_t^{\mathrm{ref}})+2\,\operatorname{rad}_t(d_t\mid\pi_t^{\mathrm{ref}})+2\xi_t.
\end{equation}

Finally, define $G_T:=\bigcap_{t=1}^T(\mathcal{C}_t\cap\mathcal{D}_t).$ The previous results show that $\mathbb{P}\left(\mathcal{D}_t^c\right) \leq\left(1-\bar{q}_t\right)^J$ and $\mathbb{P}\left(\mathcal{C}_t^c\right) \leq \delta_t.$
By a union bound,
\[
\mathbb{P}(G_T)
\ge 1-\sum_{t=1}^T\left[\delta_t+(1-\bar q_t)^J\right].
\]
On $G_T$, Step~2 gives safety for every period, and Step~3 gives the per-period oracle-safe deployment-gap bound. Summing \eqref{eq:Gamma_t eq} over $t=1,\ldots,T$ yields
\[
\sum_{t=1}^T\left[V_t^{\mathrm{safe},*}-V_t^{\mathrm{dep}}(d_t\mid\pi_t^{\mathrm{ref}})\right]
\le
\sum_{t=1}^T\Gamma_t(\tilde\pi_t,d_t),
\]
which proves Theorem~\ref{thm:rolling_regret}.

\section{Training Configuration}
\label{app:rl_config}

We implemented the reinforcement learning training on a hardware platform consisting of eight NVIDIA H200 GPUs, featuring a memory configuration of $8 \times 141$ GB VRAM. The detailed configuration is listed in Table \ref{tab:rl_config}.

\begin{table}[!htbp]
\centering
\caption{Reinforcement learning training configuration}
\label{tab:rl_config}
\begin{tabular}{ll}
\toprule
\textbf{Parameter} & \textbf{Value} \\
\midrule
\multicolumn{2}{l}{\textit{Model}} \\
Base model & GLM-4.7-Flash (30B MoE, 3B active) \\
\midrule
\multicolumn{2}{l}{\textit{RL Algorithm}} \\
Algorithm & GRPO \\
Group size ($N$) & 4 \\
Clip range ($\epsilon$) & 0.2 \\
KL penalty coefficient & 0.001 \\
Reward & Binary $\{0, 1\}$ \\
\midrule
\multicolumn{2}{l}{\textit{Optimization}} \\
Optimizer & Adam ($\beta_1{=}0.9, \beta_2{=}0.98$) \\
Learning rate & $5 \times 10^{-6}$ (constant) \\
Weight decay & 0.1 \\
Gradient clipping & 1.0 \\
\midrule
\multicolumn{2}{l}{\textit{Rollout}} \\
Prompts per step & 2 \\
Rollouts per step & 8 ($2 \times 4$) \\
Gradient steps & 5 \\
Max response length & 81{,}920 tokens \\
Sampling temperature & 1.0 \\
Agent tool-call budget ($J$) & 60 \\
\bottomrule
\end{tabular}
\end{table}

\section{Synthetic Data Construction Details}
\label{app:synthetic_data_generation}

This appendix records the additional technical details of the synthetic generator beyond the summary in Section~\ref{subsec:synthetic_data}. The synthetic data used in this study, together with the inference traces of InvEvolve, will be released.

\subsection{Generation Primitives and Covariate Construction}
\label{app:synthetic_primitives}

We construct \textbf{47 synthetic seed datasets}, each spanning daily observations from January 1, 2024 to December 31, 2025. The datasets cover multiple inventory-relevant environments, including consumer products, industrial spare parts, restaurant ingredients, medical supplies, power-grid load, cloud-computing demand, EV battery supply chains, and construction materials.

Let $r=1,\dots,R$ index calendar days, where $R=731$. For each seed dataset $j$ and day $r$, we generate an observed covariate vector $x_{j,r}$, an optional textual note $n_{j,r}$, a latent regime variable $z_{j,r}$, an event-state variable $e_{j,r}$, and realized demand $y_{j,r}$.

The observed covariates are derived from a shared latent environment and then instantiated in a dataset-specific manner. The shared primitives include:
\begin{itemize}
    \item \textbf{calendar signals}, such as month, day-of-week, weekend indicators, and selected derived flags;
    \item \textbf{weather signals}, generated from seasonal components, autoregressive disturbances, and occasional extreme events;
    \item \textbf{promotion and festival signals}, formed by event-centered promotional windows with decaying intensity and superimposed short-lived shocks;
    \item \textbf{macro or cost indices}, generated as stochastic trends with occasional spikes or drawdowns.
\end{itemize}

For example, temperature is generated as
\[
\mathrm{Temp}_r
=
\alpha_0
+
\alpha_1 \sin\!\left(\frac{2\pi(d_r-\phi)}{365.25}\right)
+
u_r
+
\varepsilon_r^{\mathrm{w}},
\]
where $d_r$ is the day-of-year, $u_r$ is an AR(1) disturbance, and $\varepsilon_r^{\mathrm{w}}$ captures rare weather shocks. Other weather variables, such as humidity, precipitation, and UV index, are then generated conditionally on season and temperature.

The final covariate schema is not shared across datasets. Instead, each seed dataset selects and transforms a domain-relevant subset of the shared primitives and augments them with domain-specific proxies. Representative examples include weather and promotion variables for retail-like products, utilization and maintenance proxies for industrial spare parts, admission and influenza signals for medical supplies, and traffic or latency signals for digital services. To avoid a trivial mapping from domain to schema, we apply \emph{feature-subset randomization}: after a candidate feature pool is constructed, each dataset retains only a random subset of features, while always preserving identifiers, notes, and the target demand variable.

\subsection{Demand Law, Hidden States, and Nonstationarity}
\label{app:synthetic_demand_law}

For each dataset $j$ and day $r$, demand is drawn from a time-varying distribution
\[
y_{j,r} \sim \mathcal{D}_{j,r}\!\bigl(\theta_{j,r}\bigr),
\]
where both the family $\mathcal{D}_{j,r}$ and parameter vector $\theta_{j,r}$ may vary across datasets and over time. The generator allows several demand families, including:
\begin{itemize}
    \item negative binomial distributions for over-dispersed count demand;
    \item zero-inflated or hurdle-like models for intermittent demand;
    \item mixture models for regime-switching demand;
    \item continuous positive demand for abstract entities such as regional power load.
\end{itemize}

In the count-demand setting, the conditional mean is specified as
\[
\log \mu_{j,r}
=
\beta_{j,0}(r)
+
\beta_j(r)^\top x_{j,r}
+
\gamma_j^\top z_{j,r}
+
\delta_j^\top e_{j,r},
\]
where $\beta_{j,0}(r)$ is a time-varying intercept, $\beta_j(r)$ is a potentially time-varying coefficient vector, $z_{j,r}$ denotes latent demand regimes, and $e_{j,r}$ denotes the event-state process. Demand may then be sampled, for example, from
\[
y_{j,r} \sim \mathrm{NB}\bigl(\mu_{j,r},\kappa_{j,r}\bigr),
\]
with dispersion parameter $\kappa_{j,r}$ allowed to vary with latent volatility. In intermittent-demand settings, we additionally introduce a zero-generation mechanism,
\[
y_{j,r}
=
\begin{cases}
0, & \text{with probability } p_{j,r}^{(0)},\\
\tilde y_{j,r}, & \text{otherwise,}
\end{cases}
\]
where $\tilde y_{j,r}$ follows the corresponding base demand law.

Many datasets include latent mixture regimes. Specifically, for some seeds,
\[
y_{j,r}
\sim
\sum_{k=1}^{K_j} \pi_{j,r}^{(k)} \, \mathcal{D}_{j,r}^{(k)},
\]
where the mixture weights $\pi_{j,r}^{(k)}$ are time-varying and may depend on observed covariates and latent conditions. This mechanism allows bursty, clustered, or intermittent behavior even under similar observed contexts.

The synthetic suite incorporates several forms of nonstationarity:
\begin{enumerate}
    \item \textbf{baseline drift}, through gradual variation in $\beta_{j,0}(r)$;
    \item \textbf{coefficient drift}, through time variation in $\beta_j(r)$;
    \item \textbf{structural breaks}, representing abrupt changes in demand formation;
    \item \textbf{mixture-weight drift}, through time-varying regime probabilities;
    \item \textbf{covariate shift}, through evolution of the latent environment that generates $x_{j,r}$.
\end{enumerate}
Accordingly, the suite exhibits both temporal variation in the covariate distribution and temporal variation in the conditional law of demand.

\subsection{Disruptive Events, Note Persistence, and Within-Archetype Heterogeneity}
\label{app:synthetic_events_heterogeneity}

To incorporate unstructured information, we generate rare disruptive events such as typhoons, recalls, strikes, outages, pandemic waves, and food-safety rumors. Each event begins at some date $\tau$ and lasts for $\Lambda^{\mathrm{evt}}$ days. A textual note is recorded only at event onset:
\[
n_{j,r}
=
\begin{cases}
\text{event description}, & r=\tau \text{ and the event is observed},\\
\emptyset, & \text{otherwise.}
\end{cases}
\]
Thus, the note process records the onset of an observed disruption, while persistence is represented through the latent event-state process rather than repeated annotations.

Although the note appears only on the first day, the event continues to affect demand through a decaying latent impact:
\[
e_{j,r}^{(\ell)}
=
\begin{cases}
\kappa_{j,\ell}\,\omega_\ell(r-\tau), & \tau \le r < \tau+\Lambda^{\mathrm{evt}},\\
0, & \text{otherwise,}
\end{cases}
\]
where $\ell$ indexes the event type, $\kappa_{j,\ell}$ is a dataset-specific intensity, and $\omega_\ell(\cdot)$ is a decreasing decay profile. In implementation, both the conditional mean and the variance may respond to $e_{j,r}^{(\ell)}$. The sign and magnitude of event effects are domain-dependent; the same event label may increase demand in one domain and decrease it in another.

To avoid excessive similarity within a broad archetype, we introduce heterogeneity at three levels. First, \emph{variant-level heterogeneity} assigns different seeds within the same archetype to different demand-generation variants, thereby changing the functional form of the generator rather than only its coefficients. Second, \emph{feature-level heterogeneity} arises from feature-subset randomization, so datasets within the same archetype need not share the same observable schema. Third, \emph{shock-level heterogeneity} randomizes the timing, duration, observability, and intensity of disruptive events at the seed level. As a result, the final suite combines shared generation primitives with dataset-specific randomization while preserving a common construction template across seeds.

\section{Estimating the Replay--Deployment Discrepancy Budget \texorpdfstring{$\xi_t$}{xi\_t} in Real Data}
\label{app:xi-estimation}

In the main text, the replay--deployment discrepancy budget $\xi_t$ is defined through the integral probability metric
\[
\xi_t \equiv D_{\mathfrak F_t}(P_t^{\mathrm{dep}}, P_t^{\mathrm{rep}}),
\]
and enters both the promotion gate and the final deployment rule as a period-level safety margin. In implementation, however, the deployment distribution $P_t^{\mathrm{dep}}$ is not directly observable at decision time. As a result, $\xi_t$ must be estimated from historical replay-to-forward discrepancies rather than treated as a known quantity. This appendix presents a practical estimation strategy for $\xi_t$ in real data. We first introduce a general methodology that is independent of any specific dataset, and then instantiate it for the Dunnhumby Complete Journey benchmark used in Section~\ref{sec:exp-cj}.

\subsection{General estimation principle}
\label{app:xi-estimation-general}

Fix a period $t$. Let $R_t$ denote the replay window available at decision time, and let $D_t$ denote the corresponding forward window used only for retrospective validation.

\paragraph{Retrospective target.}
For any evaluated policy--comparator pair $(\pi,\tilde{\pi})$, define the empirical replay gain
\[
\widehat V_t^{rep}(\pi \mid \tilde{\pi})
=
\frac{1}{|R_t|}
\sum_{\omega \in R_t}
\bigl(C_t(\tilde{\pi};\omega)-C_t(\pi;\omega)\bigr),
\]
and the empirical forward gain
\[
\widehat V_t^{fwd}(\pi \mid \tilde{\pi})
=
\frac{1}{|D_t|}
\sum_{\omega \in D_t}
\bigl(C_t(\tilde{\pi};\omega)-C_t(\pi;\omega)\bigr).
\]
The realized replay-to-forward discrepancy for that pair is then
\[
\Delta_t(\pi,\tilde{\pi})
:=
\left|
\widehat V_t^{fwd}(\pi \mid \tilde{\pi})
-
\widehat V_t^{rep}(\pi \mid \tilde{\pi})
\right|.
\]

Let $\mathcal E_t$ denote the set of policy--comparator pairs evaluated in period $t$. A natural ex post benchmark is the oracle discrepancy
\[
\xi_t^{oracle}
:=
\max_{(\pi,\tilde{\pi}) \in \mathcal E_t}
\Delta_t(\pi,\tilde{\pi}),
\]
or, more robustly, its upper empirical quantile
\[
\xi_t^{oracle}(\beta)
:=
Q_{1-\beta}
\Bigl(
\{\Delta_t(\pi,\tilde{\pi}) : (\pi,\tilde{\pi}) \in \mathcal E_t\}
\Bigr),
\qquad \beta \in (0,1).
\]
The quantity $\xi_t^{oracle}$ is useful for retrospective diagnosis, but it cannot be used directly online because $D_t$ lies in the future relative to the decision at period $t$. When $\mathcal E_t$ includes the realized candidate policies generated during search, $\xi_t^{oracle}$ should be interpreted as an ex post operational calibration target attached to the realized evaluation set rather than as a method-free environment parameter. If a method-independent target is desired, the same construction can instead be applied to a fixed archived library of policy--comparator pairs.

\paragraph{Online budget.}
For implementation, we therefore distinguish between:
\begin{enumerate}
    \item the \emph{retrospective target} $\xi_t^{oracle}$, used only ex post for evaluation and calibration; and
    \item the \emph{online budget} $\widehat{\xi}_t$, used in the gate and deployment rule.
\end{enumerate}
The former is an ex post diagnostic target; the latter is the only quantity used online.

To construct $\widehat{\xi}_t$, we separate two roles for the shift features. First, retrospective calibration features $u_t^{\mathrm{ret}}=\phi(R_t,D_t)$ may use realized forward outcomes ex post to learn the mapping from realized shift patterns to oracle discrepancies. Second, online proxy features $u_t^{\mathrm{on}}=\phi(R_t,\widetilde D_t)$ use only pre-decision covariates or deployment proxies, where $\widetilde D_t$ denotes a proxy for the current deployment regime. Typical components of these feature vectors include differences in demand moments, changes in zero-demand ratio, covariate-distribution shifts, seasonality indicators, and differences in baseline-policy costs.

Using historical periods $s < t$, we form calibration pairs $\{(u_s^{\mathrm{ret}},\xi_s^{oracle})\}_{s=1}^{t-1}$, or, more robustly, $\{(u_s^{\mathrm{ret}},\xi_s^{oracle}(\beta))\}_{s=1}^{t-1}.$ All calibration of the quantile model is performed using slices strictly preceding the evaluation slice, or on a calibration pool disjoint from that slice, so that no forward window from the reported test workspace is used to fit $\widehat{\xi}_t$. We then fit an upper conditional quantile model
\[
\widehat q_{1-\alpha}(u)
\approx
Q_{1-\alpha}(\xi^{oracle}\mid u),
\]
and define the operational safety budget by $\widehat{\xi}_t=\widehat q_{1-\alpha}(u_t^{\mathrm{on}}) + b_t,$ where $b_t \ge 0$ is an optional finite-sample inflation term. The resulting quantity $\widehat{\xi}_t$ can be plugged directly into the certification gate and deployment rule:
\[
LCB_t(\pi \mid \pi_t^{\mathrm{ref}}) \ge \widehat{\xi}_t,
\qquad
LCB_t(\pi \mid \pi_{t,j-1}^{\mathrm{ch}}) \ge \varepsilon + \widehat{\xi}_t.
\]

\subsection{Operational estimators and recommended use}
\label{app:xi-estimation-operational}

The construction above can be implemented in several complementary ways.

\paragraph{Historical quantile calibration.}
A simple conservative estimator is obtained by pooling realized replay-to-forward discrepancies from the most recent $M_{\mathrm{hist}}$ historical periods:
\[
\widehat{\xi}_t^{hist}
=
Q_{1-\alpha}
\Bigl(
\bigcup_{s=t-M_{\mathrm{hist}}}^{t-1}
\{\Delta_s(\pi,\tilde{\pi}) : (\pi,\tilde{\pi}) \in \mathcal E_s\}
\Bigr).
\]
This estimator is easy to compute and does not require fitting a predictive model, but it does not adapt to the current period's specific shift pattern.

\paragraph{Shift-conditioned quantile estimation.}
A more adaptive estimator uses the online proxy feature vector $u_t^{\mathrm{on}}$:
\[
\widehat{\xi}_t^{shift}
=
\widehat q_{1-\alpha}(u_t^{\mathrm{on}}) + b_t.
\]
Possible choices for $\widehat q_{1-\alpha}$ include linear quantile regression, tree-based quantile regression, and conformalized quantile regression. This approach is particularly useful when replay--deployment mismatch varies systematically with observable state changes.

\paragraph{IPM plug-in estimator.}
Because the theoretical definition of $\xi_t$ is IPM-based, one may also construct an estimator that directly approximates the discrepancy between replay and deployment distributions. Let $\psi(\omega)\in\mathbb R^d$ be a path embedding, and let $\widehat P_t^{rep,\psi}$ and $\widehat P_t^{dep,\psi}$ denote empirical replay and deployment-proxy distributions in the embedding space. Then one may define
\[
\widehat{\xi}_t^{IPM}
=
\widehat D_{\mathfrak F_t}
\bigl(
\widehat P_t^{dep,\psi},
\widehat P_t^{rep,\psi}
\bigr),
\]
for a chosen function class $\mathfrak F_t$. This estimator is conceptually closest to the main theory, although it requires the additional design of an embedding map and a deployment proxy.

In practice, we recommend a conservative combination:
\[
\widehat{\xi}_t
=
\max\{\widehat{\xi}_t^{hist},\ \widehat{\xi}_t^{shift}\}.
\]
In the CJ instantiation below, we use this combined form to define the operational discrepancy budget. A concrete implementation should additionally report the numerical choices of $M_{\mathrm{hist}}$, $\alpha$, $b_t$, and the conditional-quantile model class used to construct $\widehat q_{1-\alpha}$.

\paragraph{Interpretation.}
From an operational viewpoint, $\widehat{\xi}_t$ should be interpreted as a conservative safety margin rather than a fully identified structural parameter. Its purpose is to upper-bound replay-to-forward gain errors with sufficiently high probability, thereby preserving the deployment guarantees encoded in the gate.

\subsection{Instantiation for the Dunnhumby Complete Journey dataset}
\label{app:xi-estimation-cj}

We now instantiate the above methodology for the Dunnhumby Complete Journey (CJ) benchmark used in Section~\ref{sec:exp-cj}. In the main experiment, each workspace is constructed from a temporal slice consisting of 365 historical days followed by 30 evaluation days. Each slice also contains six exogenous features:
\[
\texttt{is\_weekend},\ 
\texttt{is\_holiday},\ 
\texttt{discount\_rate},\ 
\texttt{is\_on\_display},\ 
\texttt{is\_in\_mailer},\ 
\texttt{day\_of\_week}.
\]
This structure makes the CJ data particularly suitable for replay-to-forward discrepancy calibration.

For each temporal slice $t$, we define:
\begin{itemize}
    \item the first 365 days as the replay window $R_t$;
    \item the subsequent 30 days as the forward validation window $D_t$.
\end{itemize}
For every evaluated pair $(\pi,\tilde{\pi})\in\mathcal E_t$, including pre-loop baseline/incumbent versus reference pairs and inner-loop generated-candidate pairs, we compute
\[
\Delta_t(\pi,\tilde{\pi})
=
\left|
\widehat V_t^{fwd}(\pi \mid \tilde{\pi})
-
\widehat V_t^{rep}(\pi \mid \tilde{\pi})
\right|.
\]
This yields the ex post oracle discrepancy $\xi_t^{oracle}=\max_{(\pi,\tilde{\pi})\in\mathcal E_t}
\Delta_t(\pi,\tilde{\pi}),$ or its quantile analogue. In this form, the CJ oracle discrepancy is defined relative to the realized evaluation set in the slice. If a method-independent calibration target is preferred, the same construction can instead be applied to a fixed archived policy library.

For retrospective calibration, we define a CJ-specific feature vector
\[
\begin{aligned}
u_t^{\mathrm{ret}}
= {} & \Bigl(
|\bar d_{D_t}-\bar d_{R_t}|,
|\sigma_{D_t}-\sigma_{R_t}|,
|\rho^0_{D_t}-\rho^0_{R_t}|,
\\
& \hspace{1.6em}|\overline{\texttt{discount}}_{D_t}-\overline{\texttt{discount}}_{R_t}|,
|\overline{\texttt{display}}_{D_t}-\overline{\texttt{display}}_{R_t}|,
\\
& \hspace{1.6em}|\overline{\texttt{mailer}}_{D_t}-\overline{\texttt{mailer}}_{R_t}|,
|\overline{\texttt{holiday}}_{D_t}-\overline{\texttt{holiday}}_{R_t}|
\Bigr).
\end{aligned}
\]
where $\bar d$ denotes mean demand, $\sigma$ denotes demand standard deviation, and $\rho^0$ denotes the zero-demand ratio. This retrospective form uses realized forward-window statistics only for ex post calibration. In online deployment, it should be replaced by an observable proxy vector $u_t^{\mathrm{on}}$ computed from pre-decision covariates or other deployment-time signals.

Using historical slices strictly preceding the evaluation slice, or a calibration pool disjoint from that slice, we fit an upper-quantile model $\widehat q_{1-\alpha}(u)\approx Q_{1-\alpha}(\xi^{oracle}\mid u).$ In particular, no forward window from the reported test workspace is used to fit the operational budget. We then define
\[
\widehat{\xi}_t^{shift,CJ}
=
\widehat q_{1-\alpha}(u_t^{\mathrm{on}})+b_t,
\qquad
\widehat{\xi}_t^{hist,CJ}
=
Q_{1-\alpha}
\Bigl(
\bigcup_{s=t-M_{\mathrm{hist}}}^{t-1}
\{\Delta_s(\pi,\tilde{\pi}) : (\pi,\tilde{\pi}) \in \mathcal E_s\}
\Bigr),
\]
and set the CJ operational discrepancy budget to
\[
\widehat{\xi}_t^{CJ}
=
\max\{\widehat{\xi}_t^{hist,CJ},\ \widehat{\xi}_t^{shift,CJ}\}.
\]
Finally, we replace $\xi_t$ in the promotion and deployment rules by $\widehat{\xi}_t^{CJ}$:
\[
LCB_t(\pi \mid \pi_t^{\mathrm{ref}}) \ge \widehat{\xi}_t^{CJ},
\qquad
LCB_t(\pi \mid \pi_{t,j-1}^{\mathrm{ch}}) \ge \varepsilon + \widehat{\xi}_t^{CJ}.
\]

This procedure preserves the distinction between the retrospective target $\xi_t^{oracle}$ and the online budget $\widehat{\xi}_t^{CJ}$. The former is used only ex post for calibration and diagnosis; the latter is the only quantity used in the gate and deployment rule.

\section{Practical Small-Sample Blockwise \texorpdfstring{$t$}{t}-Based Replay Radius Adjustment for Replay Certification}
\label{app:small-sample-radius}

The main text adopts the Hoeffding-style confidence radius in \eqref{eq:def_radius} as a conservative, distribution-free replay certificate. In practice, however, short replay windows may lead to substantial finite-sample conservatism. This appendix is mainly motivated by shorter replay windows, especially the 100+30 synthetic setting of Section~\ref{sec:exp-synthetic}, although the same idea can be used more broadly as a practical adjustment. The purpose of this appendix is therefore purely operational: we introduce a small-sample-friendly blockwise \(t\)-based replay radius for practical screening. This adjustment is intended as a practical screening heuristic rather than a finite-sample-valid replacement for the formal Hoeffding certificate in the main theory. Accordingly, the blockwise \(t\)-based radius should not be used to verify
Assumption~\ref{ass:adaptive_ci} unless it is combined with a separate
sample-splitting, holdout, or conformal calibration argument that restores
simultaneous adaptive validity.

A practical complication in inventory replay is that daily gain observations may exhibit short-range dependence due to inventory carryover, replenishment delay, and lead-time effects. To reduce the impact of such local dependence, we first aggregate daily replay gains into non-overlapping blocks.

For a candidate policy \(\pi\) and comparator \(\tilde\pi\), recall the daily replay gain
\[
Z_{t,\ell}(\pi\mid\tilde\pi)
=
C_t(\tilde\pi;\omega_{t,\ell})-C_t(\pi;\omega_{t,\ell}),
\qquad \ell=1,\dots,m_t.
\]
Let the replay window be partitioned into \(K_t\) non-overlapping blocks $\mathcal B_{t,1},\dots,\mathcal B_{t,K_t},$ each of length \(b_t\), except possibly the last block. We define the blockwise average gain by
\[
G_{t,k}(\pi\mid\tilde\pi)
=
\frac{1}{|\mathcal B_{t,k}|}
\sum_{\ell\in\mathcal B_{t,k}}
Z_{t,\ell}(\pi\mid\tilde\pi),
\qquad k=1,\dots,K_t.
\]
The corresponding blockwise mean and sample variance are
\[
\bar G_t(\pi\mid\tilde\pi)
=
\frac{1}{K_t}\sum_{k=1}^{K_t}G_{t,k}(\pi\mid\tilde\pi),
\quad s^2_{G,t}(\pi\mid\tilde\pi)
=
\frac{1}{K_t-1}\sum_{k=1}^{K_t}
\left(
G_{t,k}(\pi\mid\tilde\pi)-\bar G_t(\pi\mid\tilde\pi)
\right)^2.
\]

We then define the practical blockwise \(t\)-based replay radius by
\[
\widehat{\mathrm{rad}}^{\,\mathrm{prac}}_t(\pi\mid\tilde\pi)
=
t_{1-\alpha_t(\pi,\tilde\pi),\,K_t-1}
\cdot
\frac{s_{G,t}(\pi\mid\tilde\pi)}{\sqrt{K_t}},
\]
where \(t_{1-\alpha,\nu}\) denotes the \((1-\alpha)\)-quantile of the Student-\(t\) distribution with \(\nu\) degrees of freedom.

The resulting practical lower and upper confidence bounds are
\[
LCB_t^{\mathrm{prac}}(\pi\mid\tilde\pi)
=
\bar G_t(\pi\mid\tilde\pi)
-
\widehat{\mathrm{rad}}^{\,\mathrm{prac}}_t(\pi\mid\tilde\pi),
\quad 
UCB_t^{\mathrm{prac}}(\pi\mid\tilde\pi)
=
\bar G_t(\pi\mid\tilde\pi)
+
\widehat{\mathrm{rad}}^{\,\mathrm{prac}}_t(\pi\mid\tilde\pi).
\]

For multiple testing across the evaluated-pair budget \(N_t\), a simple conservative choice is $\alpha_t(\pi,\tilde\pi)=\frac{\delta_t}{2N_t}.$ This mirrors the role of \(N_t\) in the formal Hoeffding construction while remaining more sample-adaptive in practice.

\paragraph{Recommended block size and use in short replay windows.}
In inventory applications, we recommend
\[
b_t=\max\{7,\,L_t+1\},
\]
where \(L_t\) is the effective lead time in the current workspace. When lead time is fixed or not explicitly modeled in the replay implementation, a default choice \(b_t=7\) is a reasonable weekly aggregation rule. In a 100-day replay window, this yields approximately \(K_t\approx 14\) effective block observations. Thus, in the synthetic 100+30 setting, the natural default is weekly blocking with \(b_t=7\). For longer replay windows such as the CJ 365+30 slices, the same adjustment can still be used as an optional practical screen, but the small-sample motivation is weaker.


\section{Representative Structurally Evolved Policies}
\label{app:structural-evolved-policies}

This appendix reports several representative policies generated by the evolutionary search.  The purpose is to illustrate the types of white-box structural variation that the evolutionary process can discover while preserving direct implementability.

\begin{table}[!htbp]
\centering
\small
\caption{Evolved policies selected from the CJ 365+30 benchmark. Costs are average daily test costs; reductions are measured relative to the best classical baseline for the same workspace.}
\label{tab:representative-structural-policies}

\resizebox{\textwidth}{!}{%
\begin{tabular}{lllrp{0.34\textwidth}}
\toprule
Workspace & Product category & Best classical baseline & Reduction & Structural feature of evolved rule \\
\midrule
\texttt{produce\_860776\_s120} & Cucumbers & Capped base-stock & 23.5\% & Direct order-size rule with display-adjusted recent mean and variance. \\
\texttt{grocery\_871756\_s150} & Salad dressing & $(s,S)$ & 14.6\% & EWMA-style recent-demand center, local trend correction, and capped order volume. \\
\texttt{grocery\_5569327\_s300} & Tortilla chips & Capped base-stock & 10.3\% & Regime-adaptive base-stock target that changes with the amount and recency of demand history. \\
\texttt{grocery\_1092026\_s210} & Soft drinks & Capped base-stock & 27.7\% & Aggressive on-hand-only capped replenishment rule, deliberately excluding pipeline inventory from the trigger state. \\
\bottomrule
\end{tabular}%
}

\end{table}

\newpage

\begin{AIbox}{Case 1: Display-aware direct order-size rule}
\textbf{Workspace.} \texttt{cjv2\_produce\_860776\_s120}, National cucumbers.  

\textbf{Performance.} The evolved rule reduces average daily test cost from 0.6650 under the best classical baseline to 0.5089, a 23.5\% reduction.

\medskip
\noindent\textbf{Rule.} Rather than ordering up to a fixed inventory position, the policy directly computes the next order quantity from a short rolling demand window and a promotion/display indicator:
\[
\mu_t^{\rm eff}=\mu_t(1+0.30\,\mathbf{1}\{\mathrm{display}_t=1\}),\qquad
\sigma_t^{\rm eff}=\sigma_t(1+0.15\,\mathbf{1}\{\mathrm{display}_t=1\}),
\]
\[
Q_t=\left[\mu_t^{\rm eff}+\frac{1}{2}\Phi^{-1}(0.7026)\sigma_t^{\rm eff}\right]_{[0,50]}.
\]
This differs from a classical base-stock rule because it does not maintain a fixed order-up-to level; it acts as a feature-conditioned order-size controller.

\begin{lstlisting}[language=Python]
def get_order_quantity(on_hand_inventory, pipeline_inventory,
                       demand_history, lead_time, **kwargs):
    window = demand_history[-28:]
    mu = np.mean(window) if len(window) > 0 else 11.0
    sigma = np.std(window) if len(window) > 1 else 6.0

    display = kwargs.get('is_on_display', kwargs.get('is_disp', 0.0))
    mu_eff = mu * (1.3 if display > 0.5 else 1.0)
    sigma_eff = sigma * (1.15 if display > 0.5 else 1.0)

    z_alpha = stats.norm.ppf(0.7026)
    Q = mu_eff + z_alpha * sigma_eff * 0.5
    return max(0, int(max(0, min(Q, 50))))
\end{lstlisting}
\end{AIbox}

\begin{AIbox}{Case 2: EWMA--trend capped target rule}
\textbf{Workspace.} \texttt{cjv2\_grocery\_871756\_s150}, semi-solid salad dressing.  

\textbf{Performance.} The evolved rule reduces average daily test cost from 0.9090 under the best classical baseline to 0.7763, a 14.6\% reduction.

\medskip
\noindent\textbf{Rule.} The policy first estimates a local demand center using recency-weighted demand, then adjusts the target by a clipped one-step trend, and finally applies a cap to the order quantity:
\[
\widehat d_t=\sum_{j=1}^5 w_j d_{t-5+j},\qquad
\widetilde S_t=\alpha S+(1-\alpha)\widehat d_t+\mathrm{clip}(d_{t-1}-d_{t-2},-1,1),
\]
\[
Q_t=\min\{C,\,[\widetilde S_t-I_t-P_{t,1}]_+\}.
\]
Compared with $(s,S)$, this rule replaces a discontinuous reorder trigger with a smoothed adaptive target.

\begin{lstlisting}[language=Python]
def get_order_quantity(on_hand_inventory, pipeline_inventory,
                       demand_history, lead_time, **kwargs):
    S = kwargs.get('S', 20.0)
    C = kwargs.get('C', 6.0)
    alpha = kwargs.get('alpha', 0.4)

    if len(demand_history) >= 5:
        recent = np.array(demand_history[-5:])
        weights = np.array([1, 1.5, 2, 2.5, 3.0])
        weights = weights / weights.sum()
        mu = np.sum(weights * recent)
        target = S * alpha + mu * (1 - alpha)
    else:
        target = S

    if len(demand_history) >= 2:
        slope = demand_history[-1] - demand_history[-2]
        target = max(0, target + max(-1, min(1, int(slope * 1.2))))

    pipeline = (pipeline_inventory if isinstance(pipeline_inventory, (int, float))
                else pipeline_inventory[0] if len(pipeline_inventory) > 0 else 0)
    Q = max(0, int(np.ceil(target - on_hand_inventory - pipeline)))
    return min(C, Q)
\end{lstlisting}
\end{AIbox}

\begin{AIbox}{Case 3: Regime-adaptive base-stock target}
\textbf{Workspace.} \texttt{cjv2\_grocery\_5569327\_s300}, tortilla/nacho chips.  

\textbf{Performance.} The evolved rule reduces average daily test cost from 1.2152 under the best classical baseline to 1.0900, a 10.3\% reduction.

\medskip
\noindent\textbf{Rule.} The policy keeps the transparency of base-stock control, but replaces the fixed base-stock level with a regime-dependent target.  When little history is available it shrinks the target; with moderate history it expands or contracts according to a recent 15-period demand mean; with long history it uses a conservative long-history adjustment:
\[
S'_t=
\begin{cases}
0.8\gamma S, & |\mathcal H_t|<5,\\
S\left(1+\alpha(\bar d_{t,15}-1.5)/5\right), & 5\le |\mathcal H_t|<15,\\
0.95S, & |\mathcal H_t|\ge 365,\\
S, & \text{otherwise.}
\end{cases}
\qquad
Q_t=[S'_t-I_t-\sum_{\ell}P_{t,\ell}]_+.
\]
This is more than parameter tuning: the state-dependent target itself changes across data regimes.

\begin{lstlisting}[language=Python]
def get_order_quantity(on_hand_inventory, pipeline_inventory,
                       demand_history, lead_time, **kwargs):
    S = kwargs.get('S')
    gamma = kwargs.get('gamma', 1.0)
    alpha = kwargs.get('alpha', 1.0)

    if len(demand_history) < 5:
        S_prime = int(S * gamma * 0.8)
    elif len(demand_history) < 15:
        recent = demand_history[-15:]
        recent_mu = float(np.mean(recent)) if len(recent) > 0 else 2.5
        S_prime = int(S * (1 + alpha * (recent_mu - 1.5) / 5))
    elif len(demand_history) >= 365:
        S_prime = int(S * 0.95)
    else:
        S_prime = S

    inventory_position = on_hand_inventory + sum(pipeline_inventory)
    return max(0, int(S_prime) - int(inventory_position))
\end{lstlisting}
\end{AIbox}

\begin{AIbox}{Case 4: Aggressive on-hand-only capped replenishment}
\textbf{Workspace.} \texttt{cjv2\_grocery\_1092026\_s210}, soft drinks.  

\textbf{Performance.} The evolved rule reduces average daily test cost from 0.9216 under the best classical baseline to 0.6659, a 27.7\% reduction.

\medskip
\noindent\textbf{Rule.} This case is intentionally simple but structurally distinct.  A standard capped base-stock rule uses the full inventory position, including outstanding pipeline inventory.  The evolved rule instead triggers replenishment from on-hand inventory only:
\[
Q_t=\min\{C,[S-I_t]_+\},
\]
which creates a more aggressive service-oriented controller when the lost-sales penalty is high and demand is increasing.  Because the rule differs by the state variable rather than by a parameter, it is a useful example of a nonclassical but auditable heuristic discovered by evolution.

\begin{lstlisting}[language=Python]
def get_order_quantity(on_hand_inventory, pipeline_inventory,
                       demand_history, lead_time, **kwargs):
    S = kwargs.get('S')
    C = kwargs.get('C')

    # Unlike classical base-stock, this rule deliberately ignores pipeline.
    inventory_state = on_hand_inventory
    order = max(0, S - inventory_state)
    return min(int(order), int(C))
\end{lstlisting}
\end{AIbox}

\section{Construction and Censored-Demand Adjustment for the CJ Benchmark}
\label{app:cj-lost-sales}

This appendix describes the construction of the real-data CJ benchmark and the treatment of the lost-sales nature of point-of-sale data. The Dunnhumby Complete Journey dataset records realized sales transactions rather than true customer demand. In a lost-sales inventory system, observed sales may therefore be censored because unmet demand is not recorded. Formally, for product $i$ on day $t$, we observe
\[
Y_{it}=\min\{D_{it},A_{it}\},
\]
where $Y_{it}$ denotes observed sales, $D_{it}$ denotes latent demand, and $A_{it}$ denotes available inventory. Hence, sales equal demand only when inventory does not constrain the product. This feature motivates a preprocessing procedure that combines conservative sample screening with censored-demand estimation.

\paragraph{Screening zero-sales observations.}
We first aggregate transactions into a daily product-level panel and construct temporal slices, each consisting of a 365-day historical window followed by a 30-day evaluation window. To reduce the risk of treating frequent stockouts or highly intermittent products as genuinely low-demand items, we apply several screening rules before selecting the final test workspaces. Each product contributes at most one slice. For every candidate slice, the observed zero-sales ratio must be below 20\% in the historical window, the evaluation window, and the full 395-day window. We also impose minimum observed-sales volume requirements and remove nonstandard retail departments, such as miscellaneous transaction records and gas-kiosk records. After these filters, we obtain a candidate pool of 30 product-specific workspaces. The statistics of zero sales are reported in Table \ref{tab:cj-window-summary-by-dept}.

\begin{table}[!htbp]
  \centering
  \caption{Summary statistics of zero-sales rates by department and time window on the CJ 365+30 benchmark. Entries report the minimum, mean, median, and maximum zero-sales rates across workspaces within each department, expressed as percentages.}
  \label{tab:cj-window-summary-by-dept}
  \small
  \begin{tabular}{llccccc}
    \toprule
    \textbf{Department} & \textbf{Window} & \textbf{\# Cases} & \textbf{Min} & \textbf{Mean} & \textbf{Median} & \textbf{Max} \\
    \midrule
    Deli & 365-day history      & 1  & 0.3\%  & 0.3\%  & 0.3\%  & 0.3\% \\
         & 30-day evaluation    & 1  & 0.0\%  & 0.0\%  & 0.0\%  & 0.0\% \\
         & Full 395-day slice   & 1  & 0.3\%  & 0.3\%  & 0.3\%  & 0.3\% \\
    \midrule
    Drug \& GM & 365-day history      & 2  & 5.8\%  & 12.3\% & 12.3\% & 18.9\% \\
               & 30-day evaluation    & 2  & 6.7\%  & 10.0\% & 10.0\% & 13.3\% \\
               & Full 395-day slice   & 2  & 5.8\%  & 12.2\% & 12.2\% & 18.5\% \\
    \midrule
    Grocery & 365-day history      & 12 & 0.3\%  & 8.6\%  & 9.2\%  & 18.4\% \\
            & 30-day evaluation    & 12 & 0.0\%  & 4.7\%  & 3.3\%  & 16.7\% \\
            & Full 395-day slice   & 12 & 0.3\%  & 8.3\%  & 8.9\%  & 17.2\% \\
    \midrule
    Meat & 365-day history      & 1  & 5.8\%  & 5.8\%  & 5.8\%  & 5.8\% \\
         & 30-day evaluation    & 1  & 0.0\%  & 0.0\%  & 0.0\%  & 0.0\% \\
         & Full 395-day slice   & 1  & 5.3\%  & 5.3\%  & 5.3\%  & 5.3\% \\
    \midrule
    Meat (Packaged) & 365-day history      & 1  & 13.7\% & 13.7\% & 13.7\% & 13.7\% \\
                    & 30-day evaluation    & 1  & 10.0\% & 10.0\% & 10.0\% & 10.0\% \\
                    & Full 395-day slice   & 1  & 13.4\% & 13.4\% & 13.4\% & 13.4\% \\
    \midrule
    Produce & 365-day history      & 13 & 0.3\%  & 6.4\%  & 7.9\%  & 19.2\% \\
            & 30-day evaluation    & 13 & 0.0\%  & 2.6\%  & 0.0\%  & 10.0\% \\
            & Full 395-day slice   & 13 & 0.3\%  & 6.1\%  & 7.6\%  & 18.5\% \\
    \midrule
    \textbf{Overall} & \textbf{365-day history}    & \textbf{30} & \textbf{0.3\%} & \textbf{7.7\%} & \textbf{8.1\%} & \textbf{19.2\%} \\
                     & \textbf{30-day evaluation}  & \textbf{30} & \textbf{0.0\%} & \textbf{4.0\%} & \textbf{3.3\%} & \textbf{16.7\%} \\
                     & \textbf{Full 395-day slice} & \textbf{30} & \textbf{0.3\%} & \textbf{7.4\%} & \textbf{7.6\%} & \textbf{18.5\%} \\
    \bottomrule
  \end{tabular}
\end{table}

\paragraph{Censored-demand reconstruction.}
Because the CJ data do not provide reliable daily inventory positions, we infer potential censoring using a conservative stockout proxy. For each product, we estimate a local reference sales level from the rolling median of recent positive sales. We then adjust this benchmark upward when the item is on display, appears in mailers, or receives a positive discount. A day is flagged as potentially censored if sales are zero despite a high local expected sales level and recent sales activity or promotion exposure, or if positive sales are unusually low relative to the same local benchmark.

Let $c_{it}\in\{0,1\}$ denote the resulting censoring indicator. When $c_{it}=0$, we treat the observation as uncensored and set $D_{it}=Y_{it}$. When $c_{it}=1$, the observation implies the right-censoring condition $D_{it}\geq Y_{it}$. We then estimate a global censored Poisson generalized linear model,
\[
\log \mu_{it}=\beta^\top x_{it},
\]
where $x_{it}$ includes product-level demand scale, weekend and holiday indicators, discount, display, mailer, cyclic day-of-week terms, and department fixed effects. We estimate the model by maximizing the censored likelihood
\[
\mathcal{L}(\beta)
=
\prod_{c_{it}=0}
\mathbb{P}_{\beta}(D_{it}=Y_{it}\mid x_{it})
\prod_{c_{it}=1}
\mathbb{P}_{\beta}(D_{it}\geq Y_{it}\mid x_{it}).
\]
The final reconstructed demand used in the policy-evaluation scripts is
\[
\widehat D_{it}
=
\begin{cases}
Y_{it}, & c_{it}=0,\\
\mathbb{E}_{\widehat\beta}[D_{it}\mid D_{it}\geq Y_{it},x_{it}], & c_{it}=1.
\end{cases}
\]
For auditability, each workspace retains the original observed sales, the censoring indicator, the expected-sales stockout proxy, and the imputed demand lift. This procedure does not claim to recover true demand perfectly. Instead, it provides a conservative and reproducible correction for the well-known censoring problem in lost-sales retail data.

Overall, this procedure reduces the downward bias that arises when censored sales are treated as realized demand, while preserving reproducibility and auditability of the benchmark. The approach is also consistent with the censored-demand perspective in the lost-sales inventory literature \citep{huh2011adaptive,sachs2014data,trapero2024demand}.

\end{APPENDICES}

\end{document}